\RequirePackage[l2tabu,orthodox]{nag}		
\documentclass[reqno]{amsart}		
\usepackage[margin=1.5in,bottom=1.25in]{geometry}		

\usepackage{amsmath}		
\usepackage{amssymb}		
\usepackage{amsfonts}		
\usepackage{amsthm}		
\usepackage[foot]{amsaddr}		

\usepackage{mathtools}		

\mathtoolsset{%
}

\usepackage[utf8]{inputenc}		
\usepackage[T1]{fontenc}		

\usepackage[
cal=cm,
]
{mathalfa}

\usepackage[proportional,tabular,lining,sf,mono=false]{libertine}

\usepackage{dsfont}		

\usepackage{acronym}		
\renewcommand*{\aclabelfont}[1]{\acsfont{#1}}		
\newcommand{\acli}[1]{\textit{\acl{#1}}}		
\newcommand{\acdef}[1]{\textit{\acl{#1}} \textup{(\acs{#1})}\acused{#1}}		

\usepackage[labelfont={bf,small},labelsep=colon,font=small]{caption}	
\usepackage[dvipsnames,svgnames]{xcolor}		
\colorlet{MyRed}{Crimson!75!Black}
\colorlet{MyBlue}{MediumBlue!90!Black}
\colorlet{MyGreen}{DarkGreen!80!Black}

\newcommand{\afterhead}{.}
\newcommand{\ackperiod}{}		
\newcommand{\para}[1]{\medskip\paragraph{\textbf{#1\afterhead}}}

\usepackage{pifont}

\usepackage{wasysym}		

\usepackage{subcaption}		
\usepackage{tikz}		
\usepackage{tikzit}
\usetikzlibrary{calc,patterns}

\tikzstyle{new style 0}=[fill={rgb,255: red,255; green,16; blue,20}, draw=black, shape=circle]
\tikzstyle{new style 1}=[fill={rgb,255: red,66; green,255; blue,33}, draw=black, shape=circle]
\tikzstyle{new style 2}=[fill={rgb,255: red,34; green,37; blue,255}, draw=black, shape=circle]
\tikzstyle{new style 3}=[fill={rgb,255: red,10; green,104; blue,255}, draw=black, shape=circle]
\tikzstyle{new style 4}=[fill=white, draw={rgb,255: red,59; green,141; blue,21}, shape=circle]

\tikzstyle{new edge style 0}=[<->, fill=white]
\tikzstyle{new edge style 1}=[<-]
\tikzstyle{new edge style 2}=[-, draw={rgb,255: red,172; green,172; blue,255}, fill=none]
\tikzstyle{new edge style 3}=[-, draw={rgb,255: red,17; green,68; blue,255}]
\tikzstyle{purple}=[-, draw={rgb,255: red,138; green,20; blue,255}, fill={rgb,255: red,90; green,255; blue,137}]
\tikzstyle{new edge style 4}=[-, draw={rgb,255: red,250; green,0; blue,0}]
\tikzstyle{new edge style 5}=[dashed, fill=none, draw=black, -]

\usepackage{array}		
\usepackage{booktabs}		
\usepackage[inline,shortlabels]{enumitem}		
\setenumerate{itemsep=\smallskipamount,topsep=\smallskipamount}

\usepackage[kerning=true]{microtype}		

\usepackage{xspace}		

\usepackage[numbers,sort&compress]{natbib}		

\usepackage{hyperref}
\hypersetup{
colorlinks=true,
linktocpage=true,
pdfstartview=FitH,
breaklinks=true,
pdfpagemode=UseNone,
pageanchor=true,
pdfpagemode=UseOutlines,
plainpages=false,
bookmarksnumbered,
bookmarksopen=false,
bookmarksopenlevel=1,
hypertexnames=true,
pdfhighlight=/O,
urlcolor=MyRed,linkcolor=MyBlue,citecolor=MyBlue,	
pdftitle={},
pdfauthor={},
pdfsubject={},
pdfkeywords={},
pdfcreator={pdfLaTeX},
pdfproducer={LaTeX with hyperref}
}

\usepackage[sort&compress,capitalize,nameinlink]{cleveref}		
\crefname{algorithm}{Alg.}{Algs.}

\usepackage{algorithm}		
\usepackage{algpseudocode}		

\usepackage{thmtools}		
\usepackage{thm-restate}		

\theoremstyle{plain}
\newtheorem{lemma}{Lemma}		

\newtheorem*{corollary*}{Corollary}		

\theoremstyle{definition}
\newtheorem{definition}{Definition}		
\newtheorem{assumption}{Assumption}		
\newtheorem{example}{Example}		

\newtheorem*{definition*}{Definition}		
\newtheorem*{assumption*}{Assumptions}		
\newtheorem*{example*}{Example}		

\theoremstyle{remark}
\newtheorem{remark}{Remark}		
\newtheorem*{remark*}{Remark}		

\newcommand{\debug}[1]{#1}		

\newcommand{\newmacro}[2]{\newcommand{#1}{\debug{#2}}}		
\newcommand{\newop}[2]{\DeclareMathOperator{#1}{\debug{#2}}}		

\DeclarePairedDelimiter{\braces}{\{}{\}}		
\DeclarePairedDelimiter{\bracks}{[}{]}		
\DeclarePairedDelimiter{\parens}{(}{)}		

\DeclarePairedDelimiter{\abs}{\lvert}{\rvert}		
\DeclarePairedDelimiter{\ceil}{\lceil}{\rceil}		
\DeclarePairedDelimiter{\floor}{\lfloor}{\rfloor}		

\DeclarePairedDelimiterX{\inner}[2]{\langle}{\rangle}{#1,#2}		
\DeclarePairedDelimiter{\norm}{\lVert}{\rVert}		
\DeclarePairedDelimiterXPP{\dnorm}[1]{}{\lVert}{\rVert}{_{\ast}}{#1}		
\DeclarePairedDelimiterXPP{\tvnorm}[1]{}{\lVert}{\rVert}{_{\mathrm{TV}}}{#1}		
\DeclarePairedDelimiterXPP{\onenorm}[1]{}{\lVert}{\rVert}{_{1}}{#1}		
\DeclarePairedDelimiterXPP{\twonorm}[1]{}{\lVert}{\rVert}{_{2}}{#1}		
\DeclarePairedDelimiterXPP{\supnorm}[1]{}{\lVert}{\rVert}{_{\infty}}{#1}		

\DeclarePairedDelimiterX{\braket}[2]{\langle}{\rangle}{#1,#2}		

\DeclarePairedDelimiterX{\setdef}[2]{\{}{\}}{#1:#2}		
\DeclarePairedDelimiterX{\window}[2]{\{}{\}}{#1,\dotsc,#2}		
\DeclarePairedDelimiterXPP{\exclude}[1]{\mathopen{}\setminus}{\{}{\}}{}{#1}

\newcommand{\alt}[1]{#1'}		

\newcommand{\R}{\mathbb{R}}		

\DeclareMathOperator*{\argmax}{arg\,max}		
\DeclareMathOperator*{\argmin}{arg\,min}		
\DeclareMathOperator*{\union}{\bigcup}		

\DeclareMathOperator{\bigoh}{\mathcal{O}}		
\DeclareMathOperator{\diam}{diam}		
\newop{\dom}{dom}		
\DeclareMathOperator{\im}{im}		
\DeclareMathOperator{\one}{\mathds{1}}		
\DeclareMathOperator{\relint}{ri}		
\DeclareMathOperator{\unif}{unif}		

\newmacro{\coef}{s}		
\newmacro{\weight}{\alpha}		
\newmacro{\dd}{\:d}		

\newcommand{\eps}{\varepsilon}		

\newcommand{\insum}{\sum\nolimits}		

\newcommand{\cf}{cf.\xspace}		
\newcommand{\eg}{e.g.,\xspace}		
\newcommand{\ie}{i.e.,\xspace}		

\newcommand{\textpar}[1]{\textup(#1\textup)}		

\newcommand{\txs}{\textstyle}		

\newcommand{\from}{\colon}		

\newcommand{\defeq}{\coloneqq}		

\newmacro{\set}{\mathcal{S}}		
\newmacro{\setalt}{\alt\set}		
\newmacro{\iSet}{i}		
\newmacro{\jSet}{j}		
\newmacro{\nSets}{m}		
\newmacro{\sets}{\mathcal{P}}		

\newmacro{\points}{\mathcal{K}}		
\newmacro{\intpoints}{\points^{\circ}}		
\newmacro{\point}{x}		
\newmacro{\pointalt}{\alt\point}		

\newmacro{\primal}{p}		
\newmacro{\proxal}{q}		
\newmacro{\primalt}{q}		

\newmacro{\dpoints}{\mathcal{Y}}		
\newmacro{\dpoint}{y}		
\newmacro{\ppoint}{\chi}		
\newmacro{\dpointalt}{\alt\dpoint}		

\newmacro{\base}{q}		
\newmacro{\basealt}{q}		

\newmacro{\pointest}{\point}		
\newmacro{\stratest}{\strat}		
\newmacro{\ptest}{\pdist}		

\newmacro{\open}{\mathcal{U}}		
\newmacro{\closed}{\mathcal{C}}		
\newmacro{\cpt}{\mathcal{K}}		
\newmacro{\nhd}{\mathcal{U}}		

\newmacro{\start}{1}		
\newmacro{\runstart}{\tau}		
\newmacro{\running}{1,2,\dotsc}		
\newmacro{\run}{t}		
\newmacro{\runalt}{s}		
\newmacro{\nRuns}{T}		
\newmacro{\runtime}{\theta}		

\newmacro{\runs}{\mathcal{\nRuns}}		
\newmacro{\subruns}{\alt\runs}		

\newcommand{\new}[1]{#1^{+}}		

\newmacro{\choice}{\point}		
\newmacro{\policy}{\simple}		
\newmacro{\state}{X}		
\newmacro{\score}{S}		

\newmacro{\aux}{\tilde\state}		
\newmacro{\daux}{\tilde\score}		

\newmacro{\step}{\gamma}		
\newmacro{\temp}{\eta}		

\newmacro{\nSplits}{\sigma}		
\newmacro{\schedule}{u}

\newmacro{\linf}{\mathcal{L}^{\infty}(\points)}		
\newmacro{\squareints}{\mathcal{L}^{2}(\points)}		

\newmacro{\banach}{\mathcal{V}}		
\newmacro{\hilbert}{\mathcal{H}}		

\newmacro{\fun}{\varphi}		

\newmacro{\vecspace}{\mathcal{V}}		
\newmacro{\vdim}{d}		
\newmacro{\vvec}{x}		
\newmacro{\bvec}{e}		
\newmacro{\unitvec}{u}		
\newmacro{\tvec}{z}		

\newmacro{\tanhull}{\mathcal{Z}}		
\newmacro{\tanvec}{z}		

\newmacro{\dspace}{\R^{\sets}}		
\newmacro{\pspace}{\R^{\sets}}		
\newmacro{\dual}{\psi}		
\newmacro{\dvec}{v}		
\newmacro{\dbvec}{\eps}		

\newmacro{\ones}{\mathbf{1}}		
\newmacro{\mat}{M}		
\newmacro{\eye}{I}		

\newop{\tspace}{T}		
\newop{\tcone}{TC}		
\newop{\dcone}{\tcone^{\ast}}		
\newop{\ncone}{NC}		
\newop{\pcone}{PC}		

\newmacro{\cvx}{\mathcal{C}}		
\newmacro{\subd}{\partial}		
\newmacro{\hmat}{H}		

\newop{\Opt}{Opt}		
\newop{\Sol}{Sol}		

\newmacro{\obj}{f}		
\newmacro{\objalt}{g}		
\newmacro{\sobj}{F}		

\newmacro{\param}{\theta}		
\newmacro{\params}{\Theta}		

\newmacro{\gvec}{g}		
\newmacro{\vecfield}{v}		

\newmacro{\gbound}{G}		
\newmacro{\rbound}{R}		
\newmacro{\lbound}{M}		

\newcommand{\sol}[1][\point]{#1^{\ast}}		

\newmacro{\strong}{\ell}		
\newmacro{\lips}{L}		
\newmacro{\hold}{L}		
\newmacro{\hexp}{\alpha}		

\newmacro{\minmax}{\Phi}		

\newmacro{\minvar}{x}		
\newmacro{\minvaralt}{\alt x}		
\newmacro{\minvars}{\mathcal{X}}		

\newmacro{\maxvar}{y}		
\newmacro{\maxvaralt}{\alt y}		
\newmacro{\maxvars}{\mathcal{Y}}		

\newop{\NE}{NE}		
\newop{\CE}{CE}		
\newop{\CCE}{CCE}		

\newop{\brep}{br}		
\newop{\reg}{Reg}		
\newop{\regalt}{\widetilde{Reg}}		
\newop{\preg}{\overline{Reg}}		
\newop{\dynreg}{DynReg}		
\newop{\val}{val}		

\newmacro{\strat}{p}		
\newmacro{\stratalt}{q}		
\newmacro{\strats}{\simplex(\pures)}		
\newmacro{\intstrats}{\strats^{\circ}}		

\newmacro{\abscont}{\textup{ac}}
\newmacro{\sing}{\textup{sing}}

\newmacro{\pdist}{p}		
\newmacro{\dirac}{\delta}		
\newmacro{\simple}{q}		
\newmacro{\simplealt}{\alt\simple}		
\newmacro{\simples}{\mathcal{Q}}		

\newmacro{\play}{i}		
\newmacro{\playalt}{j}		
\newmacro{\nPlayers}{N}		
\newmacro{\players}{\mathcal{\nPlayers}}		

\newmacro{\pure}{a}		
\newmacro{\purealt}{a'}		
\newmacro{\nPures}{n}		
\newmacro{\pures}{\mathcal{A}}		

\newmacro{\cost}{c}		
\newmacro{\loss}{\ell}		
\newmacro{\lossvec}{L}		

\newmacro{\pay}{u}		
\newmacro{\payv}{r}		
\newmacro{\pot}{\obj}		

\newmacro{\game}{\mathcal{G}}		
\newmacro{\gamefull}{\game(\players,\points,\pay)}		

\newmacro{\fingame}{\Gamma}		
\newmacro{\fingamefull}{\Gamma(\players,\pures,\pay)}		

\newop{\Eucl}{\Pi}		
\newop{\logit}{\Lambda}		

\newmacro{\hreg}{h}		
\newmacro{\hconj}{\hreg^{\ast}}		
\newmacro{\hdec}{\theta}		
\newmacro{\breg}{D}		
\newmacro{\pmap}{P}		
\newmacro{\mirror}{Q}		
\newmacro{\fench}{F}		
\newmacro{\hstr}{K}		
\newmacro{\depth}{H}		
\newmacro{\zone}{\mathbb{D}}		
\newmacro{\subpoints}{\points^{\circ}}		
\newmacro{\proxdom}{\mathcal{Q}}		

\newmacro{\fish}{F}
\newmacro{\energy}{E}		
\newmacro{\hvol}{\phi}		
\newmacro{\smooth}{\beta}		

\DeclareMathOperator{\ex}{\mathbb{E}}		
\DeclareMathOperator{\prob}{\mathbb{P}}		
\DeclareMathOperator{\simplex}{\Delta}		

\newmacro{\sample}{\omega}		
\newmacro{\samples}{\Omega}		

\newmacro{\seed}{\xi}		
\newmacro{\seeds}{\Xi}		

\newmacro{\filter}{\mathcal{F}}		
\newmacro{\history}{\mathcal{F}}		
\newmacro{\probspace}{(\samples,\filter,\prob)}		

\newmacro{\meas}{\mu}		
\newmacro{\leb}{\lambda}		

\newmacro{\event}{E}       
\newmacro{\eventalt}{H}       

\newmacro{\mean}{\mu}		
\newmacro{\sdev}{\sigma}		
\newmacro{\variance}{\sdev^{2}}		

\newmacro{\dkl}{D_{\mathrm{KL}}}		

\providecommand\given{}		

\DeclarePairedDelimiterXPP{\exof}[1]{\ex}{[}{]}{}{
\renewcommand\given{\nonscript\,\delimsize\vert\nonscript\,\mathopen{}} #1}

\DeclarePairedDelimiterXPP{\probof}[1]{\prob}{(}{)}{}{
\renewcommand\given{\nonscript\:\delimsize\vert\nonscript\:\mathopen{}} #1}

\DeclarePairedDelimiterXPP{\oneof}[1]{\one}{(}{)}{}{
\renewcommand\given{\nonscript\,\delimsize\vert\nonscript\,\mathopen{}} #1}

\newcommand{\est}[1]{\hat #1}		

\newmacro{\model}{\est\pay}		
\newmacro{\error}{Z}		
\newmacro{\noise}{U}		
\newmacro{\bias}{b}		

\newmacro{\mbound}{M}		
\newmacro{\vbound}{M}		
\newmacro{\bbound}{\mu}		
\newmacro{\totbound}{S}		

\newmacro{\snoise}{\psi}		
\newmacro{\sbias}{\beta}		

\newmacro{\noisedev}{\sigma}		
\newmacro{\noisevar}{\noisedev^{2}}		

\newmacro{\unitvar}{E}		
\newmacro{\pertvar}{W}		

\newmacro{\radius}{r}		

\newmacro{\flowmap}{\Phi}		

\newmacro{\graph}{\mathcal{G}}
\newmacro{\vertices}{\mathcal{V}}
\newmacro{\edges}{\mathcal{E}}

\newmacro{\const}{c}

\newmacro{\diamconst}{C_\points}

\newmacro{\coord}{i}
\newmacro{\coordalt}{j}
\newmacro{\nCoords}{m}

\newmacro{\budget}{V}
\newmacro{\batch}{\Delta}
\newmacro{\iBatch}{i}
\newmacro{\nBatches}{N}

\newmacro{\pexp}{\varrho}		
\newmacro{\qexp}{\gamma}		
\newmacro{\bexp}{\beta}		
\newmacro{\mexp}{\mu}		
\newmacro{\vexp}{\nu}		
\newmacro{\wexp}{\mu}		
\newmacro{\dexp}{\kappa}		
\newmacro{\nSetsExp}{p}		

\newmacro{\kerfun}{K}		
\newmacro{\mix}{\eps}		

\newop{\vbudget}{VB}		
\newcommand{\tvar}[1][\nRuns]{\debug{V_{#1}}}		

\newcommand{\legendStatic}{\textit{Static} regret divided by $t$ ($\reg(t)/t$) in log-log scale, averaged on 46 realizations for each algorithm (solid line), against the }
\newcommand{\legendDyn}{\textit{Dynamic} regret divided by $t$ ($\dynreg(t)/t$) in log-log scale, averaged on 46 realizations for each algorithm (solid line), against the }
\newcommand{\legendSlide}{Distribution (box-and-whiskers are at a CI of $[.05, .95]$) of the averaged static regret over $t$, averaged on 46 realizations for each algorithm at 2 different timestamps, against the }

\newmacro{\cover}{\sets}		
\newmacro{\region}{\mathcal{R}}		
\newmacro{\auxreg}{\tilde R}		

\begin{document}


\newcommand{\longtitle}{\uppercase{Zeroth-Order Non-Convex Learning\\via Hierarchical Dual Averaging}}

\title[Zeroth-Order Non-Convex Learning]{\longtitle}		

\author
[A.~Héliou]
{Amélie Héliou$^{\ast}$}
\address{$\ast$\,Criteo AI Lab, France}
\email{a.heliou@criteo.com}

\author
[M.~Martin]
{Matthieu Martin$^{\ast}$}
\email{mat.martin@criteo.com}

\author
[P.~Mertikopoulos]
{\\Panayotis Mertikopoulos$^{\diamond,\ast}$}
\address{$^{\diamond}$\,%
Univ. Grenoble Alpes, CNRS, Inria, Grenoble INP, LIG, 38000 Grenoble, France.}
\email{panayotis.mertikopoulos@imag.fr}

\author
[T.~Rahier]
{Thibaud Rahier$^{\ast}$}
\email{t.rahier@criteo.com}

\thanks{Authors appear in alphabetical order; send all correspondence to \texttt{panayotis.mertikopoulos@imag.fr}.}
\thanks{
P.~Mertikopoulos is grateful for financial support by
the French National Research Agency (ANR) in the framework of
the ``Investissements d'avenir'' program (ANR-15-IDEX-02),
the LabEx PERSYVAL (ANR-11-LABX-0025-01),
MIAI@Grenoble Alpes (ANR-19-P3IA-0003),
and the grants ORACLESS (ANR-16-CE33-0004) and ALIAS (ANR-19-CE48-0018-01).
This research was also supported by the COST Action CA16228 ``European Network for Game Theory'' (GAMENET)\ackperiod}


\newacro{LHS}{left-hand side}
\newacro{RHS}{right-hand side}
\newacro{iid}[i.i.d.]{independent and identically distributed}
\newacro{lsc}[l.s.c.]{lower semi-continuous}

\newacro{MAB}{multi-armed bandit}
\newacro{OCO}{online convex optimization}
\newacro{ZOO}{zeroth-order online optimization}

\newacro{EW}{exponential weights}
\newacro{EXP3}{exponential weights algorithm for exploration and exploitation}
\newacro{HDA}{hierarchical dual averaging}
\newacro{HEW}{hierarchical exponential weights}

\newacro{IWE}{importance weighted estimator}
\newacro{IWE3}[IWE$^{3}$]{importance weighted estimator with explicit exploration}

\newacro{OGD}{online gradient descent}
\newacro{MD}{mirror descent}
\newacro{OMD}{online mirror descent}
\newacro{MWU}{multiplicative weights update}
\newacro{FTPL}{``follow the perturbed leader''}
\newacro{FTRL}{``follow the regularized leader''}
\newacro{SPSA}{simultaneous perturbation stochastic approximation}
\newacro{RN}{Radon-Nikodym}
\newacro{MCP}{maximum clique problem}
\newacro{DAIM}{dual averaging with inexact models}
\newacro{BDA}{bandit dual averaging}
\newacro{DA}{dual averaging}
\newacro{DAX}{dual averaging with an explicit cover}
\newacro{FTRL}{``follow the regularized leader''}
\newacro{NE}{Nash equilibrium}
\newacroplural{NE}[NE]{Nash equilibria}
\newacro{TA}{test acronym}

\begin{abstract}
%
%
We propose a hierarchical version of dual averaging for zeroth-order online non-convex optimization \textendash\ \ie learning processes where, at each stage, the optimizer is facing an unknown non-convex loss function and only receives the incurred loss as feedback.
The proposed class of policies relies on the construction of an online model that aggregates loss information as it arrives, and it consists of two principal components:
\begin{enumerate*}
[(\itshape a\upshape)]
\item
a regularizer adapted to the \emph{Fisher information metric} (as opposed to the metric norm of the ambient space);
and
\item
a principled exploration of the problem's state space based on an adapted hierarchical schedule.
\end{enumerate*}
This construction enables sharper control of the model's bias and variance, and allows us to derive tight bounds for both the learner's static and dynamic regret \textendash\ \ie the regret incurred against the best dynamic policy in hindsight over the horizon of play.
\end{abstract}

\maketitle
\allowdisplaybreaks		
\acresetall		

\section{Introduction}
\label{sec:introduction}

Zeroth-order \textendash\ or \emph{derivative-free} \textendash\ optimization concerns the problem of optimizing a given function without access to its gradient, stochastic or otherwise.
Its study dates back at least to \citet{Ros60}, and it has recently attracted significant interest in machine learning and artificial intelligence due to the prohibitive cost of automatic differentiation in very large neural nets and language models.

A standard approach to zeroth-order optimization involves
sampling the function to be optimized at several nearby points,
using the observed values to reconstruct the gradient of the function,
and then employing a standard, first-order method \citep{CSV09}.
This approach allows the optimizer to approximate the gradient of the function to arbitrary precision (at least, if enough queries are made).
However, this also requires that the problem's objective remain stationary during the query process.

Motivated by applications to online ad auctions and recommender systems, our paper concerns the case where this stationarity assumption breaks down \textendash\ the \acdef{ZOO} setting.
Specifically, we consider an adversarial \ac{ZOO} problem that unfolds as follows:
\begin{enumerate}
[left=\parindent]
\item
At each stage $\run=\running$, the optimizer selects an action $\choice_{\run}$ from a compact convex subset $\points$ of $\R^{\vdim}$.
\item
Simultaneously, an adversary selects a reward function $\pay_{\run}\from\points\to\R$, often assumed to take values in $[0,1]$.
\item
The optimizer receives $\pay_{\run}(\choice_{\run})$ as a reward, and the process repeats.
\end{enumerate}
The learner's performance after $\nRuns$ stages is measured here by their regret, viz. $R_{\nRuns} = \sum_{\run=\start}^{\nRuns} \bracks{\pay_{\run}(\point) - \pay_{\run}(\choice_{\run})}$, and the learner's goal is to minimize the growth rate of $R_{\nRuns}$.

Since each individual $\pay_{\run}$ may be encountered once \textendash\ and only once \textendash\ it is no longer possible to perform multiple queries per function.
On that account, the \ac{SPSA} estimator of \citet{Spa92} has been studied extensively as a viable alternative to multiple-point query methods for online optimization.
In particular, using a variant of the \ac{SPSA} scheme, \citet{FKM05} showed that it is possible to achieve $\bigoh(\nRuns^{3/4})$ regret if the payoff functions encountered are concave.
The corresponding lower bound is $\Omega(\nRuns^{1/2})$, and it was only recently achieved by the kernel-based method of \citet{BE16} and \citet{BLE17}.

When venturing beyond problems with a convex structure, the situation is significantly more complicated.
The most widely studied case is the ``Lipschitz bandit'' \textendash\ or, sometimes, ``Hölder bandit'' \textendash\ framework where each $\pay_{\run}$ is a random realization of a parametric model of the form $\pay_{\run}(\point) = \est\pay(\point;\seed_{\run})$ with Lipschitz continuous mean $\pay(\point) = \ex_{\seed}[\est\pay(\point;\seed)]$, \cf \citet{Agr95}.
In this case, the lower bound for the regret is $\Omega(\nRuns^{\frac{\vdim+1}{\vdim+2}})$, and several algorithms have been proposed to achieve it, typically by combining an intelligent discretization of the problem's search region with a deterministic UCB-type policy \citep{BMSS11,KSU08,Sli19}.

On the other hand, in an adversarial setting, an informed adversary can always impose $\Omega(\nRuns)$ regret to any \emph{deterministic} decision algorithm employed by the learner, \cf \citet{SS11,HSZ17,SN20}.
This makes the algorithms designed for Lipschitz bandits ill-suited for the framework at hand, and necessitates a different approach.
In this direction, \citet{KBTB15} showed that, if each payoff function $\pay_{\run}$ is revealed to the learner after playing, it is possible to achieve $\bigoh(\nRuns^{1/2})$ regret.
Similar bounds were obtained more recently by \citet{AGH19} and \citet{SN20}, who examined the \acdef{FTPL} algorithm of \citet{KV05} assuming access to an offline optimization oracle;
however, the knowledge of $\pay_{\run}$ is still implicitly required in these works (as input to an optimization or sampling oracle, depending on the context).

More recently, \citet{HMMR20} proposed a general \acl{DA} framework for online non-convex learning with imperfect feedback, including the bona fide, adversarial \ac{ZOO} case.
Specifically, by using a ``kernel smoothing'' method in the spirit of \citet{BLE17}, \citet{HMMR20} proposed a \ac{ZOO} method achieving
\begin{enumerate*}
[\itshape a\upshape)]
\item
a suboptimal $\bigoh(\nRuns^{\frac{\vdim+2}{\vdim+3}})$ regret bound;
and
\item
a commensurate $\bigoh(\nRuns^{\frac{\vdim+3}{\vdim+4}}\tvar^{\frac{1^{\vphantom{1}}}{\vdim+4}})$ bound for the learner's \emph{dynamic} regret, with $\tvar = \insum_{\run=\start}^{\nRuns} \supnorm{\pay_{\run+1} - \pay_{\run}}$ denoting the \emph{total variation} of the payoff functions encountered (a common dynamic regret benchmark introduced by \citealp{BGZ15}).
\end{enumerate*}
However, the kernel method employed by \citet{HMMR20} is difficult to implement because the kernel's support function may grow exponentially in both $\nRuns$ and $\vdim$.

\para{Our contributions}

In this paper, we take a different approach that fuses the \acl{DA} framework of \citet{KBTB15} with a hierarchical exploration scheme in the spirit of \citet{BMSS11} and \citet{KSU08,KSU19}.
Specifically, we propose a flexible, anytime \acdef{HDA} method with the following desirable properties:
\begin{enumerate*}
[(\itshape i\hspace*{1pt})]
\item
it enjoys a min-max optimal $\bigoh(\nRuns^{\frac{\vdim+1}{\vdim+2}})$ static regret bound;
\item
it guarantees at most $\bigoh(\nRuns^{\frac{\vdim+2}{\vdim+3}}\tvar^{\frac{1^{\vphantom{1}}}{\vdim+3}})$ dynamic regret.
\end{enumerate*}
In this way, our paper closes the optimality gap in the regret analysis of \citet{HMMR20}, and it answers in the positive the authors' conjecture that it is possible to achieve $\bigoh(\nRuns^{\frac{\vdim+2}{\vdim+3}}\tvar^{\frac{1^{\vphantom{1}}}{\vdim+3}})$ dynamic regret in adversarial \ac{ZOO} problems.

As far as we are aware, \ac{HDA} is the first algorithm in the literature enjoying this dynamic regret guarantee.
Moreover, in contrast to the CAB algorithm of \citet{Kle04}, we should stress that \ac{HDA} \emph{does not} require a restart schedule or a doubling trick.
From a practical viewpoint, this is particularly important because the doubling trick leads to sharp performance drops when the algorithm periodically restarts from scratch \textendash\ an unpleasant property, which is one of the main reasons that doubling methods are rarely employed by practitioners \citep{BMSS11}.

Our analysis relies on two principal components:
\begin{enumerate*}
[\itshape a\upshape)]
\item
a logarithmic scheduler for controlling the hierarchical exploration of the problem's state space;
and
\item
a regularization framework adapted to the Fisher information metric on the learner's mixed strategies.
\end{enumerate*}
The first of these components marks a crucial point of departure from the hierarchical approach of \citet{BMSS11} and \citet{KSU19} since, instead of increasing the granularity of our search ``pointwise'', we do so ``dimension-wise'' (but at a slower pace).
As for the second component, the use of the Fisher information metric allows us to drop the reliance of \acl{DA} on a global norm that is not adapted to the geometry of the problem at hand, and it allows us to bring into play a wide range of regularizers that were previously unexplored in the literature \textendash\ such as the Burg entropy.
This is a crucial difference with existing results on \acl{DA}, and it allows for much finer control of the learning process as it unfolds \textendash\ precisely because the information content of the learner's policy is not ignored in the process.

Upon completion of our paper, we discovered a very recent preprint by \citet{PS21} that proposes an adversarial zooming algorithm.
The authors achieve a static $\bigoh(\nRuns^{\frac{\vdim+1}{\vdim+2}})$ regret bound in high probability (but do not provide any dynamic regret guarantees).
Their algorithm uses an explicit exploration term, plus a confidence term in the per-round sampling uncertainty.
Their splitting rule splits only one-by-one cover set into $2^\vdim$ sub-covers, which might be more difficult to implement in practice.

\section{Setup and preliminaries}
\label{sec:setup}

\subsection{The model}

We assume throughout that $\points$ is a compact convex subset of an ambient real space $\R^{\vdim}$ endowed with an abstract norm $\norm{\cdot}$ and a reference measure $\leb$ (typically the ordinary Lebesgue measure).
As for the payoff functions encountered by the learner, we will make the following blanket assumption:

\begin{assumption}
\label{asm:pay}
The stream of payoff functions $\pay_{\run}\from\points\to\R$, $\run=\running$, is \emph{uniformly bounded Lipschitz},
\ie there exist nonnegative constants $\rbound, \lips \geq 0$ such that
\begin{enumerate}
[parsep=\smallskipamount,topsep=0ex]
\item
\label{asm:bounded}
$0 \leq \pay_{\run}(\point) \leq \rbound$ for all $\point\in\points$.
\item
\label{asm:Lips}
$\abs{\pay_{\run}(\pointalt) - \pay_{\run}(\point)} \leq \lips \norm{\pointalt - \point}$
for all $\point,\pointalt\in\points$.
\end{enumerate}
\end{assumption}

To avoid exploitable, deterministic strategies, we will assume that the learner has access to an unobservable randomizer that can be used to choose an action $\point\in\points$ by means of a probability distribution on $\points$ \textendash\ that is, a \emph{mixed strategy}.
Of course, in complete generality, the space of all mixed strategies is impractical to work with because it contains probability distributions that cannot be described in closed form (let alone have a ``sampling-friendly'' structure).
For this reason, we will focus on \emph{simple strategies}, \ie probability distributions with a piecewise constant density.

\begin{definition}
\label{def:simple}
A mixed strategy on $\points$ is called \emph{simple} if it admits a density function
of the form
$\simple = \sum_{\iSet=1}^{\nSets} \weight_{\iSet}\one_{\set_{\iSet}}$ for a collection of
weights $\weight_{\iSet} > 0$, $\iSet=1,\dotsc,\nSets$,
and
mutually disjoint $\leb$-measurable subsets $\set_{\iSet}$ of $\points$ ($\set_{\iSet}\cap\set_{\jSet} = \varnothing$ for $\iSet\neq\jSet$)
such that
$\int_{\points}\simple = \sum_{\iSet} \weight_{\iSet} \leb_{\iSet}(\set_{\iSet}) = 1$.
The space of simple strategies on $\points$ will be denoted by $\simples(\points)$,
and
the expectation of a function $\obj\from\points\to\R$ under $\simple$ will be written as
\(
\txs
\braket{\obj}{\simple}
	\defeq \ex_{\point\sim\simple}[\obj(\point)]
	= \sum_{\iSet=1}^{\nSets} \weight_{\iSet} \int_{\set_{\iSet}} \obj(\point) \dd\leb(\point)
\)
\end{definition}

Owing to their decomposable structure, simple strategies are relatively easy to sample from,
and
they can approximate general distributions on $\points$ to arbitrary precision \textendash\ formally, they are dense in the weak topology of (regular) probability measures on $\points$ \citep[Chap.~2]{Fol99}.
On the other hand, this ``universal approximation'' guarantee comes at the cost of an increased number of supporting sets $\set_{\iSet}$, $\iSet=1,\dotsc,\nSets$.
In particular, there is no ``free lunch'':
when $\nSets$ grows large, sampling from a simple strategy can become computationally expensive \textendash\ if not intractable \textendash\ so we will pay particular attention to the support of such strategies.
\smallskip

\begin{remark}
To facilitate sampling, we will also consider strategies of the form $\simple = \sum_{\iSet=1}^{\nSets} \weight_{\iSet} \psi_{\set_{\iSet}}$ where $\psi_{\set}$ is supported on $\set$ and can be sampled cheaply \textendash\ \eg $\psi_{\set}$ could be a suitably weighted Dirac distribution on a specific point of $\set$.
Strategies of this type are not \emph{stricto sensu} ``simple'', but our results will also cover this case, \cf \cref{sec:hierarchical}. 
\end{remark}

\subsection{Regret: static and dynamic}

Going back to the learner's sequence of play, we will assume that, at each stage $\run=\running$, the learner picks an action $\choice_{\run}\in\points$ based on a simple strategy $\policy_{\run} \in \simples$, and receives the reward $\pay_{\run}(\point_{\run})$.
The \emph{regret} of the policy $\policy_{\run}$ against a \emph{benchmark action} $\point\in\points$ is then defined as the difference between the player's mean cumulative payoff under $\policy_{\run}$ and $\point$ over a horizon of $\nRuns$ rounds.
Formally, we have
\begin{align}
\label{eq:reg-test}
\reg_{\point}(\nRuns)
	&\defeq \insum_{\run=\start}^{\nRuns}
		\ex_{\choice_{\run} \sim \policy_{\run}} \bracks{ \pay_{\run}(\point) - \pay_{\run}(\choice_{\run}) }.
\intertext{%
Moreover, letting $\sol \in \argmax_{\point\in\points} \sum_{\run=\start}^{\nRuns} \pay_{\run}(\point)$ be the ``best fixed action in hindsight'' over the horizon $\nRuns$, we also define the learner's \emph{static regret} as
} 
\label{eq:reg}
\txs
\reg(\nRuns)
	&\defeq \reg_{\sol}(\nRuns)
	= \max\nolimits_{\point\in\points} \reg_{\point}(\nRuns).
\intertext{%
Finally, to relax the requirement of using a ``fixed'' action as a comparator, we will also consider the learner's \emph{dynamic regret}, defined here as
} 
\label{eq:reg-dyn}
\dynreg(\nRuns)
	&\defeq \insum_{\run=\start}^{\nRuns} \max_{\point\in\points}
		\ex_{\choice_{\run} \sim \policy_{\run}} \bracks{ \pay_{\run}(\point) - \pay_{\run}(\choice_{\run}) },
\end{align}
\ie as the difference between the player's mean cumulative payoff and that of the best sequence of actions $\sol_{\run} \in \argmin_{\point} \pay_{\run}(\point)$ over the horizon of play $\nRuns$.
Of course, in regard to its static counterpart, the agent's dynamic regret is considerably more ambitious, and achieving sublinear dynamic regret is not always possible;
we examine this issue in detail in \cref{sec:results}.

In both cases, it should also be clear that there is no simple strategy that can match the exact performance of the ``best'' action ($\sol$ or $\sol_{\run}$, depending on the context).
For example, consider the static optimization problem $\pay_{\run}(\point) = 1 - \point^{2}/2$ with $\point\in\points = [-1,1]$:
then, any simple strategy $\simple\in\simples$ would yield a payoff strictly less than $1$ at each round because it is sampling with probability $1$ points other than $0$.
Nevertheless, the following lemma shows that the propagated error on the regret can be made arbitrarily small:

\begin{lemma}
\label{lem:point2simple}
Let $\nhd$ be a neighborhood of $\point\in\points$.
Then, for every simple strategy $\simple\in\simples$ supported on $\nhd$, we have
\begin{equation}
\label{eq:point2simple}
\reg_{\point}(\nRuns)
	\leq \lips \diam(\nhd) \nRuns
		+ \insum_{\run=\start}^{\nRuns} \braket{\pay_{\run}}{\simple - \simple_{\run}}
\end{equation}
\end{lemma}

\begin{proof}
By \cref{asm:pay}, we have $\pay_{\run}(\point) \leq \pay_{\run}(\pointalt) + \lips \norm{\point - \pointalt} \leq \pay_{\run}(\pointalt) + \lips \diam(\nhd)$ for all $\pointalt\in\nhd$.
Hence, letting $\pointalt\sim\simple$ and expectations on both sides, we get $\pay_{\run}(\point) \leq \braket{\pay_{\run}}{\simple} + \lips\diam(\nhd)$.
Our claim then follows by summing over $\run$ and invoking the definition of the regret.
\end{proof}

\begin{remark}
We note here that the bound \eqref{eq:point2simple} does not need the full capacity of the Lipschitz continuity framework;
in fact, it continues to hold under much less restrictive notions, such as the weak one-sided continuity condition of \citet{BMSS11}.
Nevertheless, in the sequel we will maintain the assumption of Lipschitz continuity for simplicity.
\end{remark}

\begin{remark}
We should also state here that, in the sequel, $\nhd$ will be chosen small relative to $\nRuns$, so the term in \eqref{eq:point2simple} becomes sublinear in the analysis.
In more detail, in the proof of our main regret bounds, \cref{lem:point2simple} will be applied several times, over windows of different lengths, and $\nhd$ will be chosen at each window to be a progressively smaller set.
The exact mechanism is detailed in \cref{app:hierarchical}.
\end{remark}

\section{\Acl{DA} with an explicit cover}
\label{sec:dualavg}

\DeclarePairedDelimiterXPP{\tnorm}[1]{}{\lVert}{\rVert}{_{\run}}{#1}		
\newmacro{\tstart}{\run_{1}}
\newmacro{\tend}{\run_{2}}

%
%
%

To build some intuition for the analysis to come, we begin by adapting the \acdef{DA} algorithm of \citet{Nes09} to the (infinite) space of simple strategies with an explicit cover.
This will allow us to introduce the relevant notions that we will need in the sequel, namely the \emph{range} of an estimator and the \emph{Fisher information metric}.

\subsection{Basic setup}

Let $\cover = \{ \set_{1},\dotsc,\set_{\nSets} \}$ be a measurable partition of $\points$ with nontrivial covering sets, \ie $\leb(\set) > 0$ and $\set\cap\setalt = \varnothing$ for all $\set,\setalt\in\cover$ with $\set\neq\setalt$.
In particular, this implies that every point $\point\in\points$ belongs to a unique element of $\cover$, denoted below by $\set_{\point}$.
Since the elements of $\cover$ cover $\points$ in an unambiguous way, we will refer to $\cover$ as an \emph{explicit cover} of $\points$.
This cover will be assumed fixed throughout this section.

In terms of sampling actions from $\points$, the above also gives rise to a set of \textit{simple strategies} supported on $\cover$, namely
\begin{equation}
\label{eq:simple-cover}
\txs
\simples_{\cover}
	= \setdef*{ \sum_{\set} \weight_{\set} \one_{\set}}{\weight_{\set} \geq 0, \sum_{\set} \weight_{\set} \leb(\set) = 1 }
\end{equation}
Geometrically, it will be convenient to interpret $\simples_{\cover}$ as a simplex embedded in the space of all test functions $\phi\from\points\to\R$ that are piecewise constant on the covering sets of $\cover$.
Since such functions may be viewed equivalently as functions $\phi\from\cover\to\R$, we will denote this function space by $\pspace$.

Moving forward, we will assume that the learner is sampling from $\points$ with simple strategies taken from $\simples_{\cover}$, and we will write
\(
\txs
\simple_{\set}
	\defeq \prob_{\point\sim\simple}(\point\in\set)
	= \int_{\set}\simple
	= \weight_{\set}\leb(\set)
\)
for the probability of choosing an element of $\set$ under $\simple$.
Accordingly, our non-convex learning framework may be encoded in more concrete terms as follows:
\begin{enumerate*}
[(\itshape i\hspace*{1pt}\upshape)]
\item
at each stage $\run=\running$, the adversary chooses (but does not reveal) a payoff function $\pay_{\run}\from\points\to[0,\rbound]$;
\item
the learner selects an action $\choice_{\run}\in\points$ based on some simple strategy $\state_{\run}$ supported on $\cover$;
and
\item
the corresponding reward $\pay_{\run}(\choice_{\run})$ is received by the learner and the process repeats.
\end{enumerate*}

As an algorithmic template for learning in this setting, we will consider an adaptation of the classical \acl{DA} algorithm of \citet{Nes09}.
Specifically, we will focus on an online policy that we call \acdef{DAX}, and which is defined recursively as
\begin{equation}
\label{eq:DAX}
\tag{DAX}
\begin{aligned}
\score_{\run+1}
	&= \score_{\run} + \model_{\run}
	\\
\choice_{\run+1}
	\sim \state_{\run+1}
	&= \mirror(\temp_{\run+1} \score_{\run+1})
\end{aligned}
\end{equation}
where

\begin{enumerate}

\item
$\model_{\run}\in\dspace$ is an \emph{estimate} \textendash\ or \emph{model} \textendash\ of the otherwise unobserved payoff function $\pay_{\run}$ of stage $\run$.

\item
$\score_{\run}\in\dspace$ is an auxiliary \emph{scoring function} that aggregates previous payoff models \textendash\ so $\score_{\run}(\point)$ indicates the learner's propensity of choosing $\point\in\points$ at stage $\run$.

\item
$\temp_{\run} > 0$ is a ``learning rate'' parameter that adjusts the sharpness of the learning process.

\item
$\mirror\from\dspace\to\simples_{\cover}$ is a \emph{choice map} that transforms scoring functions $\score_{\run}\in\dspace$ into simple strategies $\state_{\run} \in \simples_{\cover}$.
\end{enumerate}

Each component of the method is discussed in detail below.
We also note that this method is often referred to as \acdef{FTRL}, \cf \citet{SS11,SSS06}.
Our choice of terminology follows \citet{Nes09} and \citet{Xia10}.

\subsection{The choice map}

We begin by detailing the method's ``choice map'' $\mirror\from\dspace \to \simples_{\cover}$ which determines action choice probabilities based on the ``score function'' $\score_{\run}(\point)$.
With this in mind, we will focus on a class of ``regularized strategies'' that output at each stage a simple strategy $\state_{\run}\in\simples_{\cover}$ that maximizes the learner's expected score minus a regularization penalty.

Specifically, we will consider choice maps of the form
\begin{equation}
\label{eq:choice}
\mirror(\dpoint)
	= \argmax_{\simple\in\simples_{\cover}} \braces{ \braket{\dpoint}{\simple} - \hreg(\simple) }
	\quad
	\text{for all $\dpoint\in\dspace$},
\end{equation}
where the \emph{regularizer} $\hreg\from\simples_{\cover}\to\R$ is assumed to be continuous and strictly convex on $\simples_{\cover}$.
To streamline our presentation, we will further assume that $\hreg$ is \emph{decomposable}, \ie it can be written as $\hreg(\simple) = \sum_{\set\in\cover} \hdec(\simple_{\set})$ for some strictly convex, $C^{2}$-smooth function $\hdec\from(0,1]\to\R$.
Two widely used examples are as follows:

\begin{example}
[Negentropy]
\label{ex:entropy}
Consider the \emph{entropic kernel} $\hdec(\point) = \point\log\point$ with the continuity convention $0\log0 = 0$.
Then, by a standard calculation, the associated choice map is given by the logit choice model
\begin{equation}
\label{eq:logit}
\logit(\dpoint)
	= \frac{\exp(\dpoint)}{\int_{\points} \exp(\dpoint)},
\end{equation}
where $\dpoint \equiv \dpoint(\point)$ is an arbitrary piecewise constant function on $\cover$.
The entropic regularizer has a very long history in the field of (online) optimization;
for a (highly incomplete) list of references, see \citet{NY83}, \citet{ACBFS95,ACBFS02}, \citet{BecTeb03}, \citet{SS11}, \citet{BCB12}, \citet{AHK12}, \citet{MerSta18}, \citet{KSU19}, \citet{Sli19}, \citet{PS21}, and references therein.
\end{example}

\begin{example}
[Log-barrier]
\label{ex:Burg}
Another important example is the \emph{log-barrier} (or \emph{Burg entropy}) kernel $\hdec(\point) = -\log\point$.
In this case, the associated choice map does not admit a closed form expression, but it can be calculated by a binary search algorithm in logarithmic time.%
\footnote{This is done by noting that any solution of the defining maximization problem \eqref{eq:choice} would have to satisfy the first-order optimality condition $\sum_{\set\in\cover} (\xi - \dpoint_{\set})^{-1} = 1$ for some $\xi > \max_{\set} \dpoint_{\set}$ (in which region the function being searched is strictly decreasing).}
This choice has deep links to Karmarkar's ``affine scaling'' method for linear programming \citep{Kar90,VMF86}, \cf \citet{ABB04}, \citet{BBT17}, \citet{MS16,MerSan18}, \citet{BMSS19}, \citet{ABM19,ABM21}, and references therein.
For a recent use of the log-barrier function in the context of stochastic and/or contextual multi-armed bandit problems, see \citet{WL18}, \citet{PL19}, and \citet{ACGL+19}.
\end{example}

\subsection{Estimators}

The second basic ingredient of \eqref{eq:DAX} is the estimate $\model_{\run}$ of the learner's payoff function $\pay_{\run}$ at time $\run$.
Since we are working with a fixed cover $\cover$ of $\points$, the estimator $\model_{\run}$ may not exceed the cover's granularity, which is why we require $\model_{\run}$ to be piecewise constant on $\cover$ \textendash\ \ie $\model_{\run}\in\dspace$.

Overall, we will measure the quality of $\model_{\run}$ as an estimator by means of the corresponding error process
\(
\error_{\run}
	= \model_{\run} - \pay_{\run}
\)
which is assumed to capture all sources of uncertainty and lack of precision in the learner's estimation process.
To differentiate further between random (zero-mean) and systematic (nonzero-mean) errors, we will decompose $\error_{\run}$ as
\begin{equation}
\label{eq:error}
\error_{\run}
	= \noise_{\run}
		+ \bias_{\run},
\end{equation}
where
$\bias_{\run} = \exof{\error_{\run} \given \filter_{\run}}$ denotes the \emph{bias} of the estimator,
and
$\noise_{\run} = \error_{\run} - \bias_{\run}$ the inherent \emph{random noise} (so $\exof{\noise_{\run} \given \filter_{\run}} = 0$ for all $\run$).
In terms of measurability, these processes are all conditioned on the history $\filter_{\run} \defeq \filter(\state_{1},\dotsc,\state_{\run})$ of the learner's policy up to \textendash\ and including \textendash\ stage $\run$.
Thus, in terms of the sequence of events described earlier, $\state_{\run}$ is $\filter_{\run}$-measurable (by definition), but $\choice_{\run}$, $\error_{\run}$, $\noise_{\run}$ and $\bias_{\run}$ are not.

For concreteness, we provide some examples below:

\begin{example}
[\Acl{IWE}]
Motivated by the literature on \aclp{MAB} \citep{BCB12,Sli19,LS20}, a natural way to reconstruct $\pay_{\run}$ is via the \acli{IWE}\acused{IWE}
\begin{equation}
\label{eq:IWE}
\tag{\acs{IWE}}
\model_{\run}(\point)
	= \rbound
		- \frac{\rbound - \pay_{\run}(\choice_{\run})}{\state_{\set_{\run},\run}} \oneof{\point\in\set_{\run}},
\end{equation}
where $\set_{\run} \defeq \set_{\point_{\run}}$ denotes the element of $\cover$ containing the sampled action $\choice_{\run}$, and $\rbound$ is one upper bound of the learner's rewards.
This particular formulation of \eqref{eq:IWE} is known as ``loss-based'';
other normalizations are possible but this is the most widely used one when considering sampling policies based on \acl{EW} algorithms \citep{Sli19}.
\end{example}

\begin{example}
[\Acl{IWE3}]
One shortfall of \eqref{eq:IWE} is that it requires knowledge of the upper bound $\rbound$ for the learner's rewards.
When this is not known, a suitable alternative is to introduce an \emph{explicit exploration} parameter $\mix_{\run} > 0$ in the learner's sampling strategy $\state_{\run}$.
This means that the learner now chooses an action $\choice_{\run}\in\cover$ according to the perturbed strategy
\(
\est\state_{\run}
	= (1-\mix_{\run}) \state_{\run} + \mix_{\run} \unif_{\cover},
\)
where $\unif_{\cover} = \abs{\cover}^{-1} \sum_{\set\in\cover} \leb(\set)^{-1} \one_{\set}$ denotes the uniform distribution on $\cover$.
The \acli{IWE3}\acused{IWE3} is then defined as
\begin{equation}
\label{eq:IWE3}
\tag{\acs{IWE3}}
\model_{\run}(\point)
	= \frac{\pay_{\run}(\choice_{\run})}{\est\state_{\set_{\run},\run}} \oneof{\point\in\set_{\run}}
\end{equation}
with $\set_{\run} \defeq \set_{\choice_{\run}}$ as above.
In contrast to \eqref{eq:IWE}, the estimator \eqref{eq:IWE3} has bias and variance bounded respectively as $\exof{\bias_{\run}} = \bigoh(\mix_{\run})$ and $\exof{\noise_{\set,\run}^{2}} = \bigoh(1/\mix_{\run})$, \ie both can be controlled by tuning $\mix_{\run}$.
This provides additional flexibility relative to \eqref{eq:IWE}, but the introduction of the explicit exploration parameter $\mix_{\run}$ often ends up having a negative impact on the regret \citep{Sli19}, an important disadvantage.
\end{example}

Other estimators have also been used in the literature, such as \emph{implicit} exploration and its variants \citep{KNVM14}.
For posterity, we only note that the set of possible values $\region \defeq \union_{\run} \im(\model_{\run}) \subseteq \dspace$ attained by an estimator will play an important role in the sequel.
When the estimator is understood from the context, we will refer to this image set as its \emph{range};
in the examples above, we have:
\begin{enumerate}
\item
For \eqref{eq:IWE}:
	$\region = (-\infty,\rbound]^{\cover}$.
\item
For \eqref{eq:IWE3}:
	$\region = \R_{+}^{\cover}$.
\end{enumerate}
We will return to this point in the next section.

\subsection{Strong convexity and the Fisher metric}

Deriving explicit regret guarantees for \acl{DA} methods is typically contingent on the method's regularizer being \emph{strongly convex} \citep{SS11,BCB12}.
Formally, strong convexity posits that there exists some $\hstr>0$ such that, for all $\simple,\simplealt\in\simples$ and all $\coef\in[0,1]$, we have
\begin{equation}
\label{eq:strong}
\begin{aligned}
\hreg(\coef\simple + (1-\coef)\simplealt)
	&\leq \coef \hreg(\simple)
		+ (1-\coef) \hreg(\simplealt)
	\\
	&- \frac{\hstr}{2} \coef(1-\coef) \norm{\simple - \simplealt}^{2}
\end{aligned}
\end{equation}
In the above, $\norm{\cdot}$ denotes an arbitrary reference norm on $\pspace$, usually taken to be the Euclidean norm $\twonorm{\cdot}$ or the Manhattan $L^{1}$ norm $\onenorm{\cdot}$.
However, in our case, seeing as we are comparing \emph{probability distributions}, an arbitrary reference norm does not seem particularly adapted to the problem at hand.

Instead, when dealing with probability distributions,
it is common to measure the distance of $\simplealt$ relative to $\simple$ via the \emph{Fisher information metric}, which is typically used to compute the informational difference between probability distributions.
In our context, the Fisher metric is defined for all $\simple,\simplealt\in\simples_{\cover}$ with $\simple \ll \simplealt$ as
\begin{equation}
\label{eq:Fisher}
\norm{\simplealt - \simple}_{\simple}^{2}
	= \int_{\points} \bracks*{\frac{d(\simplealt - \simple)}{d\simple}}^{2} \dd\simple
	= \sum_{\set\in\cover} \frac{(\simplealt_{\set} - \simple_{\set})^{2}}{\simple_{\set}}.
\end{equation}
We will then posit the following strong convexity requirement relative to the Fisher metric
\begin{equation}
\label{eq:strong-F}
\begin{aligned}
\hreg(\coef\simple + (1-\coef)\simplealt)
	&\leq \coef \hreg(\simple)
		+ (1-\coef) \hreg(\simplealt)
	- \frac{\hstr}{2} \coef(1-\coef) \norm{\simple - \simplealt}_{\simple}^{2}
\end{aligned}
\end{equation}
for all $\simple,\simplealt\in\simples$ and all $\coef\in[0,1]$.
Since this is a non-standard requirement, we proceed with an example.

\begin{example}
The Burg entropy $\hreg(\point) = -\sum_{\set\in\cover} \log\simple_{\set}$ is $1$-strongly convex relative to the Fisher metric.
Indeed, since $\hreg$ is smooth, the strong convexity requirement for $\hreg$ with $\hstr=1$ can be rewritten as
\(
\breg_{\mathrm{IS}}(\simplealt,\simple)
	\geq \frac{1}{2} \sum_{\set\in\cover} (\simplealt_{\set} - \simple_{\set})^{2} / \simple_{\set}
\)
where $\breg_{\mathrm{IS}}(\simplealt,\simple) = \sum_{\set\in\cover} \bracks{\simplealt_{\set}/\simple_{\set} - \log(\simplealt_{\set}/\simple_{\set}) - 1}$ denotes the Itakura\textendash Saito distance on $\simples_{\cover}$.
Our claim then follows from \citet[Ex.~4]{ABM20}.
\end{example}

The key implication of Fisher strong convexity for our analysis is the following characterization:

\begin{restatable}{lemma}{smooth}
\label{lem:smooth}
Let $\hconj(\dpoint) = \max_{\simple\in\simples_{\cover}} \{ \braket{\dpoint}{\simple} - \hreg(\simple) \}$ be the convex conjugate of $\hreg$.
The following are equivalent:
\begin{enumerate}
[left=\parindent]
\item
$\hreg$ satisfies \eqref{eq:strong-F}.
\item
$\hconj$ is $(1/\hstr)$-Lipschitz smooth relative to the dual Fisher norm $\norm{\dpoint}_{\simple,\ast}^{2} = \sum_{\set\in\cover} \simple_{\set} \dpoint_{\set}^{2}$ on $\dspace$;
specifically, for all $\dpoint,\dvec\in\dspace$ and $\ppoint = \mirror(\dpoint)$, we have
\begin{equation}
\label{eq:smooth-F}
\hconj(\dpoint + \dvec)
	\leq \hconj(\dpoint)
		+ \braket{\dvec}{\ppoint}
		+ \frac{1}{2\hstr} \norm{\dvec}_{\ppoint,\ast}^{2}.
\end{equation}
\end{enumerate}
\end{restatable}

\Cref{lem:smooth} mirrors the well-known equivalence between strong convexity in the primal and Lipschitz smoothness in the dual \citep{SS11,BCB12}.
However, we must stress here that the norms in \eqref{eq:strong-F} are \emph{not} global, but \emph{strategy-dependent} \textendash\ in effect, they comprise a Riemannian metric on the set of simple strategies $\simples_{\cover}$.
This is a crucial difference with the standard analysis of \acl{DA}, and it allows for much finer control of the learning process as it unfolds \textendash\ precisely because the base distribution $\ppoint = \mirror(\dpoint)$ is not ignored in the process.

We close this section by noting that the entropic regularizer of \eqref{ex:entropy} \emph{does not} satisfy \eqref{eq:strong-F};
we provide an explicit discussion of this point in the supplement.
However, as we also show in the supplement, it \emph{does} satisfy the Lipschitz smoothness requirement \eqref{eq:smooth-F} for all $\dvec\in\dspace$ that are ``upper-bounded'', \ie $\sup_{\set\in\cover} \dvec_{\set} \leq \vbound$ for some $\vbound\in\R$.
From an algorithmic viewpoint, this relaxation of \eqref{eq:strong-F} will play a pivotal role in the sequel, so we encode it as follows:
\begin{definition}
\label{def:smooth}
Let $\region$ be a nonempty convex subset of $\dspace$.
We say that $\hreg$ is $\hstr$-\emph{tame} relative to $\region$ if \eqref{eq:smooth-F} holds for all $\dpoint\in\dspace$ and all $\dvec\in\region$.
\end{definition}

Clearly, by \cref{lem:smooth}, any regularizer satisfying \eqref{eq:strong-F} is tame relative to any subset of $\dspace$ (including $\dspace$ itself).
By contrast, as we mentioned above, the entropic regularizer of \cref{ex:entropy} is $1$-tame over the region $\region = \setdef{\dpoint\in\dspace}{\dpoint_{\set} \leq 1}$, but \emph{it is not tame} over all of $\dspace$.
In the analysis to come, we will see that this property introduces an intricate interplay between the two principal components of \eqref{eq:DAX}, namely the choice of regularizer $\hreg$ and the estimator $\model$.
Unless explicitly mentioned otherwise, in the rest of this section we will assume that $\region$ is fixed and $\hreg$ is $\hstr$-tame relative to $\region$.

\subsection{Regret analysis}

The key element in our analysis will be to control the ``divergence'' between a scoring function $\score_{\run}$ and a comparator strategy $\base\in\simples_{\cover}$.
Because these two elements live in different spaces, we introduce below the \emph{Fenchel coupling}
\begin{equation}
\label{eq:Fench}
\fench(\base,\dpoint)
	= \hreg(\base)
		+ \hconj(\dpoint)
		- \braket{\dpoint}{\base},
\end{equation}
for all $\base\in\simples_{\cover}$, $\dpoint\in\dspace$.
Clearly, by the Fenchel-Young inequality, we have $\fench(\base,\dpoint) \geq 0$ with equality if and only if $\mirror(\dpoint) = \base$.
More to the point, as we show in the supplement, the Fenchel coupling enjoys the following growth property:

\begin{restatable}{lemma}{Fenchel}
\label{lem:Fenchel}
For all $\dpoint\in\dspace$ and all $\dvec\in\region$, we have
\begin{subequations}
\label{eq:Fench-bound}
\begin{align}
\label{eq:Fench-bound-sharp}
\fench(\base,\dpoint+\dvec)
	&= \fench(\base,\dpoint)
		+ \braket{\dvec}{\ppoint - \base}
		+ \fench(\ppoint,\dpoint+\dvec)\!\!
	\\
\label{eq:Fench-bound-norm}
	&\leq \fench(\base,\dpoint)
		+ \braket{\dvec}{\ppoint - \base}
		+ \frac{1}{2\hstr} \norm{\dvec}_{\ppoint,\ast}^{2}
\end{align}
\end{subequations}
where $\ppoint = \mirror(\dpoint)$.
\end{restatable}

Using \eqref{eq:Fench-bound}, we will analyze the regret properties of \eqref{eq:DAX} via the \emph{$\temp_{\run}$-deflated coupling}
\begin{equation}
\label{eq:energy}
\energy_{\run}
	= \frac{1}{\temp_{\run}} \fench(\base,\temp_{\run}\score_{\run}).
\end{equation}
Doing so leads to the following result:

\begin{restatable}{lemma}{energybound}
\label{lem:energy}
Suppose that \eqref{eq:DAX} is run with an estimator with range $\region$.
For all $\run=\running$, we have
\begin{align}
\label{eq:energy-bound-tight}
\energy_{\run+1}
	\leq \energy_{\run}
		&+ \braket{\model_{\run}}{\state_{\run} - \simple}
		+ \parens*{\temp_{\run+1}^{-1} - \temp_{\run}^{-1}} \bracks{\hreg(\base) - \min\hreg}
	\notag\\
		&+ \temp_{\run}^{-1} \fench(\state_{\run},\temp_{\run}\score_{\run+1}).
\end{align}
If, in addition, $\hreg$ is $\hstr$-tame relative to $\region$, the last term in \eqref{eq:energy-bound-tight} is bounded as
\begin{equation}
\label{eq:energy-bound-norm}
\temp_{\run}^{-1} \fench(\state_{\run},\temp_{\run}\score_{\run+1})
	\leq \temp_{\run}/(2\hstr) \tnorm{\model_{\run}}^{2},
\end{equation}
where $\tnorm{\cdot}$ is the dual Fisher norm $\tnorm{\dvec}\defeq \norm{\dvec}_{\state_{\run},\ast}$.
\end{restatable}

Thus, telescoping \cref{lem:energy}, we obtain the bound below.
\begin{restatable}{proposition}{templateDAX}
\label{prop:reg-DAX}
The regret incurred relative to $\simple\in\simples_{\cover}$ over the interval $\runs = \{\tstart,\dotsc,\tend-1\}$ is bounded as
\begin{align}
\label{eq:reg-bound}
\reg_{\base}(\runs)
	&\txs
	\leq \energy_{\tstart} - \energy_{\tend}
		+ \parens*{\temp_{\tend}^{-1} - \temp_{\tstart}^{-1}} \bracks{\hreg(\base) - \min\hreg}
	\notag\\
	&+ \sum_{\run\in\runs} \braket{\error_{\run}}{\state_{\run} - \base}
	+ \frac{1}{2\hstr} \sum_{\run\in\runs} \temp_{\run} \tnorm{\model_{\run}}^{2}.
\end{align}
\end{restatable}

We are finally in a position to state our main regret guarantees for \eqref{eq:DAX}.
For generality, we state our result with a generic estimator $\model_{\run}$ enjoying the following bounds:
\begin{subequations}
\label{eq:sigbounds}
\begin{alignat}{3}
\label{eq:bbound}
a)
	\quad
	&\textit{Bias:}
	&\hspace{1em}
	&\abs{\braket{\bias_{\run}}{\simple}}
	\leq \bbound_{\run}
	\\
\label{eq:mbound}
b)
	\quad
	&\textit{Mean square:}
	&\hspace{1em}
	&\exof{\tnorm{\model_{\run}}^{2} \given \filter_{\run}}
		\leq \mbound_{\run}^{2}
	\hspace{2em}
\end{alignat}
\end{subequations}
for all $\run=\running$, and all $\simple\in\simples_{\cover}$.
We stress here that the use of the Fisher metric in \eqref{eq:sigbounds} is \emph{crucial}:
for example, the \ac{IWE} estimator satisfies \eqref{eq:mbound} with $\mbound_{\run} = \bigoh(\rbound^{2} |\sets|)$ (where $|\sets|$ is the size of the underlying partition)
but it does not satisfy this bound for \emph{any} global norm.
Again, the reason for this is that the dual Finsler norm can be considerably smaller than any other global norm, depending on the information content of $\state_{\run}$.

This feature plays a key role in deriving the regret of \eqref{eq:DAX}:

\begin{restatable}{theorem}{regretDAX}
\label{thm:DAX}
Suppose that \eqref{eq:DAX} is run with assumptions as in \cref{prop:reg-DAX}.
Then the learner's regret is bounded as
\begin{align}
\label{eq:reg-DAX}
\exof{\reg_{\base}(\nRuns)}
	&\leq \energy_{\start} - \energy_{\nRuns+1}
		+ \parens*{\temp_{\nRuns+1}^{-1} - \temp_{\start}^{-1}} \bracks{\hreg(\base) - \min\hreg}
	\notag\\
	&+ 2 \insum_{\run=\start}^{\nRuns} \bbound_{\run}
		+ \frac{1}{2\hstr} \insum_{\run=\start}^{\nRuns} \temp_{\run} \mbound_{\run}^{2}.
\end{align}
\end{restatable}


This theorem is proved in the supplement and constitutes the main ingredient for the analysis to come.


\section{\Acl{HDA}}
\label{sec:hierarchical}
\newcommand{\splitsched}{\mathcal{T}_{\mathrm{split}}}
\newcommand{\scheduler}{v}
\newcommand{\coverset}{\mathcal{S}}

In this section, we proceed to define the mechanism that we will use to recursively ``zoom-in'' on different regions of the state space.
This hierarchical approach is inspired by earlier works by \citet{BMSS11}, but with the crucial difference that we do not zoom in ``pointwise'' but ``dimension-wise''.
We explain all this in detail below.

\subsection{The splitting mechanism}

As in the case of \citet{BMSS11} and \citet{KSU08,KSU19}, the basic element of our construction is an infinite ``tree of coverings'', each of whose levels $\nSplits = 1,\dotsc$ defines a successively finer cover $\cover_{\nSplits}$ of $\points$ (\ie $\cover_{\nSplits} \subseteq \cover_{\nSplits+1}$ for all $\nSplits=1,\dotsc$). %
However, in contrast to these previous works, we do not consider \emph{binary} trees, but \emph{dyadic} ones;
specifically, each cover $\cover_\nSplits = \left\{\coverset_{\nSplits,i}\right\}_{i \leq  2^\nSplits}$ is defined inductively as follows:
\begin{enumerate*}
[(\itshape i\hspace*{1pt}\upshape)]
\item
$\cover_0 = \{\coverset_{0,1}\}=\{\points\}$;
\item
at specific stages of the learning process (that we define later), a \emph{splitting event} occurs, and each leaf\footnote{In a slight overload, we also write $\cover$ for the tree inducing the cover, and therefore refer to its components as \emph{leaves}} of the current cover
is split into $2$ sub-leaves as detailed below (refer also to \cref{fig:tree,fig:tree2} for intuition in the case $\vdim=2$).
\end{enumerate*}
We perform splitting events successively along each dimension in a round-robin manner, ensuring each node is split into two subnodes of equal volume.
Formally, for a given node $\coverset_{\nSplits,i}$, we define $\coverset_{\nSplits+1,2i-1}$ and $\coverset_{\nSplits+1,2i}$ as the two subsets obtained from splitting the leaf $\coverset_{\nSplits,i}$ in $2$ equally sized leaves using a hyperplane\footnote{Given a set $\set \subseteq \mathbb{R}^\vdim$ and a dimension $k \in \{1, \dots, \vdim\}$, an hyperplane with normal vector $e_k$ which splits $\set$ into two equally sized subsets exists by the intermediate value theorem and can be found efficiently by line search.} orthogonal to the canonical basis vector of $\mathbb{R}^\vdim$ number $\nSplits + 1 \ (\mathrm{mod} \ \vdim)$. 
We then have $\coverset_{\nSplits+1,2i}\cup \coverset_{\nSplits+1,2i-1}=\coverset_{\nSplits,i}$, $\coverset_{\nSplits+1,2i}\cap \coverset_{\nSplits+1,2i-1}=\varnothing$ and $\leb(\coverset_{\nSplits+1,2i})=\leb(\coverset_{\nSplits+1,2i-1})=\leb(\coverset_{\nSplits,i})/2$.

\begin{figure}[htbp]
\ctikzfig{Figures/tree-hierarchical2}
\vspace{-2cm}
\caption{Example of the 3 first \textit{splitting events} for $\points=[0,1]^2$}
\label{fig:tree}
\end{figure}

In the sequel, for any cover $\cover$, we write $\cover^{+}$ for its \emph{successor} cover, \ie the cover after a splitting event on $\cover$.

\begin{figure}[ht]
\ctikzfig{Figures/tree-hierarchical}
\caption{Example of a covering tree for the cube $\points=[0,1]^2$}
\label{fig:tree2}
\end{figure}

A crucial information for the sequel is the diameter of the leaves $\coverset_{\nSplits,i}$ of a given cover $\cover$, for which we make a geometric assumption similar to \citet[A1]{BMSS11}:

\begin{restatable}{assumption}{diamassumpt}
\label{assumpt:diameter}
There exists some $\diamconst >0$ such that
\begin{equation}
\label{eq:assumption-diameter}
\diam(\coverset_{\nSplits,i}) \leq \diamconst \diam(\points)2^{-\floor{\nSplits/\vdim}}.
\end{equation}
for all $\sigma \geq 0$ and $i \leq 2^\sigma$ as above.
\end{restatable}


\Cref{assumpt:diameter} only concerns the problem's domain $\points$, and it can be lifted by embedding $\points$ in a suitable box and then proceeding with a splitting schedule that follows a fixed volumetric mesh.
This approach could lead to leaves of different volume at each splitting event, which would in turn make the analysis more cumbersome.
The example below shows that $\diamconst$ can be easy to calculate in many cases:

\begin{example}[$\points = \vdim$-dimensional box]
In the particular case where $\points$ is an hyperrectangle with sides parallel to the canonical basis vectors of $\R^{\vdim}$, we have $\diam(\coverset_{\nSplits,i}) \leq  \diam(\points)2^{-\floor{\nSplits/\vdim}}$ and $\diamconst = 1$.
\end{example}


\subsection{The \acl{HDA} algorithm}

As a prelude to the definition of our algorithm, we introduce the following notions:
for all $\run=\running$,
\begin{enumerate*}
[(\itshape i\hspace*{1pt}\upshape)]
\item
$\cover_{\run}$ will denote the current cover at time $\run$;
\item
we will write $\nSplits_{\run}$ for the number of splitting events made prior to time $\run$ (so $\nSplits_{\run}$ is also the height of the tree $\cover_{\run}$);
and
\item
$\nSets_{\run} = 2^{\nSplits_t}$ will denote the number of leaves of $\cover_{\run}$.
\end{enumerate*}
Moreover, a \emph{splitting schedule} is an increasing sequence of integers $\splitsched = \{\run_1, \run_2, \dots \}$ such that we perform a splitting event at each round $\run \in \splitsched$. For convenience, we will rather manipulate \emph{scheduler sequences} $\{\scheduler_{\run}\}_{t\geq 1}$, \ie increasing real sequences that are uniquely mapped to a splitting schedule by
$\splitsched(\scheduler) = \{\run \geq 1\;\text{such that}\;\floor{\scheduler_{\run}} = \floor{\scheduler_{\run - 1}} + 1\}$. In the sequel and when the context is non ambiguous we may use the term \emph{splitting schedule} to refer to its associated \emph{scheduler sequence}.
We note that for all $\run$, these definitions imply $\floor{\scheduler_{\run}} = \nSplits_{\run}$, and that in the light of the relation stated in the previous subsection we have, for any $\set \in \cover_{\run}$, $\diam(\set) \leq 2\diam(\points) \nSets_{\run}^{-1/\vdim}$ and $\leb(\set) = \nSets_{\run}^{-1}  \leb(\points)$.

We are now in a position to define our learning algorithm in detail.
Its components are threefold:
\begin{enumerate*}
[(\itshape i\hspace*{1pt}\upshape)]
\item
a sequence of estimators $\model_{\run}$ with range $\region$;
\item
a regularizer that is $\hstr$-tame relative to $\region$;
and
\item
a splitting schedule $\splitsched(\scheduler)$ as above.
\end{enumerate*}
Then, the \acdef{HDA} is defined as
\begin{equation}
\label{eq:HDA}
\tag{HDA}
\begin{aligned}
\score_{\run+1}
	&\gets \score_{\run} + \model_{\run}
	\\
\choice_{\run+1}
	\sim \state_{\run+1}
	&\gets \mirror^{\cover_{\run}}(\temp_{\run+1} \score_{\run+1})
	\\
\cover_{\run+1}
	&\gets \new\cover_{\run} \text{if $\run\in\splitsched(\scheduler)$}
\end{aligned}
\end{equation}
where
$\mirror^\cover$ denotes the choice map induced by $\hreg$ for a given cover $\cover$ of $\points$ (by convention, we take $\cover_{0} = \{\points\}$),
$\temp_{\run}$ is a variable learning rate sequence,
and
we implicitly treat $\model_{\run}$ and $\score_{\run}$ as elements $\mathbb{R}^{\cover_t}$ and $\state_{\run}$ as an element of $\simples_{\cover_t}$.

By construction, \eqref{eq:HDA} comprises a succession of applications of \eqref{eq:DAX} to sequences of successive rounds during which the underlying partition $\cover_{\run}$ stays the same (\ie in between two successive splitting events).
An important special case is the specific instance of \eqref{eq:HDA} obtained by the entropic kernel $\hdec(\point) = \point\log\point$ (\cf \cref{ex:entropy}) and the estimator \eqref{eq:IWE};
we will refer to this instance as the \acdef{HEW} algorithm.

\section{Analysis and results}
\label{sec:results}

\newmacro{\cst}{C}

In this section, we leverage the regret guarantees established in \cref{sec:dualavg} for \eqref{eq:DAX} to derive a template regret bound for \eqref{eq:HDA} \textendash\ and, in particular, for \ac{HEW}.

\para{Static regret}

Our template bound for \ac{HDA} is as follows.


\begin{restatable}{theorem}{regstat}
\label{thm:reg-stat}
The \ac{HDA} algorithm enjoys the regret bound
\begin{align}
\label{eq:reg-bound-stat}
\exof{\reg_{\point}(\nRuns)}
	&\leq \frac{\hvol(\nSets_\nRuns) + \cst_\hdec \log_2(\nSets_\nRuns)}{\temp_{\nRuns+1}}
	\notag\\
	&+ 2 \lips\diamconst\diam(\points) \insum_{\run=\start}^{\nRuns} \nSets_{\run}^{-1 / \vdim}
	\notag\\
	&+ 2 \insum_{\run=\start}^{\nRuns} \bbound_{\run}
		+ \frac{1}{2\hstr} \insum_{\run=\start}^{\nRuns} \temp_{\run} \mbound_{\run}^{2},
\end{align}
where $\nSets_{\run}$ is the number of sets in the partition $\sets_{\run}$,  
$\hvol(z) = z \hdec(1/z)$ for all $z>0$
and $\cst_\hdec$ is a constant depending only on $\hdec$.
In particular, if \eqref{eq:HDA} is run with
learning rate $\temp_{\run} \propto 1/\run^{\pexp}$, $\pexp\in(0,1)$,
a logarithmic splitting schedule $\scheduler_{\run} = \nSetsExp \log_2(\run)$
and
a sequence of estimators $\model_{\run}$ such that $\bbound_{\run} = \bigoh(1/\run^{\bexp})$ and $\mbound_{\run}^{2} = \bigoh(\run^{2\mexp})$ for some $\bexp,\mexp\geq0$,
then
\begin{equation}
\label{eq:reg-bound-stat-powers}
\exof{\reg(\nRuns)}
	= \bigoh(\hvol(\nRuns^{-\nSetsExp}) \nRuns^{\pexp} + \nRuns^{1-\nSetsExp / \vdim} + \nRuns^{1-\bexp} + \nRuns^{1+2\mexp-\pexp}).
\end{equation}
\end{restatable}

The proof of \cref{thm:reg-stat} is presented in detail in \cref{app:hierarchical} and hinges on applying \cref{prop:reg-DAX} to bound the regret of \eqref{eq:HDA} on each time window during which the algorithm maintains a constant cover of $\points$.
Aggregating these bounds provides a regret guarantee for \eqref{eq:HDA} over the entire horizon time of play;
however, since \eqref{eq:HDA} is \emph{not} restarted at each window, joining the resulting window-by-window bounds ends up being fairly delicate.
The main dificulties (and associated error terms in the regret) are as follows:
\begin{enumerate}
\item
A comparator for a given time frame may not be admissible for a previous time frame because the granularity of an antecedent cover may not suffice to include the comparator in question.
This propagates a ``resolution error'' that becomes smaller when the cover gets finer, but larger when the window gets longer.

\item
At every splitting event, the algorithm retains the same probability distribution over $\points$ (to avoid restart-forget effects).
However, this introduces a ``splitting residue'' because of the necessary correction in the learner's scores when the resolution of the cover increases.
\end{enumerate}

The above steps are made precise in \cref{app:hierarchical}, where we show how each of these contributing factors can be bounded in an efficient manner.
For the moment, we only note that the template bound of \cref{thm:reg-stat} can be used to derive tight regret bounds for particular instances of
\begin{enumerate*}
[(\itshape i\upshape)]
\item
the estimator sequence $\{\model_{\run}\}_{\run} $ with range $\region$;
\item
the regularizer $\hreg$ which is $\hstr$-\emph{tame} relative to $\region$;
and
\item
the splitting schedule $\{\scheduler_{\run}\}_{\run}$.
\end{enumerate*}

We carry all this out for the \ac{HEW} algorithm below:

\begin{restatable}{corollary}{regstatIWE}
\label{cor:reg-stat-IWE}
If \ac{HEW} is run with learning rate $\temp_{\run} \propto \run^{-\pexp}$ and
the logarithmic splitting schedule $\scheduler_{\run} = \nSetsExp \log_2(\run)$, the learner enjoys the bound
\begin{equation}
\label{eq:reg-bound-stat-powers-cor}
\exof{\reg(\nRuns)}
	= \bigoh(\nRuns^{\pexp} + \nRuns^{1-\nSetsExp / \vdim} + \nRuns^{1+\nSetsExp-\pexp}).
\end{equation}
In particular, if the algorithm is run with $\pexp = (\vdim + 1) / (\vdim + 2)$ and $\nSetsExp = \vdim / (\vdim +2)$ we obtain the bound 
\begin{equation}
\exof{\reg(\nRuns)}
	= \bigoh(\nRuns^{\frac{\vdim+1}{\vdim+2}}).
\end{equation}
\end{restatable}

\para{Dynamic regret guarantees}

We now show guarantees of \ac{HDA} in terms of the dynamic regret introduced in \eqref{eq:reg-dyn}. We would like to stress that the expected dynamic regret of an algorithm cannot be bounded without any restriction on the sequence of payoffs \citep{SS11}.
For this reason, dynamic regret guarantees are often stated in terms of the \emph{variation} of the payoff functions $\{\pay_{\run}\}_{\run}$, defined as follows \citep{BGZ15}
\begin{equation}
\label{eq:tvar}
\tvar
	\defeq \insum_{\run=\start}^{\nRuns} \supnorm{\pay_{\run+1} - \pay_{\run}},
\end{equation}
with the convention $\pay_{\run+1} = \pay_{\run}$ for $\run=\nRuns$.
We then have:
\begin{restatable}{theorem}{regdyn}
\label{thm:reg-dyn}
Suppose that \eqref{eq:HDA} is run with the negentropy kernel,
and assumptions as in \cref{thm:reg-stat}.
Then:
\begin{equation}
\label{eq:reg-bound-dyn}
\ex[\dynreg(\nRuns)]
	= \bigoh(\nRuns^{1+2\mexp-\pexp} + \nRuns^{1-\bexp} + \nRuns^{1-\nSetsExp / \vdim} + \nRuns^{2\pexp-2\mexp}\tvar).
\end{equation}
\end{restatable}

Finally, with judiciously chosen parameters, our template bound yields the following improvement over previous dynamic regret bounds in the literature:

\begin{restatable}{corollary}{regdynIWE}
\label{cor:reg-dyn-IWE}
Suppose that \ac{HEW} is run with splitting schedule $\scheduler_{\run} = \nSetsExp \log_2(\run)$ and learning rate $\temp_{\run} \propto 1/\run^{\pexp}$.
Then:
\begin{equation}
\label{eq:reg-bound-dyn-IWE}
\exof{\dynreg(\nRuns)}
	= \bigoh(\nRuns^{1+\nSetsExp-\pexp} + \nRuns^{1-\nSetsExp / \vdim} + \nRuns^{2\pexp-\nSetsExp}\tvar).
\end{equation}
Hence, if $\tvar = \bigoh(\nRuns^{\vexp})$ for some $\vexp<1$, setting $\pexp = (1 - \vexp)(\vdim +1)/(\vdim + 3)$ and $\nSetsExp =(1-\vexp)\vdim/(\vdim + 3) $ delivers 
\begin{equation}
\label{eq:reg-bound-dyn-IWE-tuned}
\exof{\dynreg(\nRuns)}=
\bigoh(\nRuns^{(\vdim + 2)/(\vdim + 3)}\tvar^{1/(\vdim+3)}).
\end{equation}
\end{restatable}

This result was conjectured by \citet{HMMR20} and, as far as we are aware, this is the first time it is achieved.
The main limiting factor in the kernel-based approach of \citet{HMMR20} is that it requires the introduction of an explicit exploration term;
in turn, this leads to an unavoidable extra term in the regret, and to suboptimal regret bounds.
The importance of the proposed splitting mechanism is that it does not require a kernel to smooth \eqref{eq:IWE}, and the use of the Fisher information metric allows us to control the variance of \eqref{eq:IWE} without introducing an explicit exploration error.

We suspect that the above bound is min-max optimal, but we are not aware of any lower bounds for the dynamic regret against non-convex Lipschitz losses \textendash\ this is actually stated as an open problem in the paper of \citet{BGZ15}.
In particular, the analysis of \citet[Theorem 2]{BGZ15} seems to suggest that, if an informed adversary can impose $\Omega(\nRuns^{q})$ \emph{static} regret, they can also impose $\Omega(T^{1/(2-q)} \budget_{\nRuns}^{(1-q)/(2-q)})$ \emph{dynamic} regret.
If this conjecture is true, this would mean that our bound is itself tight, because the static regret exponent $q=(\vdim+1)/(\vdim+2)$ is well known to be tight for Lipschitz losses.

\begin{remark}
We should also state here that \eqref{eq:HDA} is not parameter-free, as it implicitly requires knowledge of an upper bound for $\budget_{\nRuns}$.
As one of the reviewers pointed out, this requirement \textendash\ which was stated as an open problem in the work of \citet{BGZ15} \textendash\ has been partially resolved for (stochastic) multi-armed and contextual bandits by \citet{ACGL+19};
we are not aware of a similar result for adversarial online non-convex optimization problems.
\end{remark}


\para{Numerical experiments}

\begin{figure}[tbp]
\centering
\includegraphics[width=0.42\textwidth]{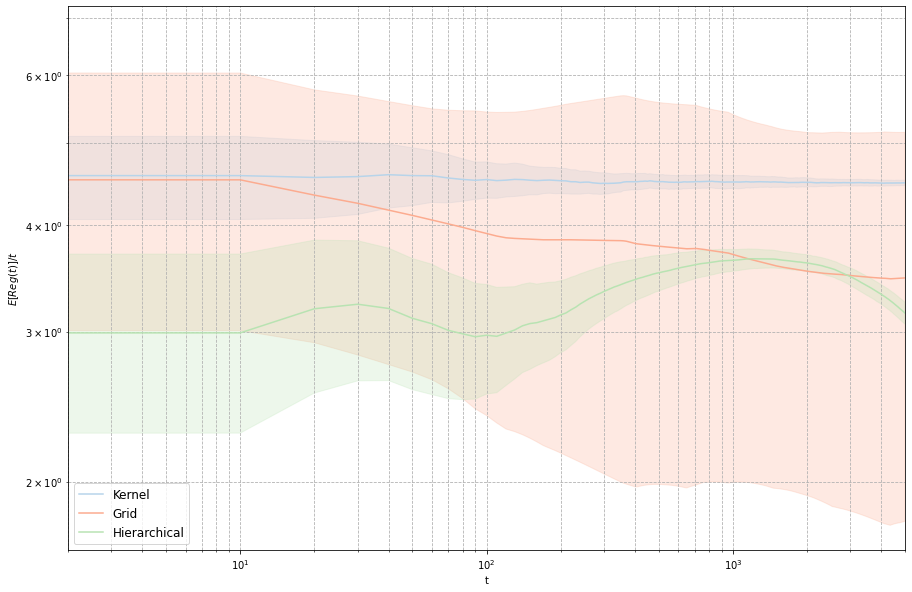}
\caption{Expected average regret over 92 realizations for each algorithm (solid line).
The shaded area represents one standard deviation from the mean. The three strategies compared are ''Hierarchical'' (ours), ''Grid'' being a naive discretized mesh on the search space, and  the ``Kernel'' strategy from \citep{HMMR20}}
\vspace{-2ex}
\label{fig:regret}
\end{figure}

For illustration purposes, \cref{fig:regret} provides some numerical experiments on different no-regret policies discussed in the rest of our paper. 
Specifically, we compared $3$ strategies, ``Grid'', ``Kernel'' and ``Hierarchical'', and plot the current instantaneous regret w.r.t. the current round $t$. The shaded area representing the instantaneous variance of such regret, each strategies being launched with multiple initialized seed (92). 
First the ``Hierarchical'' method is as outlined in \cref{sec:hierarchical} with parameters of the algorithm described below.
Second the ``Grid'' method involves partitioning the search space $\points$ into a regular grid of a given mesh-size. This \textit{a priori} discretization level constitute the algorithm hyperparameter. The ``Grid'' then treats the problem as a finite-armed bandit on the latter discretized search space, applying the EXP3 algorithm \citep{ACBF02}. Finally, the ``Kernel'' strategy is based on \citet{HMMR20}, using a squared-kernel based estimate. The adversarial function is analytic and randomly drawn, with known maximum.
We present the full details of our experiments in \cref{app:numerics}.


\appendix
\numberwithin{equation}{section}		
\numberwithin{lemma}{section}		
\numberwithin{proposition}{section}		
\numberwithin{theorem}{section}		


\section{Fisher regularizers and their properties}
\label{app:Fisher}

Our goal here is to formally state and prove some basic properties for the regularizer functions that underlie the definition of \eqref{eq:DAX}.
These properties are relatively well-known in the literature in the case where $\hreg$ is strongly convex relative to a global, reference norm;
however, the use of the Fisher information metric introduces a number of complications that necessitate a more careful treatment.

We begin by recalling the basic setup of \eqref{eq:DAX}, as formalized in \cref{sec:dualavg} for a fixed cover $\sets$ of $\points$.
In particular, we will write $\dspace$ for the space of piecewise constant functions on $\sets$
and
$\simples_{\sets}$ for the space of probability distributions supported on $\sets$.
Then, given $\tvec\in\pspace$ and $\dpoint \in \dspace$, we define respectively the primal and dual Fisher norm relative to $\simple\in\simples_{\sets}$ as
\begin{subequations}
\label{eq:Fisher-norm}
\begin{align}
\label{eq:Fisher-primal}
\norm{\tvec}_{\simple}^{2}
	&= \int_{\points} \frac{\tvec(\point)^{2}}{\simple(\point)} \dd\leb(\point)
	= \sum_{\set\in\sets} \frac{\tvec_{\set}^{2}}{\simple_{\set}},
	\\
\label{eq:Fisher-dual}
\norm{\dpoint}_{\simple,\ast}^{2}
	&= \int_{\points} \simple(\point) \dpoint(\point)^{2} \dd\leb(\point)
	= \sum_{\set\in\sets} \simple_{\set} \dpoint_{\set}^{2}.
\end{align}
\end{subequations}
We then have the following basic lemma.

\begin{lemma}
\label{lem:Fisher}
With notation as above, we have
\begin{enumerate}
\item
$\norm{\dpoint}_{\simple,\ast} = \max\setdef{\abs{\braket{\dpoint}{\tvec}}}{\norm{\tvec}_{\simple} = 1}$.
\item
$\norm{\cdot}_{\simple} \geq \onenorm{\cdot}$ and $\norm{\cdot}_{\simple,\ast} \leq \supnorm{\cdot}$.
\end{enumerate}
\end{lemma}

\begin{proof}
For the first part of our claim, an application of the Cauchy-Schwarz inequality gives
\begin{equation}
\abs{\braket{\dpoint}{\tvec}}
	= \abs*{\sum_{\set\in\sets} \tvec_{\set} \dpoint_{\set}}
	= \abs*{\sum_{\set\in\sets} \frac{\tvec_{\set}}{\sqrt{\simple_{\set}}} \cdot \sqrt{\simple_{\set}} \dpoint_{\set}}
	\leq \norm{\tvec}_{\simple} \cdot \norm{\dpoint}_{\simple,\ast}.
\end{equation}
Since equality is attained when $\tvec_{\set} \propto \simple_{\set} \dpoint_{\set}$, maximizing over the Fisher unit sphere $\norm{\tvec}_{\simple} = 1$ yields the desired result.

For the second part of our claim, a second application of the Cauchy-Schwarz inequality readily gives
\begin{equation}
\sum_{\set\in\sets} \frac{\tvec_{\set}^{2}}{\simple_{\set}}
	= \sum_{\set\in\sets} \simple_{\set} \cdot \sum_{\set\in\sets} \frac{\tvec_{\set}^{2}}{\simple_{\set}}
	\geq \parens*{ \sum_{\set\in\sets} \sqrt{\simple_{\set}} \frac{\abs{\tvec_{\set}}}{\sqrt{\simple_{\set}}} }^{2}
	= \onenorm{\tvec}^{2},
\end{equation}
\ie $\norm{\cdot}_{\simple} \geq \onenorm{\cdot}$, as claimed.
The inequality $\norm{\cdot}_{\simple,\ast} \leq \supnorm{\cdot}$ then follows by taking duals.
\end{proof}

To proceed, recall that the convex conjugate $\hconj$ of $\hreg$ is defined as
\begin{equation}
\label{eq:conj}
\hreg^{\ast}(\dpoint)
	= \sup_{\simple\in\simples_{\sets}} \{ \braket{\dpoint}{\simple} - \hreg(\simple) \}.
\end{equation}
Since $\hreg$ is assumed strongly convex relative to the Fisher information metric, \cref{lem:Fisher} shows that it is also strongly convex relative to the $\ell_{1}$-norm on $\simples_{\sets}$.
As a result, the supremum in \eqref{eq:conj} is always attained, and $\hreg^{\ast}(\dpoint)$ is finite for all $\dpoint\in\dspace$ \citep{BC17}.
Moreover, by standard results in convex analysis \citep[Chap.~26]{Roc70}, it follows that $\hreg^{\ast}$ is differentiable on $\dspace$;
finally, by Danskin's theorem \citep[Chap.~4]{Ber97}, its gradient satisfies the identity
\begin{equation}
\label{eq:dconj}
\nabla\hreg^{\ast}(\dpoint)
	= \argmax_{\simple\in\simples_{\sets}} \{ \braket{\dpoint}{\simple} - \hreg(\simple) \}.
\end{equation}
Thus, recalling the definition \eqref{eq:choice} of the choice map $\mirror\from\dspace\to\simples_{\sets}$, we get the equivalent expression
\begin{equation}
\label{eq:mirror}
\mirror(\dpoint)
	= \nabla\hreg^{\ast}(\dpoint).
\end{equation}
For convenience and concision, any regularizer as above will be referred to as a \emph{Fisher regularizer} on $\simples_{\sets}$.

With this background in hand, we proceed to prove some auxiliary results and estimates that are used throughout the analysis of \cref{sec:dualavg,sec:results}.
The first concerns the basic primal-dual properties of the choice map $\mirror$.
\begin{lemma}
\label{lem:mirror}
Let $\hreg$ be a Fisher regularizer on $\simples_{\sets}$.
Then, for all $\ppoint\in\dom\subd\hreg$ and all $\dpoint,\dvec\in\dspace$, we have:
\begin{subequations}
\label{eq:links}
\begin{alignat}{4}
\label{eq:links-mirror}
&a)&
	&\;\;
	\ppoint = \mirror(\dpoint)
	&\;\iff\;
	&\dpoint \in \subd\hreg(\ppoint).
	\\
\label{eq:links-prox}
&b)&
	&\;\;
	\new\ppoint = \mirror(\nabla\hreg(\ppoint) + \dvec)
	&\;\iff\;
	&\nabla\hreg(\ppoint) + \dvec \in \subd\hreg(\new\ppoint)
	\hspace{20em}
\end{alignat}
\end{subequations}
Finally, if $\ppoint = \mirror(\dpoint)$ and $\base\in\simples_{\sets}$, we have
\begin{equation}
\label{eq:selection}
\braket{\nabla\hreg(\ppoint)}{\ppoint - \base}
	\leq \braket{\dpoint}{\ppoint - \base}.
\end{equation}
\end{lemma}

\begin{remark*}
Note that \eqref{eq:links-prox} directly implies that $\subd\hreg(\new\simple) \neq \varnothing$, \ie $\new\simple \in \dom\subd\hreg$ for all $\dvec\in\dspace$.
An immediate consequence of this is that the update rule $\new\simple = \mirror(\dpoint + \dvec)$ is \emph{well-posed} for all $\dpoint\in\dspace$, $\dvec\in\region$,
\ie it can be iterated in perpetuity.
\end{remark*}

\begin{proof}[Proof of \cref{lem:mirror}]
To prove \eqref{eq:links-mirror}, note that $\ppoint$ solves \eqref{eq:dconj} if and only if $\dpoint - \subd\hreg(\ppoint) \ni 0$, \ie if and only if $\dpoint\in\subd\hreg(\ppoint)$.
\cref{eq:links-prox} is then obtained in the same manner.

For the inequality \eqref{eq:selection}, it suffices to show it holds for all $\base\in\relint\simples_{\sets}$ (by continuity).
To do so, let
\begin{equation}
\phi(t)
	= \hreg(\ppoint + t(\base-\ppoint))
	- \bracks{\hreg(\ppoint) +  \braket{\dpoint}{\ppoint + t(\base-\ppoint)}}.
\end{equation}
Since $\hreg$ is strongly convex relative to the Fisher metric and $\dpoint\in\subd\hreg(\ppoint)$ by \eqref{eq:links-mirror}, it follows that $\phi(t)\geq0$ with equality if and only if $t=0$.
Moreover, note that $\psi(t) = \braket{\nabla\hreg(\ppoint + t(\base-\ppoint)) - \dpoint}{\base - \ppoint}$ is a continuous selection of subderivatives of $\phi$.
Since $\phi$ and $\psi$ are both continuous on $[0,1]$, it follows that $\phi$ is continuously differentiable and $\phi' = \psi$ on $[0,1]$.
Thus, with $\phi$ convex and $\phi(t) \geq 0 = \phi(0)$ for all $t\in[0,1]$, we conclude that $\phi'(0) = \braket{\nabla\hreg(\ppoint) - \dpoint}{\base - \ppoint} \geq 0$, from which our claim follows.
\end{proof}

We now proceed to prove the basic properties of $\hreg$ and $\hconj$ relative to the primal and dual Fisher norms respectively.
For convenience, we restate the relevant result below.

\smooth*

\begin{proof}[Proof of \cref{lem:smooth}]
We begin with the direct implication ``$(1)\implies(2)$''.
For convenience, let $\new\dpoint = \dpoint + \dvec$, and set $\ppoint = \mirror(\dpoint)$, $\new\ppoint = \mirror(\new\dpoint)$.
We then have:
\begin{align}
\label{eq:hconj-1}
\hconj(\new\dpoint)	
		- \hconj(\dpoint)
		- \braket{\dvec}{\ppoint}
	&= \hconj(\new\dpoint)
		- \braket{\new\dpoint}{\ppoint}
		+ \hreg(\ppoint)
	\notag\\
	&= \braket{\new\dpoint}{\new\ppoint}
		- \hreg(\new\ppoint)
		- \braket{\new\dpoint}{\ppoint}
		+ \hreg(\ppoint)
	\notag\\
	&= \hreg(\ppoint)
		- \hreg(\new\ppoint)
		- \braket{\new\dpoint}{\ppoint - \new\ppoint}.
\end{align}
However, by \cref{lem:mirror}, we also have $\dpoint \in \subd\hreg(\ppoint)$ and $\new\dpoint \in \subd\hreg(\new\ppoint)$.
Hence, by the strong convexity of $\hreg$ relative to the Fisher information metric, we readily get
\begin{equation}
\label{eq:hreg-1}
\hreg(\new\ppoint)
	\geq \hreg(\ppoint)
		+ \braket{\dpoint}{\new\ppoint - \ppoint}
		+ \frac{\hstr}{2} \norm{\new\ppoint - \ppoint}_{\ppoint}^{2}.
\end{equation}
Therefore, substituting \eqref{eq:hreg-1} into \eqref{eq:hconj-1} and rearranging, we obtain
\begin{align}
\label{eq:hconj-2}
\hconj(\new\dpoint)	
		- \hconj(\dpoint)
		- \braket{\dvec}{\ppoint}
	&= \hreg(\ppoint)
		- \hreg(\new\ppoint)
		- \braket{\new\dpoint}{\ppoint - \new\ppoint}
	\notag\\
	&\leq \braket{\dpoint}{\ppoint - \new\ppoint}
		- \frac{\hstr}{2} \norm{\new\ppoint - \ppoint}_{\ppoint}^{2}
		- \braket{\new\dpoint}{\ppoint - \new\ppoint}
	\notag\\
	&= \braket{\new\dpoint - \dpoint}{\new\ppoint - \ppoint}
		- \frac{\hstr}{2} \norm{\new\ppoint - \ppoint}_{\ppoint}^{2}
	\notag\\
	&\leq \frac{\hstr}{2} \norm{\new\ppoint - \ppoint}_{\ppoint}^{2}
		+ \frac{1}{2\hstr} \norm{\new\dpoint - \dpoint}_{\ppoint,\ast}^{2}
		- \frac{\hstr}{2} \norm{\new\ppoint - \ppoint}_{\ppoint}^{2}
	\notag\\
	&= \frac{1}{2\hstr} \norm{\new\dpoint - \dpoint}_{\ppoint,\ast}^{2}
\end{align}
where, in the last line, we used \cref{lem:Fisher} to apply the Fenchel\textendash Young inequality to the convex function $\phi(\cdot) = (\hstr/2) \norm{\cdot}_{\ppoint}^{2}$ and its conjugate $\phi^{\ast}(\cdot) = 1/(2\hstr) \norm{\cdot}_{\ppoint,\ast}^{2}$.
Our claim then follows by a trivial rearrangement of \eqref{eq:hconj-2}.

For the converse direction ``$(2)\implies(1)$'', fix some $\ppoint,\new\ppoint \in \relint\simples_{\sets}$, and let $\dpoint\in\subd\hreg(\ppoint)$, $\new\dpoint\in\subd\hreg(\new\ppoint)$.
Then, reversing \eqref{eq:hconj-1} gives
\begin{align}
\label{eq:hreg-2}
\hreg(\new\ppoint)
		- \hreg(\ppoint)
		- \braket{\dpoint}{\new\ppoint - \ppoint}
	&= \hconj(\dpoint)
		- \hconj(\new\dpoint)
		+ \braket{\new\dpoint - \dpoint}{\new\ppoint}.
\end{align}
However, by the Lipschitz smoothness of $\hconj$, we have
\begin{equation}
\hconj(\new\dpoint)
	\leq \hconj(\dpoint)
		+ \braket{\ppoint}{\new\dpoint - \dpoint}
		+ \frac{1}{2\hstr} \norm{\new\dpoint - \dpoint}_{\ppoint,\ast}^{2}
\end{equation}
and hence
\begin{align}
\hreg(\ppoint)
		- \hreg(\new\ppoint)
		+ \braket{\dpoint}{\new\ppoint - \ppoint}
	&= \hconj(\new\dpoint)
		- \hconj(\dpoint)
		- \braket{\new\dpoint - \dpoint}{\new\ppoint}
	\notag\\
	&= \hconj(\new\dpoint)
		- \hconj(\dpoint)
		- \braket{\new\dpoint - \dpoint}{\ppoint}
		- \braket{\new\dpoint - \dpoint}{\new\ppoint - \ppoint}
	\notag\\
	&\leq \frac{1}{2\hstr} \norm{\new\dpoint - \dpoint}_{\ppoint,\ast}^{2}
		- \frac{1}{2\hstr} \norm{\new\dpoint - \dpoint}_{\ppoint,\ast}^{2}
		- \frac{\hstr}{2} \norm{\new\ppoint - \ppoint}_{\ppoint}^{2}
	\notag\\
	&= - \frac{\hstr}{2} \norm{\new\ppoint - \ppoint}_{\ppoint}^{2},
\end{align}
where we used the Fenchel\textendash Young inequality as above.
Our claim then follows by rearranging.
\end{proof}

We now proceed to establish some of the basic properties for the Fenchel coupling
\begin{equation}
\fench(\base,\dpoint)=\hreg(\base)+\hreg^{\ast}(\dpoint)-\braket{\dpoint}{\base}.
\end{equation}
The first property we present is a primal-dual analogue of the so-called ``three-point identity'' that is commonly used in the theory of Bregman functions \citep{CT93}.

\begin{lemma}
\label{lem:3point}
With notation as above, we have:
\begin{equation}
\label{eq:3point}
\fench(\base,\new\dpoint)
	= \fench(\base,\dpoint)
	+ \fench(\ppoint,\new\dpoint)
	+ \braket{\new\dpoint-\dpoint}{\ppoint - \base}.
\end{equation}
\end{lemma}

\begin{proof}
By alternating the dual point of comparison in the definition of the Fenchel coupling, we have:
\begin{subequations}
\begin{align}
\label{eq:Fench-a}
\fench(\base,\new\dpoint)
	&= \hreg(\base) +\hreg^{\ast}(\new\dpoint) - \braket{\new\dpoint}{\base}
	\\
\label{eq:Fench-b}
\fench(\base,\dpoint)\hphantom{'}
	&= \hreg(\base) + \hreg^{\ast}(\dpoint) - \braket{\dpoint}{\base}.
\end{align}
\end{subequations}
Then, by subtracting \eqref{eq:Fench-a} from \eqref{eq:Fench-b}, we get:
\begin{align*}
\fench(\base,\new\dpoint)-\fench(\base,\dpoint)
	&= \hreg(\base)
		+ \hreg^{\ast}(\new\dpoint)
		- \braket{\new\dpoint}{\base}
		- \hreg(\base)
		- \hreg^{\ast}(\dpoint)
		+ \braket{\dpoint}{\base}
	\notag\\
	&=\hreg^{\ast}(\new\dpoint)
		- \hreg^{\ast}(\dpoint)
		- \braket{\new\dpoint - \dpoint}{\base}
	\notag\\
	&=\hreg^{\ast}(\new\dpoint)
		- \braket{\dpoint}{\mirror(\dpoint)}
		+ \hreg(\mirror(\dpoint))
		- \braket{\new\dpoint - \dpoint}{\base}
	\notag\\
	&=\hreg^{\ast}(\new\dpoint)
		- \braket{\dpoint}{\ppoint}
		+ \hreg(\ppoint)
		- \braket{\new\dpoint - \dpoint}{\base}
	\notag\\
	&= \hreg^{\ast}(\new\dpoint)
		+ \braket{\new\dpoint - \dpoint}{\ppoint}
		- \braket{\new\dpoint}{\ppoint}
		+ \hreg(\ppoint)
		- \braket{\new\dpoint - \dpoint}{\base}
	\notag\\
	&= \fench(\ppoint,\new\dpoint)
		+ \braket{\new\dpoint - \dpoint}{\ppoint - \base}.
\qedhere
\end{align*}
\end{proof}

We are now in a position to prove \cref{lem:Fenchel}, which we restate below for convenience:

\Fenchel*

\begin{proof}[Proof of \cref{lem:Fenchel}]
Let $\new\dpoint = \dpoint + \dvec$.
Then, by the three-point identity \eqref{eq:3point}, we readily get
\begin{equation}
\label{eq:Fench-upd}
\fench(\base,\new\dpoint)
	= \fench(\base,\dpoint)
		+ \braket{\dvec}{\ppoint - \base}
		+ \fench(\ppoint,\new\dpoint)
\end{equation}
so we are left to show that $\fench(\ppoint,\new\dpoint) \leq 1/(2\hstr) \norm{\new\dpoint - \dpoint}_{\ppoint,\ast}^{2}$.
To that end, \cref{lem:smooth} yields
\begin{align}
\fench(\ppoint,\new\dpoint)
	&= \hreg(\ppoint)
		+ \hconj(\new\dpoint)
		- \braket{\new\dpoint}{\ppoint}
	\notag\\
	&\leq \hreg(\ppoint)
		+ \hconj(\dpoint)
		+ \braket{\new\dpoint - \dpoint}{\ppoint}
		+ \frac{1}{2\hstr} \norm{\new\dpoint - \dpoint}_{\ppoint,\ast}^{2}
		- \braket{\new\dpoint}{\ppoint}
	\notag\\
	&= \hreg(\ppoint)
		+ \hconj(\dpoint)
		- \braket{\dpoint}{\ppoint}
		+ \frac{1}{2\hstr} \norm{\new\dpoint - \dpoint}_{\ppoint,\ast}^{2}
	\notag\\
	&= \frac{1}{2\hstr} \norm{\new\dpoint - \dpoint}_{\ppoint,\ast}^{2}
\end{align}
where we used the fact that $\ppoint = \mirror(\dpoint)$, so $\hreg(\ppoint) + \hconj(\dpoint) - \braket{\dpoint}{\ppoint} = 0$.
\end{proof}

We close this section by discussing the properties of the negentropy regularizer $\hdec(\point) = \point\log\point$.
Regarding the strong convexity of this regularizer relative to the Fisher information metric, we would need $\hdec$ to satisfy the condition
\begin{equation}
\hdec(\base)
	\geq \hdec(\point)
		+ \hdec'(\point)(\base - \point)
		+ \frac{\hstr}{2} \frac{(\base - \point)^{2}}{\point}
\end{equation}
for some $\hstr>0$ and for all $\base,\point\in(0,1)$.
Rearranging the above inequality, and recalling the definition of the Kullback-Leibler divergence $\dkl(\base,\point) = \base \log(\base/\point)$, this requirement boils down to
\begin{equation}
\dkl(\base,\point)
	\geq (\base - \point)
		+ \frac{\hstr}{2} \frac{(\base - \point)^{2}}{\point}
\end{equation}
for some $\hstr>0$ and for all $\base,\point\in(0,1)$.
However, for any \emph{fixed} $\base\in(0,1)$, the \acl{RHS} of the above equation exhibits an $\Omega(1/\point)$ singularity as $\point\to0^{+}$, while the \acl{LHS} grows as $\bigoh(\log\point)$.
As a result, we conclude that the negentropy regularizer is \emph{not} strongly convex relative to the Fisher information metric.

On the other hand, as we show below, the entropy is \emph{tame} relative to the estimation region $\region = (-\infty,\rbound]^{\vdim}$.
To see this, note that $\hconj(\dpoint) = \log\sum_{\set\in\sets} e^{\dpoint_{\set}}$, so
\begin{equation}
\label{eq:logsumexp}
\hconj(\dpoint + \dvec) - \hconj(\dpoint)
	= \log\frac{\sum_{\set\in\sets}\exp(\dpoint_{\set} + \dvec_{\set})}{\sum_{\set\in\sets} \exp(\dpoint_{\set})}
	= \log \sum_{\set\in\sets} \ppoint_{\set} \exp(\dvec_{\set}).
\end{equation}
Now, if $\dvec\in\region$, we have $\dvec_{\set} \leq \rbound$ for all $\set\in\sets$, so there exists some $\hstr>0$ such that $\exp(\dvec_{\set}) \leq 1 + \dvec_{\set} + \dvec_{\set}^{2}/(2\hstr)$ for all $\set\in\sets$.
Then, plugging this estimate into \eqref{eq:logsumexp}, we conclude that
\begin{align}
\hconj(\dpoint + \dvec) - \hconj(\dpoint)
	\leq \log \sum_{\set\in\sets} \ppoint_{\set} (1 + \dvec_{\set} + \frac{\dvec_{\set}^{2}}{2\hstr})
	&= \log\parens*{1 + \braket{\dvec}{\ppoint} + \frac{1}{2\hstr} \sum_{\set\in\sets} \ppoint_{\set} \dvec_{\set}^{2}}
	\notag\\
	\leq 1 + \braket{\dvec}{\ppoint} + \frac{1}{2\hstr} \sum_{\set\in\sets} \ppoint_{\set} \dvec_{\set}^{2}.
\end{align}
The specific value of $\hstr$ is $1/2$ if $\rbound=1$;
for general $\rbound$, the value of $\hstr$ can be estimated by backsolving the equation $1 + \rbound + \rbound^{2}/(2\hstr) = \exp(\rbound)$.

\section{Regret guarantees for \acl{DAX}}
\label{app:explicit}

In this appendix, our aim is to prove the rest of the results presented in \cref{sec:dualavg} for \eqref{eq:DAX}.
We begin with the algorithm's template bound for the $\temp$-deflated Fenchel coupling $\energy_{\run} = \frac{1}{\temp_{\run}} \fench(\base,\temp_{\run}\score_{\run})$ as defined in \eqref{eq:energy};
for convenience, we restate the relevant result below.

\energybound*

\begin{proof}[Proof of \cref{lem:energy}]
Our proof follows the general structure of the proof of \citet[Lemma~2]{HMMR20};
however, the use of the Fisher information metric instead of a global norm introduces a number of subtleties that require special care.

We begin by rewriting the difference $\energy_{\run+1} - \energy_{\run}$ as
\begin{subequations}
\label{eq:energy-basic}
\begin{align}
\energy_{\run+1} - \energy_{\run}
	= \frac{1}{\temp_{\run+1}} \fench(\base,\temp_{\run+1}\score_{\run+1})
		- \frac{1}{\temp_{\run}} \fench(\base,\temp_{\run}\score_{\run})
	= \label{eq:energy-const}
	\frac{1}{\temp_{\run+1}}
		&\fench(\base,\temp_{\run+1}\score_{\run+1})
		- \frac{1}{\temp_{\run}} \fench(\base,\temp_{\run}\score_{\run+1})
	\\
	\label{eq:energy-update}
	+ \frac{1}{\temp_{\run}}
		&\fench(\base,\temp_{\run}\score_{\run+1})
		- \frac{1}{\temp_{\run}} \fench(\base,\temp_{\run}\score_{\run}).
\end{align}
\end{subequations}
We will proceed to bound each of these terms separately.

Beginning with the latter, the first part of \cref{lem:Fenchel} allows us to rewrite \eqref{eq:energy-update} as
\begin{align}
\eqref{eq:energy-update}
	= \frac{1}{\temp_{\run}} \bracks{\fench(\base,\temp_{\run}\score_{\run} + \temp_{\run}\model_{\run}) - \fench(\base,\temp_{\run}\score_{\run})}
	&= \frac{1}{\temp_{\run}} \bracks{\fench(\state_{\run},\temp_{\run}\score_{\run+1}) + \braket{\temp_{\run}\model_{\run}}{\state_{\run} - \base}}
	\notag\\
	&= \frac{\fench(\state_{\run},\temp_{\run}\score_{\run+1})}{\temp_{\run}}
		+ \braket{\model_{\run}}{\state_{\run} - \base}
\end{align}
where we used the fact that $\state_{\run} = \mirror(\temp_{\run}\score_{\run})$ by the definition of \eqref{eq:DAX}.
As for the term \eqref{eq:energy-const}, we readily have
\begin{align}
\eqref{eq:energy-const}
	&= \bracks*{\frac{1}{\temp_{\run+1}} - \frac{1}{\temp_{\run}}} \hreg(\base)
		+ \frac{1}{\temp_{\run+1}} \hconj(\temp_{\run+1}\score_{\run+1}) - \frac{1}{\temp_{\run}} \hconj(\temp_{\run}\score_{\run+1})
\end{align}
by the definition \eqref{eq:Fench} of the Fenchel coupling.
We will proceed to bound this term by studying the function $\varphi(\temp) = \temp^{-1} [\hconj(\temp\dpoint) + \min\hreg]$ as a function of $\temp$ for a \emph{fixed} $\dpoint\in\dspace$.
To that end, using \cref{lem:mirror} to differentiate $\varphi$ gives
\begin{align}
\varphi'(\temp)
	= \frac{1}{\temp} \braket{\dpoint}{\mirror(\temp\dpoint)}
		-\frac{1}{\temp^{2}} \bracks{ \hconj(\temp\dpoint) + \min\hreg }
	= \frac{1}{\temp^{2}} \bracks{\hreg(\mirror(\temp\dpoint)) - \min\hreg}
	\geq 0,
\end{align}
where
we used the fact that $\braket{\temp\dpoint}{\mirror(\temp\dpoint)} - \hconj(\temp\dpoint) = \hreg(\mirror(\temp\dpoint))$.
Thus, with $\temp_{\run+1} \leq \temp_{\run}$ for all $\run=\running$, we conclude that $\varphi(\temp_{\run}) \geq \varphi(\temp_{\run+1})$, and hence:
\begin{equation}
\label{eq:hconj-delta}
\frac{1}{\temp_{\run+1}} \hconj(\temp_{\run+1}\score_{\run+1})
	- \frac{1}{\temp_{\run}} \hconj(\temp_{\run}\score_{\run+1})
	\leq \bracks*{\frac{1}{\temp_{\run}} - \frac{1}{\temp_{\run+1}}} \min\hreg.
\end{equation}
Thus, recombining everything in \eqref{eq:energy-basic}, we obtain \eqref{eq:energy-bound-tight}, as claimed.

Finally, for \eqref{eq:energy-bound-norm}, recall that the first part of \cref{lem:Fenchel} is valid independently of the strong convexity modulus of $\hreg$ relative to the Fisher metric.
Thus, by invoking the assumption that $\hreg$ is $\hstr$-tame relative to $\region$, we get
\begin{align}
\fench(\state_{\run},\temp_{\run}\score_{\run+1})
	&= \hreg(\state_{\run})
		+ \hconj(\temp_{\run}\score_{\run+1})
		- \braket{\temp_{\run}\score_{\run+1}}{\state_{\run}}
	\notag\\
	&= \hreg(\state_{\run})
		+ \hconj(\temp_{\run}\score_{\run} + \temp_{\run}\model_{\run})
		- \temp_{\run} \braket{\score_{\run+1}}{\state_{\run}}
	\notag\\
	&\leq \hreg(\state_{\run})
		+ \hconj(\temp_{\run}\score_{\run})
		+ \temp_{\run}\braket{\model_{\run}}{\state_{\run}}
		+ \frac{\temp_{\run}^{2}}{2\hstr} \norm{\model_{\run}}_{\state_{\run},\ast}^{2}
		- \braket{\temp_{\run}\score_{\run+1}}{\state_{\run}}
	\notag\\
	&= \hreg(\state_{\run})
		+ \hconj(\temp_{\run}\score_{\run})
		- \braket{\temp_{\run}\score_{\run}}{\state_{\run}}
		+ \frac{\temp_{\run}^{2}}{2\hstr} \tnorm{\model_{\run}}^{2}
	\notag\\
	&= \frac{\temp_{\run}^{2}}{2\hstr} \tnorm{\model_{\run}}^{2}
\end{align}
where, in the last line, we used the fact that $\state_{\run} = \mirror(\temp_{\run}\score_{\run})$, so $\hreg(\state_{\run}) + \hconj(\temp_{\run}\score_{\run}) - \braket{\temp_{\run}\score_{\run}}{\state_{\run}} = 0$.
Thus, dividing both sides of the above inequality by $\temp_{\run}$ yields the desired result.
\end{proof}

We are now in a position to prove our template regret bounds for \eqref{eq:DAX};
for completeness, we restate them both below.

\templateDAX*

\regretDAX*

\begin{proof}[Proof of \cref{prop:reg-DAX}]
Substituting $\model_{\run} \gets \pay_{\run} + \error_{\run}$ in \eqref{eq:energy-bound-norm} and rearranging, we get
\begin{equation}
\braket{\pay_{\run}}{\base - \state_{\run}}
	= \energy_{\run} - \energy_{\run+1}
		+ \braket{\error_{\run}}{\state_{\run} - \base}
		+ \parens*{\frac{1}{\temp_{\run+1}} - \frac{1}{\temp_{\run}}} \bracks{\hreg(\base) - \min\hreg}
		+ \frac{\temp_{\run}}{2\hstr} \tnorm{\model_{\run}}^{2}.
\end{equation}
Our claim then follows by summing the above over $\run\in\runs = \{\tstart,\dotsc,\tend-1\}$.
\end{proof}

\begin{proof}[Proof of \cref{thm:DAX}]
Simply set $\tstart\gets1$, $\tend\gets\nRuns$ in \eqref{eq:reg-bound} and take expectations.
\end{proof}

\section{Regret guarantees for \acl{HDA}}
\label{app:hierarchical}
\newmacro{\windowIndex}{j}
\newmacro{\splitTerm}{\epsilon}
\newmacro{\terrorTerm}{A}
\newmacro{\surpriseTerm}{B}

\subsection{Static regret guarantees}

In the first part of this appendix, our aim is to prove the regret guarantees of \eqref{eq:HDA} against static comparators, as presented in \cref{thm:reg-stat} below.

\regstat*

\para{Overview}

Our proof hinges on applying \cref{prop:reg-DAX} to bound the regret of \eqref{eq:HDA} on each time window during which the algorithm maintains a constant cover of $\points$.
Aggregating these bounds provides a regret guarantee for \eqref{eq:HDA} over the entire horizon time of play;
however, since the algorithm is \emph{not} restarted at each window, joining the resulting window-by-window bounds ends up being fairly delicate.
The main dificulties (and associated contributing terms in the regret) are as follows:
\begin{enumerate}
\item
A comparator for a given time frame may not be admissible for a previous time frame because the granularity of an antecedent cover may not suffice to include the comparator in question.
This propagates a ``resolution error'' that becomes smaller when the cover gets finer, but larger when the window gets longer.

\item
At every change of window, the algorithm retains the same probability distribution over $\points$ (to avoid restart-forget effects).
However, this introduces a ``splitting residue'' term in the regret because of the necessary correction in the learner's scores when the resolution of the cover increases.
\end{enumerate}

\para{The covering hierarchy}

We begin by detailing how the algorithm unfolds window-by-window.
Referring to \cref{sec:hierarchical} for the relevant definitions, consider a splitting schedule $\splitsched = \left\{\run_{\windowIndex} \right\}_{1 \leq \windowIndex \leq \nSplits_{\nRuns}}$ where we recall that $\nSplits_{\run}$ is the number of splitting events which occurred before round $\run$.
For every $\windowIndex \in \{1, \dots, \nSplits_{\nRuns}\}$, we define $\runs_{\windowIndex} = \{\run_{\windowIndex}, \dots, \run_{\windowIndex+1}-1\}$, the time window between the $\windowIndex$-th and the $(\windowIndex + 1)$-th splitting event.
By convention, we denote $\run_{\nSplits_{\nRuns} + 1} = \nRuns + 1$: the last time window is therefore $\runs_{\nSplits_{\nRuns}} = \{\run_{\nSplits_{\nRuns}}, \dots, \nRuns\}$ and is a priori ``incomplete'' since the $(\nSplits_\nRuns+1)$-th splitting time has not been reached yet at time $\nRuns$.

Now, during each window $\runs_{\windowIndex}$, the underlying partition $\cover_{\windowIndex}$ contains $\nSets_{\run_{\windowIndex}} = 2^\windowIndex$ components and is \emph{fixed} throughout this window.
At each time $\run_{\windowIndex} \in \splitsched$, a splitting event is performed on $\cover_{\windowIndex-1}$, in order to obtain $\cover_{\windowIndex}$ by splitting in two each set in $\cover_{\windowIndex-1}$ such that $\cover_{\windowIndex} = \cover_{\windowIndex-1}^{+}$, as described in \eqref{eq:HDA}.
Then, for a \emph{fixed} point $\point \in \points$ and all $\windowIndex = \running$, we define the corresponding \emph{approximate identity at $\point$} to be the simple strategy $\simple_{\windowIndex}^{\point} \in \simples_{\cover_{\windowIndex}}$ such that
\begin{equation}
\simple_{\windowIndex}^{\point}(\coverset)
	= \oneof{\point\in\coverset}
	\quad
	\text{for all $\coverset\in\cover_{\windowIndex}$}
\end{equation}
\ie $\simple_{\windowIndex}^{\point}$ is the best approximation of $\delta_\point$ among simple strategies of $\simples_{\cover_{\windowIndex}}$.
In the following, we will write $\coverset_{\windowIndex}^{\point}$ for the support of $\simple_{\windowIndex}^{\point}$, \ie for the unique covering element of $\cover_{\windowIndex}$ containing $\point$.

With this background in hand, \cref{lem:point2simple} yields
\begin{equation}
\label{eq:reg-simple-window-1}
\reg_\point(\runs_{\windowIndex})
	\leq \reg_{\simple_{\windowIndex}^{\point}}(\runs_{\windowIndex})
		+ \lips \diam(\coverset_{\windowIndex}^{\point}) \abs{\runs_{\windowIndex}}
\end{equation}
where, by definition, $\abs{\runs_{\windowIndex}} = \run_{\windowIndex+1} - \run_{\windowIndex}$.
In turn, this allows us to bound $\diam(\coverset_{\windowIndex}^{\point})$ with respect to $\nSets_{\run_{\windowIndex}}$, the number of sets in the partition $\cover_{\windowIndex}$.

\begin{lemma}
\label{lem:diameter-wrt-nSets}
If $\cover$ is the partition of $\points$ after $\nSplits$ splitting events \textpar{and therefore containing $\nSets=2^\nSplits$ covering sets}, then, for all $\coverset \in \cover$, we have
\begin{equation}
\diam(\coverset)
	\leq 2\diamconst \diam(\points)\nSets^{-1/\vdim}
\end{equation}
where $\diam$ is defined with respect to the ambient norm of $\points \subseteq \R^{\vdim}$.
\end{lemma}

\begin{proof}
Let $\cover$ be the partition of $\points$ after $\nSplits$ splitting events, which contains $\nSets=2^\nSplits$ covering sets.

For any set $\coverset \in \cover$, \cref{assumpt:diameter} gives us the following bound for $\diam(\coverset)$ (similar assumptions are made for example in \cite{BMSS11}):
$$
\diam(\coverset) \leq \frac{\diamconst \diam(\points)}{2^{\floor{\sigma/\vdim}}},
$$
We may now write the following sequence of inequalities using the definition of $\floor{.}$ and the fact that $x\mapsto2^x$ is increasing
\begin{align*}
    \floor{\nSplits/\vdim} &> \nSplits/\vdim - 1\\
    2^{\floor{\nSplits / \vdim}} &> \frac{1}{2} 2^{\nSplits / \vdim}\\
    2^{-\floor{\nSplits / \vdim}} &< 2 \times 2^{-\nSplits / \vdim}\\
    \diamconst \diam(\points)2^{-\floor{\nSplits / \vdim}} &< 2\diamconst \diam(\points) 2^{-\nSplits / \vdim}.
\end{align*}
Now, using the fact that $\nSets = 2^{\nSplits}$, we finally get that for any $\set \in \cover$, 
\begin{equation*}
\diam(\set) \leq 2\diamconst \diam(\points) \nSets^{-1/\vdim}.
\qedhere
\end{equation*}
\end{proof}

\para{Aggregating cover bounds}

To proceed, injecting the estimate of \cref{lem:diameter-wrt-nSets} into \eqref{eq:reg-simple-window-1} delivers
\begin{equation}
\reg_\point(\runs_{\windowIndex})
	\leq \reg_{\simple_{\windowIndex}^{\point}}(\runs_{\windowIndex})
		+ 2\lips\diamconst \diam(\points) \abs{\runs_{\windowIndex}} \nSets_{\run_{\windowIndex}}^{-1/\vdim}.
\end{equation}
and hence, by \cref{prop:reg-DAX} applied to $\simple_{\windowIndex}^{\point} \in \simples_{\cover_{\windowIndex}}$, we get
\begin{align}
\reg_\point(\runs_{\windowIndex})
	&\leq \energy^{\cover_{\windowIndex}}_{\run_{\windowIndex}} - \energy^{\cover_{\windowIndex}}_{\run_{\windowIndex+1}}
	+ \parens*{\temp_{\run_{\windowIndex+1}}^{-1} - \temp_{\run_{\windowIndex}}^{-1}}
		\bracks*{\hreg^{\cover_{\windowIndex}}(\simple_{\windowIndex}^{\point}) - \min\hreg^{\cover_{\windowIndex}}}
	\notag\\
	&+ \sum_{\run\in\runs_{\windowIndex}}
		\braket*{\error_{\run}}{\state_{\run} - \simple_{\windowIndex}^{\point}}
	+ \frac{1}{2\hstr} \sum_{\run\in\runs_{\windowIndex}} \temp_{\run} \tnorm{\model_{\run}}^{2}
	+ 2\lips\diamconst \diam(\points) \abs{\runs_{\windowIndex}} \nSets_{\run_{\windowIndex}}^{-1/\vdim}
\end{align}
Noting that $\min\hreg^{\cover_{\windowIndex}} = \hvol(\nSets_{\run_{\windowIndex}})$ where $\hvol(z) = z\hdec(1/z)$ for all $z>0$, and that $\hreg(\simple_{\windowIndex}^{\point}) = 0$ we can write
\begin{align}
\label{eq:hreg2hvol}
	\hreg^{\cover_{\windowIndex}}(\simple_{\windowIndex}^{\point}) - \min\hreg^{\cover_{\windowIndex}}
	= \hvol(\nSets_{\run_{\windowIndex}})
\end{align}
leading in turn to the expression
\begin{align}
\label{eq:reg-stat-simple-bound-1-window-1}
	\reg_\point(\runs_{\windowIndex}) &\leq 
	\energy^{\cover_{\windowIndex}}_{\run_{\windowIndex}} - \energy^{\cover_{\windowIndex}}_{\run_{\windowIndex+1}}
	+ \parens*{\temp_{\run_{\windowIndex+1}}^{-1} - \temp_{\run_{\windowIndex}}^{-1}} \hvol(\nSets_{\run_{\windowIndex}})
	\notag\\
	&+ \sum_{\run\in\runs_{\windowIndex}} \braket*{\error_{\run}}{\state_{\run} - \simple_{\windowIndex}^{\point}}
	+ \frac{1}{2\hstr} \sum_{\run\in\runs_{\windowIndex}} \temp_{\run} \tnorm{\model_{\run}}^{2}
	+ 2\lips\diamconst \diam(\points)\abs{\runs_{\windowIndex}}\nSets_{\run_{\windowIndex}}^{-1/\vdim}
\end{align}
where we have made an explicit reference to the underlying partition in the exponent of the $\energy$ and $\hreg$ terms.
As indicated by the presence of the term $\sum_{\windowIndex = 2}^{\nSplits_{\nRuns}}\bracks*{\energy_{\run_{\windowIndex}}^{\cover_{\windowIndex}} - \energy_{\run_{\windowIndex}}^{\cover_{\windowIndex-1}}}$, this subtlety is crucial for the algorithm's regret, as it accounts for the cost of descending to a cover with higher granularity.

To make this precise, note that the regret incurred by \eqref{eq:HDA} over $\nRuns$ stages can be decomposed as 
\begin{equation}
\label{eq:sum-of-regret-on-windows}
\reg_\point(\nRuns)
	= \sum_{\windowIndex=\start}^{\nSplits_{\nRuns}}
		\reg_\point(\runs_{\windowIndex})
\end{equation}
where each $\reg_\point(\runs_{\windowIndex})$ corresponds to the regret incurred by \eqref{eq:HDA} on a fixed partition \textendash\ \ie the regret induced by \eqref{eq:DAX} over the said partition, assuming the algorithm was initialized at the last state of the previous window (since the algorithm does not restart).
Then, combining \eqref{eq:reg-stat-simple-bound-1-window-1} and \eqref{eq:sum-of-regret-on-windows}, we get
\begin{align}
\reg_\point(\nRuns)
	&= \sum_{\windowIndex=\start}^{\nSplits_{\nRuns}}\reg_\point(\runs_{\windowIndex})
	\notag\\
	&\leq \sum_{\windowIndex=\start}^{\nSplits_{\nRuns}}
		\bracks*{\energy^{\cover_{\windowIndex}}_{\run_{\windowIndex}} - \energy^{\cover_{\windowIndex}}_{\run_{\windowIndex+1}}}
	+ \sum_{\windowIndex=\start}^{\nSplits_{\nRuns}}
		\hvol(\nSets_{\run_{\windowIndex}})
		\parens*{\temp_{\run_{\windowIndex+1}}^{-1} - \temp_{\run_{\windowIndex}}^{-1}}
	\notag\\
	&\qquad
	+ \sum_{\windowIndex=\start}^{\nSplits_{\nRuns}}
		\sum_{\run\in\runs_{\windowIndex}}
			\braket*{\error_{\run}}{\state_{\run} - \simple_{\windowIndex}^{\point}}
	+ \frac{1}{2\hstr} \sum_{\windowIndex=\start}^{\nSplits_{\nRuns}}
		\sum_{\run\in\runs_{\windowIndex}}
			\temp_{\run} \tnorm{\model_{\run}}^{2}
	+ 2\lips\diamconst \diam(\points)
		\sum_{\windowIndex=\start}^{\nSplits_{\nRuns}} \abs{\runs_{\windowIndex}}\nSets_{\run_{\windowIndex}}^{-1/\vdim}
	\notag\\
	&= \energy_{\run_1}^{\cover_1} - \energy_{\run_{\nSplits_{\nRuns}}}^{\cover_{\nSplits_{\nRuns+1}}}
		+ \sum_{\windowIndex=2}^{\nSplits_{\nRuns}}
			\bracks*{\energy^{\cover_{\windowIndex}}_{\run_{\windowIndex}} - \energy^{\cover_{\windowIndex -1}}_{\run_{\windowIndex}}}
	\notag\\
	&\qquad
	+ \hvol(\nSets_{\nRuns})\temp_{\nRuns+1}^{-1} - \hvol(1)\temp_{1}^{-1} 
	+ \temp_{\nRuns}^{-1}\sum_{\windowIndex=2}^{\nSplits_{\nRuns}}
			\parens*{\hvol(\nSets_{\run_{\windowIndex-1}}) - \hvol(\nSets_{\run_{\windowIndex}})}
	\notag\\
	&\qquad
	+ \sum_{\windowIndex=\start}^{\nSplits_{\nRuns}}
		\sum_{\run\in\runs_{\windowIndex}}
			\braket*{\error_{\run}}{\state_{\run} - \simple_{\windowIndex}^{\point}} 
	+ \frac{1}{2\hstr} \sum_{\run=1}^{\nRuns}\temp_{\run} \tnorm{\model_{\run}}^{2}
	+ 2\lips\diamconst \diam(\points) \sum_{\run=1}^{\nRuns} \nSets_{\run}^{-1/\vdim}.
\end{align}
where we used the fact that $\temp_{\run}$ is nonincreasing.
Thus, noting that $\energy^{\cover_{\nSplits_{\nRuns}}}_{\run_{\nSplits_{\nRuns} + 1}} = \energy^{\cover_{\nSplits_{\nRuns}}}_{\nRuns+1} \geq 0$
and
$\energy^{\cover_1}_1 = \temp_1^{-1}\bracks*{\hreg^{\cover_1}(\base_1^{\point}) + \hconj(0)}=\hvol(1)/ \temp_1$,
we get the following bound for the regret incurred by \eqref{eq:HDA}: 
\begin{align}
\label{eq:intermediary-template-bound-thm-statreg}
\reg_\point(\nRuns)
	&\leq \sum_{\windowIndex=2}^{\nSplits_{\nRuns}}
		\bracks*{\energy^{\cover_{\windowIndex}}_{\run_{\windowIndex}} - \energy^{\cover_{\windowIndex -1}}_{\run_{\windowIndex}}}
	+ \temp_{\nRuns}^{-1} \sum_{\windowIndex=2}^{\nSplits_{\nRuns}}
		\parens*{\hvol(\nSets_{\run_{\windowIndex-1}}) - \hvol(\nSets_{\run_{\windowIndex}})}
	+ \hvol(\nSets_{\nRuns})\temp_{\run_{\nRuns+1}}^{-1}
	\notag\\
	&\qquad
	+ \sum_{\windowIndex=\start}^{\nSplits_{\nRuns}}
		\sum_{\run\in\runs_{\windowIndex}}
			\braket*{\error_{\run}}{\state_{\run} - \simple_{\windowIndex}^{\point}}
	+ \frac{1}{2\hstr} \sum_{\run=1}^{\nRuns}\temp_{\run} \tnorm{\model_{\run}}^{2}
	+ 2\lips\diamconst \diam(\points)\sum_{\run=1}^{\nRuns}\nSets_{\run}^{-1/\vdim}.
\end{align}
Controlling the growth of each term in the above will be the main focus of our analysis in the sequel.

\para{The splitting residue}

Somewhat surprisingly, the first two terms of \eqref{eq:intermediary-template-bound-thm-statreg} turn out to be the most challenging ones to control.
Because both terms are due to the algorithm's hierarchical splitting schedule, we will refer collectively to the sum
\begin{equation}
\splitTerm_{\nRuns}
	= \sum_{\windowIndex = 2}^{\nSplits_{\nRuns}}\bracks*{\energy_{\run_{\windowIndex}}^{\cover_{\windowIndex}} - \energy_{\run_{\windowIndex}}^{\cover_{\windowIndex-1}}} + \temp_{\nRuns}^{-1}\sum_{\windowIndex=2}^{\nSplits_{\nRuns}}\parens*{\hvol(\nSets_{\run_{\windowIndex-1}}) - \hvol(\nSets_{\run_{\windowIndex}})},
\end{equation}
as the algorithm's \emph{splitting residue}.

To analyze this term,
given a regularization kernel $\hdec$ and a partition $\cover$ of $\points$, let $\hreg^{\cover}$ be the corresponding (decomposable) regularizer induced by $\hdec$ on $\simples_{\cover}$, and write $\mirror^{\cover}$ for the associated choice map $\mirror^{\cover} \from \dspace \to \simples_{\cover}$.
Moreover, given a simple strategy $\state\in\simples_{\cover}$ and recalling that $\new\cover$ denotes the successor of $\cover$ after a splitting event, we will write $\new\state \in \simples_{\new\cover}$ for the mixed strategy on $\new\cover$ such that the canonical cast of $\state$ and $\new\state$ as distributions (with piecewise constant densities) on $\points$ are the same.
Finally, for $\score \in \R^{\cover}$ such that $\mirror^{\cover}(\score) = \simple$, we will write $\new\score$ for any piecewise constant function in $\R^{\new\cover}$ such that $\mirror^{\new\cover}(\new\score) = \new\simple$.%
\footnote{Any two such functions will only differ by a constant.
This constant plays no role in our analysis, so we will ignore it in the sequel.}

With all this in hand, our next result provides an
an inverse-rate proportional upper bound for the \emph{splitting residue} term $\splitTerm_{\nRuns}$.

\begin{lemma}
\label{lem:splitting-residue}
Let $K_{\hdec} = \sup_{a\in(0,1]} \bracks*{\hdec(a) - 2\hdec(a/2)}/a$ and $\alt K_{\hdec} = \sup_{\score} \inf_{\new\score} \supnorm{\new\score - \score}$.
Then 
\begin{subequations}
\label{eq:splitbound}
\begin{align}
\label{eq:terrorterm}
	&\sum_{\windowIndex = 2}^{\nSplits_{\nRuns}}
		\bracks*{\energy_{\run_{\windowIndex}}^{\cover_{\windowIndex}} - \energy_{\run_{\windowIndex}}^{\cover_{\windowIndex-1}}} 
	\leq (K_{\hdec} + 2K_{\hdec}^{'}) \frac{\nSplits_{\nRuns}-1}{\temp_{\nRuns}}
	\\
\label{eq:surpriseterm}
\frac{1}{\temp_{\nRuns}}
	&\sum_{\windowIndex=2}^{\nSplits_{\nRuns}}
		\parens*{\hvol(\nSets_{\run_{\windowIndex-1}}) - \hvol(\nSets_{\run_{\windowIndex}})}
	\leq K_{\hdec} \frac{\nSplits_{\nRuns} - 1}{\temp_{\nRuns}}
\end{align}
\end{subequations}
\end{lemma}

\begin{remark}
In the above, $\score$ and $\new\score$ are viewed as piecewise constant functions of $\points$, with respective covers $\cover$ and $\new\cover$.
As an example, the negentropy regularizer $\hdec(\point) = \point\log\point$ has $K_{\hdec} = \alt K_{\hdec} = \log2$.
\end{remark}

\begin{proof}[Proof of \cref{lem:splitting-residue}]
The series of calculations required to prove the bounds \eqref{eq:splitbound} is quite intreicate and needs a fair amount of groundwork.
First, for a given $\windowIndex \in \{1, \dots, \nSplits_{\nRuns}\}$ we will use a ``$^{-}$'' exponent to refer to quantities that \emph{would have existed if there had not been a splitting event at time $\run_\windowIndex$}, and a ``$^{+}$'' exponent to refer to quantities that are derived in a scenario where there \emph{is indeed a splitting event happening at $\run_\windowIndex$}.
For more concreteness, let $\dpoint_{\run_\windowIndex - 1} \in \R^{\cover_{\windowIndex}}$ be the score function at time $\run_\windowIndex$. Then $\dpoint^{-}_{\run_\windowIndex}$ is such that $\dpoint^{-}_{\run_\windowIndex} = \dpoint_{\run_\windowIndex - 1} + \model_{\run_\windowIndex}$, \ie it is still an element of $\R^{\cover_{\windowIndex-1}}$, and correspond of what to the score at time $\run_\windowIndex$ is no splitting event happens, then we have $\state_{\run_\windowIndex}^{-} = \mirror^{\cover_{\windowIndex-1}}(\temp_{\run_\windowIndex}\dpoint^{-}_{\run_\windowIndex})$ is the corresponding probability distribution on $\cover_{\windowIndex-1}$.
On the contrary, $\dpoint^{+}_{\run_\windowIndex}$ is an element of $\R^{\cover_{\windowIndex}}$ and is such that $\state_{\run_\windowIndex}^{+} = \mirror^{\cover_{\windowIndex}}(\temp_{\run_\windowIndex}\dpoint^{+}_{\run_\windowIndex})$ is \emph{consistent} with $\state_{\run_\windowIndex}^{-}$, \ie their cast as densities of $\points$ are \emph{the same}.

These distinctions are subtle, but essential to grasp the meaning of the splitting residue $\splitTerm_{\nRuns}$. To streamline the proof, let us decompose this terms into two terms as follows:
\begin{equation}
\splitTerm_{\nRuns}
	= \underbrace{\sum_{\windowIndex = 2}^{\nSplits_{\nRuns}} \bracks*{\energy_{\run_{\windowIndex}}^{\cover_{\windowIndex}} - \energy_{\run_{\windowIndex}}^{\cover_{\windowIndex-1}}}}_{\terrorTerm_{\nRuns}}
	+ \underbrace{\frac{1}{\temp_{\nRuns}}\sum_{\windowIndex=2}^{\nSplits_{\nRuns}}\parens*{\hvol(\nSets_{\run_{\windowIndex-1}}) - \hvol(\nSets_{\run_{\windowIndex}})}}_{\surpriseTerm_{\nRuns}}
\end{equation}
We bound each of these terms individually below.

\para{Step 1: Bounding $\terrorTerm_{\nRuns}$}

Let $\windowIndex \in \{1, \dots, \nSplits_{\nRuns}\}$. Using the definition of the energy and the notations introduced above we have:
\begin{equation}\label{eq:energy-terror-term}
	\energy^{\cover_{\windowIndex-1}}_{\run_{\windowIndex}} = \temp_{\run_\windowIndex}^{-1}\bracks*{\hreg^{\cover_{\windowIndex-1}}(\simple_{\windowIndex-1}^{\point}) +  \hreg^{\ast \cover_{\windowIndex-1}}(\temp_{\run_\windowIndex}\dpoint_{\run_\windowIndex}^{-}) - \braket*{\temp_{{\run_\windowIndex}} \dpoint_{\run_\windowIndex}^{-} }{\simple_{\windowIndex-1}^{\point}}^{\cover_{\windowIndex-1}}}
\end{equation}
We may drop the explicit reference to the underlying partition in exponents of $\hreg, \hconj$ and $\braket*{.}{.}$, since there respective arguments now explicitly belong to $\simples_{\cover_{\windowIndex-1}}$ and $\R^{\cover_{\windowIndex-1}}$. Using the fact that $\state_{\run_\windowIndex}^{-} = \mirror(\temp_{\run_\windowIndex}\dpoint_{\run_\windowIndex}^{-})$, we can write that
\begin{equation}
\hconj(\temp_{\run_\windowIndex}\dpoint_{\run_\windowIndex}^{-}) = \braket*{\temp_{\run_\windowIndex}\dpoint_{\run_\windowIndex}^{-}}{\state_{\run_\windowIndex}^{-}} - \hreg(\state_{\run_\windowIndex}^{-}).
\end{equation}
Injecting this in \eqref{eq:energy-terror-term}, and proceeding similarly for $\energy_{\run_\windowIndex}^{\cover_{\windowIndex}}$ finally gives
\begin{subequations}
\begin{align}
\label{eq:energy-terror-term-final-beforesplit}
		\energy^{\cover_{\windowIndex-1}}_{\run_{\windowIndex}} &= \braket*{\dpoint_{\run_\windowIndex}^{-}}{\state_{\run_\windowIndex}^{-} - \simple_{\windowIndex-1}^{\point}} - \temp_{\run_\windowIndex}^{-1}\hreg(\state_{\run_\windowIndex}^{-}) \\
\label{eq:energy-terror-term-final-aftersplit}
\energy^{\cover_{\windowIndex}}_{\run_{\windowIndex}} &= \braket*{\dpoint_{\run_\windowIndex}^{+}}{\state_{\run_\windowIndex}^{+} - \simple_{\windowIndex}^{\point}} - \temp_{\run_\windowIndex}^{-1}\hreg(\state_{\run_\windowIndex}^{+})
\end{align}
\end{subequations}
Now using the fact that $\{\temp_\run\}_\run$ is decreasing, we can bound $\terrorTerm_{\nRuns}$ as
\begin{equation}\label{eq:terror-term-bound-decomp}
	\terrorTerm_{\nRuns} \leq \temp_{\nRuns}^{-1}\underbrace{\sum_{\windowIndex = 2}^{\nSplits_{\nRuns}}\bracks*{\hreg(\state^{-}_{\run_\windowIndex}) - \hreg(\new\state_{\run_\windowIndex})}}_{\terrorTerm_{\nRuns}^{(1)}} + \underbrace{\sum_{\windowIndex = 2}^{\nSplits_{\nRuns}}\bracks*{\braket*{\dpoint^{+}_{\run_\windowIndex}}{\new\state_{\run_\windowIndex} - \simple_{\windowIndex}^{\point}} - \braket*{\dpoint^{-}_{\run_\windowIndex}}{\state^{-}_{\run_\windowIndex} - \simple_{\windowIndex-1}^{\point}}}}_{\terrorTerm_{\nRuns}^{(2)}}.
\end{equation}

We recall that the $k^{th}$ set of partition $\cover_{\windowIndex-1}$ is split into the (equally-sized) $2k^{th}$ and $(2k+1)^{th}$ sets of partition $\cover_{\windowIndex}$, and that we ensure distributions $\new\state_{\run_\windowIndex}$ and $\state^{-}_{\run_\windowIndex}$ have the same canonical cast as a distribution on $\points$, \ie  $\new\state_{\run_\windowIndex, 2k} = \new\state_{\run_\windowIndex, 2k+1} = \state^{-}_{\run_\windowIndex, k}/2$. Now using the decomposability of regularizer $\hreg$ gives for any $\windowIndex$,
\begin{align}
\hreg(\state^{-}_{\run_\windowIndex}) - \hreg(\new\state_{\run_\windowIndex})
	&= \sum_{k=1}^{2^{\windowIndex-1}}\hdec(\state^{-}_{\run_\windowIndex, k}) - \sum_{k=1}^{2^{\windowIndex}}\hdec(\new\state_{\run_\windowIndex, k})
	\notag\\
	&= \sum_{k=1}^{2^{\windowIndex-1}}\bracks*{\hdec(\state^{-}_{\run_\windowIndex, k}) - \hdec(\new\state_{\run_\windowIndex, 2k}) - \hdec(\new\state_{\run_\windowIndex, 2k+1})}
	\notag\\
	&= \sum_{k=1}^{2^{\windowIndex-1}}\underbrace{\bracks*{\hdec(\state^{-}_{\run_\windowIndex, k}) - 2\hdec(\state^{-}_{\run_\windowIndex, k}/2)}}_{\leq K_{\hdec} \state^{-}_{\run_\windowIndex, k}}
	\leq K_{\hdec}\sum_{k=1}^{2^{\windowIndex-1}} \state^{-}_{\run_\windowIndex, k}
\end{align}
where we used the fact that $\new\state_{\run_\windowIndex, 2k} = \new\state_{\run_\windowIndex, 2k+1} = \state^{-}_{\run_\windowIndex, k}/2$ 
to go from line 2 to 3, and the definition of $K_\hdec$ to go from line 3 to 4. Now using that $\sum_{k=1}^{2^{\windowIndex-1}}\state^{-}_{\run_\windowIndex, k} = 1$ and summing for $\windowIndex \in \{2, \dots, \nSplits_{\nRuns}\}$ delivers
\begin{equation}\label{eq:terror-term-A-bound}
	\terrorTerm_{\nRuns}^{(1)} \leq (\nSplits_{\nRuns} -1)K_{\hdec}\temp_{\nRuns}^{-1}
\end{equation}

Turning now to $\terrorTerm_{\nRuns}^B$ we begin by writing explicitly the braket terms associated to partitions $\cover_\windowIndex$ and $\cover_{\windowIndex-1}$ 
\begin{subequations}
\begin{align}
	\braket*{\dpoint^{-}_{\windowIndex}}{\state^{-}_{\run_\windowIndex}} &= \sum_{k=1}^{2^{\windowIndex-1}}\dpoint^{-}_{\windowIndex,k}\state^{-}_{\run_\windowIndex, k} \\
	\braket*{\dpoint^{+}_{\windowIndex}}{\new\state_{\run_\windowIndex}} &= \sum_{k=1}^{2^{\windowIndex}}\dpoint^{+}_{\windowIndex,k}\new\state_{\run_\windowIndex, k}
\end{align}
\end{subequations}
Therefore, their difference can be rewritten as:
\begin{align}
\notag
	\braket*{\dpoint^{+}_{\windowIndex}}{\new\state_{\run_\windowIndex}} - \braket*{\dpoint^{-}_{\windowIndex}}{\state^{-}_{\run_\windowIndex}} &= \sum_{k=1}^{2^{\windowIndex-1}}\bracks*{\dpoint^{+}_{\windowIndex,2k}\new\state_{\run_\windowIndex, 2k} + \dpoint^{+}_{\windowIndex,2k+1}\new\state_{\run_\windowIndex, 2k+1} - \dpoint^{-}_{\windowIndex,k}\state^{-}_{\run_\windowIndex, k}}\\
	\notag
	&= \sum_{k=1}^{2^{\windowIndex-1}}\state^{-}_{\run_\windowIndex, k}\left[\frac{\dpoint^{+}_{\windowIndex,2k} + \dpoint^{+}_{\windowIndex,2k+1}}{2} - \dpoint^{-}_{\windowIndex,k}\right]\\
	\label{eq:terror-term-B-prior-bound}
	&= \frac{1}{2}\sum_{k=1}^{2^{\windowIndex-1}}\state^{-}_{\run_\windowIndex, k}\left[\left(\dpoint^{+}_{\windowIndex,2k} - \dpoint^{-}_{\windowIndex,k} \right) + \left(\dpoint^{+}_{\windowIndex,2k+1} - \dpoint^{-}_{\windowIndex,k}\right)\right]
\end{align}
where we used the fact that $\new\state_{\run_\windowIndex, 2k} = \new\state_{\run_\windowIndex, 2k+1} = \state^{-}_{\run_\windowIndex, k}/2$. 
Finally, the second assumption of \cref{lem:splitting-residue} states that there exists a constant $\alt{K}_{\hdec}$ such that, if $S, \new\score$ are such that $\state = \mirror(S)$ and $\new\state = \mirror^{+}(\new\score)$ where $\state$ and $\new\state$ are consistent distributions on two successive partitions of $\points$, $\cover$ and $\new\cover$, then $\supnorm{S - \new\score} \leq \alt{K}_{\hdec}$. Applying this condition to $S = \temp_{\run_\windowIndex}\dpoint^{-}_{\run_\windowIndex}$ and $\new\score = \temp_{\run_\windowIndex}\dpoint^{+}_{\run_\windowIndex}$ readily gives
\begin{equation}
\supnorm{\dpoint^{-}_{\run_\windowIndex} - \dpoint^{+}_{\run_\windowIndex}} \leq \temp_{\run_\windowIndex}^{-1}\alt{K}_{\hdec},
\end{equation}
where $\dpoint^{-}_{\run_\windowIndex} \in \R^{\cover_{\windowIndex-1}}$ and $\dpoint^{+}_{\run_\windowIndex} \in \R^{\cover_{\windowIndex}}$ are cannonically casted as piecewise constant functions on $\points$. Injecting this inequality in \eqref{eq:terror-term-B-prior-bound}, using the fact that $\sum_{k=1}^{2^{\windowIndex-1}}\state^{-}_{\run_\windowIndex, k} = 1$ and summing for $\windowIndex \in \{2, \dots, \nSplits_{\nRuns}\}$ gives
\begin{equation}\label{eq:terror-term-B-almostthere-bound}
	 \sum_{\windowIndex = 2}^{\nSplits_{\nRuns}}\bracks*{\braket*{\dpoint^{+}_{\run_\windowIndex}}{\new\state_{\run_\windowIndex}} - \braket*{\dpoint^{-}_{\run_\windowIndex}}{\state^{-}_{\run_\windowIndex}}} \leq \alt{K}_{\hdec} \sum_{\windowIndex = 2}^{\nSplits_{\nRuns}} \temp_{\run_\windowIndex}^{-1} \leq \alt{K}_{\hdec} \temp_{\nRuns}^{-1}
\end{equation}
To finally conclude on a bound for $\terrorTerm_{\nRuns}^B$, we just need to bound $\braket*{\dpoint^{+}_{\windowIndex}}{\simple_{\windowIndex}^{\point}} - \braket*{\dpoint^{-}_{\windowIndex}}{\simple_{\windowIndex-1}^{\point}}$. Remarking that for all $\coverset \in \cover_{\windowIndex}$, $\simple_{\windowIndex, k}^{\point}(\coverset) = \one_{\point \in \coverset}$ and using a similar approach as before, it is straightforward to show that  
\begin{equation}
\abs{\braket*{\dpoint^{+}_{\windowIndex}}{\simple_{\windowIndex}^{\point}} - \braket*{\dpoint^{-}_{\windowIndex}}{\simple_{\windowIndex-1}^{\point}}} \leq \temp_{\run_\windowIndex}^{-1}\alt{K}_{\hdec},
\end{equation}
which combined with \eqref{eq:terror-term-B-almostthere-bound} finally gives
\begin{equation}\label{eq:terror-term-B-bound}
	 \terrorTerm_{\nRuns}^{(2)} \leq  2\alt{K}_{\hdec} (\nSplits_{\nRuns} - 1) \temp_{\nRuns}^{-1},
\end{equation}
Showing \eqref{eq:terrorterm} comes down to summing up \eqref{eq:terror-term-A-bound} \eqref{eq:terror-term-B-bound}.

\para{Step 2: Bounding $\surpriseTerm_{\nRuns}$}
We now finish this showing \eqref{eq:surpriseterm}. The bound can be directly obtained from using previously introduced tools. Indeed, we can write
\begin{align}
	\sum_{\windowIndex = 2}^{\nSplits_{\nRuns}}\bracks*{\hvol(\nSets_{\run_{\windowIndex-1}}) - \hvol(\nSets_{\run_{\windowIndex}})}
	&= \sum_{\windowIndex = 2}^{\nSplits_{\nRuns}}\bracks*{\nSets_{\run_{\windowIndex-1}}\hdec(\nSets^{-1}_{\run_{\windowIndex-1}}) - \nSets_{\run_{\windowIndex}}\hdec(\nSets^{-1}_{\run_{\windowIndex}})}
	\notag\\
	&= \sum_{\windowIndex = 2}^{\nSplits_{\nRuns}}\nSets_{\run_{\windowIndex-1}}\underbrace{\bracks*{\hdec(\nSets^{-1}_{\run_{\windowIndex-1}}) - 2\hdec(\nSets^{-1}_{\run_{\windowIndex-1}}/2 )}}_{\leq K_{\hdec} \nSets^{-1}_{\run_{\windowIndex-1}}}
	\notag\\
	&\leq \sum_{\windowIndex = 2}^{\nSplits_{\nRuns}}\nSets_{\run_{\windowIndex-1}} K_{\hdec} \nSets_{\run_{\windowIndex-1}}^{-1}
	\notag\\
	&= K_{\hdec} (\nSplits_{\nRuns} - 1).
\end{align}
where we used the fact that $\nSets_{\run_{\windowIndex}} = 2\nSets_{\run_{\windowIndex-1}}$ to go from line 1 to 2, and the first assumption of $\hdec$ to go from line 2 to 3. This directly delivers \eqref{eq:surpriseterm} after multiplying by $\temp_{\nRuns}^{-1}$, and therefore completes the proof.
\end{proof}

\para{Putting everything together}

We are finally in a position to derive our static regret guarantees for \eqref{eq:HDA}.

\begin{proof}[Proof of \cref{thm:reg-stat}]
Let $\cst_{\hdec} = 2(K_{\hdec} + \alt{K}_{\hdec})$.
Then, plugging \eqref{eq:terrorterm} and \eqref{eq:surpriseterm} into \eqref{eq:intermediary-template-bound-thm-statreg}, we obtain
\begin{align}
\label{eq:regret-simple-template}
\reg_\point(\nRuns)
	&\leq  \frac{\cst_\theta\nSplits_{\nRuns} + \hvol(\nSets_{\nRuns})}{\temp_{\nRuns+1}}
	\notag\\
	&+ \sum_{\windowIndex=\start}^{\nSplits_{\nRuns}}
			\sum_{\run\in\runs_{\windowIndex}}
				\braket*{\error_{\run}}{\state_{\run} - \simple_{\windowIndex}^{\point}}
	+ \frac{1}{2\hstr} \sum_{\run=1}^{\nRuns}\temp_{\run} \tnorm{\model_{\run}}^{2}
	+ 2\lips\diamconst \diam(\points)\sum_{\run=1}^{\nRuns} \frac{1}{\nSets_{\run}^{1/\vdim}}
\end{align}
The bound \eqref{eq:reg-bound-stat} then follows by
taking expectations in \eqref{eq:regret-simple-template},
using the bounds \eqref{eq:sigbounds} for the estimator $\model_{\run}$,
and recalling that $\nSplits_{\nRuns} = \log_2(\nSets_\nRuns)$.

We now turn to the second part of \cref{thm:reg-stat}, namely the expected regret bound \eqref{eq:reg-bound-stat-powers}.
The main challenge here is that \eqref{eq:reg-bound-stat-powers} bounds the algorithm's \emph{expected regret} (and not the incurred \emph{pseudo}-regret), so we cannot simply exchange the maximum and expectation operations.
The obstacle to this is the term $\sum_{\windowIndex=\start}^{\nSplits_{\nRuns}} \sum_{\run\in\runs_{\windowIndex}} \braket*{\error_{\run}}{\state_{\run} - \simple_{\windowIndex}^{\point}}$ in \eqref{eq:regret-simple-template}, which we will bound window-by-window below.

To do so, let
\begin{equation}
\label{eq:auxreg}
\auxreg_{\windowIndex}(\point)
	= \sum_{\run\in\runs_{\windowIndex}}
		\braket*{\error_{\run}}{\state_{\run} - \simple_{\windowIndex}^{\point}}
	\qquad
	\windowIndex = \running,\nSplits_{\nRuns},
\end{equation}
and consider the auxiliary processes
\begin{equation}
\label{eq:DA-aux}
\daux_{\run+1}
	= \daux_{\run} - \error_{\run},
\quad
\aux_{\run+1}
	= \mirror^{\cover_{\windowIndex}}(\temp_{\run+1} \daux_{\run+1}),
	\qquad
	\run = \run_{\windowIndex},\dotsc,\run_{\windowIndex+1} - 1,
\end{equation}
with $\aux_{\run_{\windowIndex}} = \state_{\run_{\windowIndex}}$.
We then have
\begin{subequations}
\label{eq:reg-aux}
\begin{align}
\auxreg_{\windowIndex}(\point)
	= \sum_{\run\in\runs_{\windowIndex}}
		\braket*{\error_{\run}}{(\state_{\run} - \aux_{\run}) + (\aux_{\run} - \simple_{\windowIndex}^{\point})}
	&\label{eq:reg-aux-diff}
	= \sum_{\run\in\runs_{\windowIndex}}
		\braket*{\error_{\run}}{\state_{\run} - \aux_{\run}}
	\\
	&\label{eq:reg-aux-error}
	+ \sum_{\run\in\runs_{\windowIndex}}
		\braket*{\error_{\run}}{\aux_{\run} - \simple_{\windowIndex}^{\point}}
\end{align}
\end{subequations}
We now proceed to bound each of the above terms in expectation:
\begin{enumerate}
[left=\parindent,label={\itshape\alph*\upshape)}]
\item
Since the term \eqref{eq:reg-aux-diff} does not depend on $\point$, we readily obtain
\begin{align}
\label{eq:reg-aux1}
\exof*{\max_{\point\in\points} \sum_{\run\in\runs_{\windowIndex}} \braket*{\error_{\run}}{\state_{\run} - \aux_{\run}}}
	&= \sum_{\run\in\runs_{\windowIndex}}
		\exof*{\exof*{\braket*{\error_{\run}}{\state_{\run} - \aux_{\run}} \given \filter_{\run}}}
	\notag\\
	&= \sum_{\run\in\runs_{\windowIndex}}
		\exof*{\braket*{\bias_{\run}}{\state_{\run} - \aux_{\run}}}
	\leq 2 \sum_{\run\in\runs_{\windowIndex}}\bbound_{\run}
\end{align}
where the last line follows from \eqref{eq:bbound} and the fact that $\state_{\run}$ and $\aux_{\run}$ are both simple strategies supported on $\cover_{\windowIndex}$ (so $\braket{\bias_{\run}}{\state_{\run} - \aux_{\run}} \leq \abs{\braket{\bias_{\run}}{\state_{\run}}} + \abs{\braket*{\bias_{\run}}{\aux_{\run}}} \leq 2 \bbound_{\run}$).


\item
For the term \eqref{eq:reg-aux-error}, applying \cref{prop:reg-DAX} to the sequence of ``virtual'' payoff functions $-\error_{\run}$, $\run=\running$, we get
\begin{equation}
\sum_{\run\in\runs_{\windowIndex}}
		\braket*{\error_{\run}}{\aux_{\run} - \simple_{\windowIndex}^{\point}}
	\leq \energy^{\cover_{\windowIndex}}_{\run_{\windowIndex}} - \energy^{\cover_{\windowIndex}}_{\run_{\windowIndex+1}}
	+ \parens*{\temp_{\run_{\windowIndex+1}}^{-1} - \temp_{\run_{\windowIndex}}^{-1}}
		\hvol(\nSets_{\run_{\windowIndex}})
	+ \frac{1}{2\hstr}
		\sum_{\run\in\runs_{\windowIndex}}
			\temp_{\run} \tnorm{\error_{\run}}^{2}.
\end{equation}
Since the \acl{RHS} of this last equation does not depend on $\point$, maximizing and taking expectations yields
\begin{equation}
\label{eq:reg-aux2}
\exof*{\max_{\point\in\points} \sum_{\run\in\runs_{\windowIndex}} \braket*{\error_{\run}}{\aux_{\run} - \simple_{\windowIndex}^{\point}}}
	\leq \energy^{\cover_{\windowIndex}}_{\run_{\windowIndex}} - \energy^{\cover_{\windowIndex}}_{\run_{\windowIndex+1}}
	+ \parens*{\temp_{\run_{\windowIndex+1}}^{-1} - \temp_{\run_{\windowIndex}}^{-1}}
		\hvol(\nSets_{\run_{\windowIndex}})
	+ \frac{1}{2\hstr} \sum_{\run\in\runs_{\windowIndex}} \temp_{\run} \zeta_{\run}^{2},
\end{equation}
where we set $\zeta_{\run}^{2} = \exof*{\tnorm{\error_{\run}}^{2}}$.
\end{enumerate}
Therefore, taking expectations in \eqref{eq:reg-aux} and plugging \cref{eq:reg-aux1,eq:reg-aux2} into the resulting expression, we obtain
\begin{equation}
\exof*{\max_{\point\in\points} \auxreg_{\windowIndex}(\point)}
	\leq \energy^{\cover_{\windowIndex}}_{\run_{\windowIndex}} - \energy^{\cover_{\windowIndex}}_{\run_{\windowIndex+1}}
	+ \parens*{\temp_{\run_{\windowIndex+1}}^{-1} - \temp_{\run_{\windowIndex}}^{-1}}
		\hvol(\nSets_{\run_{\windowIndex}})
	+ 2 \sum_{\run\in\runs_{\windowIndex}}\bbound_{\run}
	+ \frac{1}{2\hstr} \sum_{\run\in\runs_{\windowIndex}} \temp_{\run} \zeta_{\run}^{2}
\end{equation}
and hence, using \cref{lem:splitting-residue} and working as in the case of \eqref{eq:regret-simple-template}, we get:
\begin{align}
\exof*{\max_{\point\in\points}
		\sum_{\windowIndex=\start}^{\nSplits_{\nRuns}}
			\sum_{\run\in\runs_{\windowIndex}}
				\braket*{\error_{\run}}{\state_{\run} - \simple_{\windowIndex}^{\point}}
		} 
	&\leq
\sum_{\windowIndex=\start}^{\nSplits_{\nRuns}}
	\exof*{\max_{\point\in\points} \auxreg_{\windowIndex}(\point)}
	\notag\\
	&\leq \frac{\cst_\theta(\nSplits_{\nRuns}-1) + \hvol(\nSets_{\nRuns})}{\temp_{\nRuns+1}}
		+ 2 \sum_{\run=\start}^{\nRuns} \bbound_{\run}
		+ \frac{1}{2\hstr} \sum_{\run=\start}^{\nRuns} \temp_{\run} \zeta_{\run}^{2}.
\end{align}
Thus, going back to \eqref{eq:regret-simple-template} and taking expectations, we get the expected regret bound
\begin{align}
\exof*{\reg(\nRuns)}
	&\leq  2 \frac{\cst_\theta(\nSplits_{\nRuns}-1) + \hvol(\nSets_{\nRuns})}{\temp_{\nRuns+1}}
	\notag\\
	&+ 2 \sum_{\run=\start}^{\nRuns} \bbound_{\run}
		+ \frac{1}{2\hstr} \sum_{\run=1}^{\nRuns}\temp_{\run} (\zeta_{\run}^{2} + \mbound_{\run}^{2})
		+ 2\lips\diamconst \diam(\points)\sum_{\run=1}^{\nRuns} \frac{1}{\nSets_{\run}^{1/\vdim}}
\end{align}
As a last step, since $\error_{\run} = \model_{\run} - \pay_{\run}$, we readily get $\exof*{\tnorm{\error_{\run}}^{2}} \leq 2 \exof*{\tnorm{\pay_{\run}}^{2} + \tnorm{\model_{\run}}^{2}} \leq 2(\rbound^{2} + \mbound_{\run}^{2})$ by \cref{lem:Fisher}, \cref{asm:pay}, and the definition \eqref{eq:mbound} of $\mbound_{\run}$.
The bound \eqref{eq:reg-bound-stat-powers} then follows by a straightforward substitution.
\end{proof}

To proceed with the proof of the specific regret bound for the \ac{HEW} instantiation of \eqref{eq:HDA}, we will require a series of intermediate results to bound the bias and second moment of the estimator \eqref{eq:IWE}.
These are as follows.

\begin{lemma}
\label{lem:HEW-bias-and-meansquare-bound}
Running \ac{HEW} with any splitting schedule implying $\nSets_{\run}$ components of the underlying partiton of $\points$ at time $\run$, the bias and mean square of the \eqref{eq:IWE} satisfy for all $\run$:
\begin{equation}
\label{eq:HEW-bias-and-meansquare-bound}
\begin{aligned}
	\bbound_{\run} &\leq 2\lips \diamconst \diam(\points) \nSets_{\run}^{- 1 / \vdim} \\
	\mbound^2_{\run} &\leq R^2 (\nSets_{\run} + 1)
\end{aligned}
\end{equation}
\end{lemma}

\begin{proof}
To streamline the proof, we first need to introduce some notation.
Specifically, we will write $\cover_\run$ for the underlying partition at time $\run$, and for any $\point \in \points$, $\coverset^\point_\run$ denotes the component of $\cover_\run$ such that $\point \in \coverset^\point_\run$. Let $\choice_\run \in \points$ be the action played at time $\run$ ; to simplify the notations we use the convention introduced in the main text and denote $\coverset_\run := \coverset_\run^{\choice_\run}$. 

Moreover, we recall that $\state_\run \in \simples_{\cover_\run}$ designates the current mixed strategy at $\run$. Specifically for any $\coverset \in \cover_\run$, $\state_{\coverset, \run}$ denote the probability to pick an action $\choice_\run \in \coverset$ at time $\run$. In a slight abuse of notation, we overload $\state_\run$ and also consider it refers to the corresponding \emph{density function} defined on $\points$, \ie for all $\point \in \points$, we have
\begin{equation}
\state_\run(\point) = \leb(\coverset_\run^\point)^{-1} \state_{\coverset_\run^\point, \run}.
\end{equation}
Finally, we recall the definition of the \acdef{IWE}:
\begin{equation}
\tag{\acs{IWE}}
\model_{\run}(\point)
	= \rbound
		- \frac{\rbound - \pay_{\run}(\choice_{\run})}{\state_{\set_{\run},\run}} \oneof{\point\in\set_{\run}},
\end{equation}

\para{Bounding $\bbound_\run$ in the setting of \ac{HEW}}
Recall first that
\begin{equation}
	\textit{Bias:}
	\hspace{1em}
	\abs{\braket*{\bias_{\run}}{\simple}}
	\leq \bbound_{\run}
\end{equation}
for all $\run=\running$, and all $\simple\in\simples$.

Let $\point \in \points$.\\
By definition $\bias_\run(\point) = \pay_\run(\point) - \exof*{\model_\run(\point) | \filter_\run}$. Using \eqref{eq:IWE}, a series of mechanical computations bring
\begin{align}
    \exof*{\model_\run(\point) \given \filter_\run} &= \exof*{\rbound
		- \frac{\rbound - \pay_{\run}(\choice_{\run})}{\state_{\set_{\run},\run}} \oneof{\point\in\set_{\run}} \given \filter_\run}
	\notag\\
		&= \rbound - \int_\points\left(\frac{\rbound - \pay_{\run}(\alt{\point})}{\state_{\set^{\alt{\point}}_{\run},\run}} \oneof{\point\in\set^{\alt{\point}}_{\run}}\right)\state_\run(\alt{\point})d\alt{\point}
	\notag\\
		&= \rbound - \int_{\set^{\point}_{\run}} \left(\frac{\rbound - \pay_{\run}(\alt{\point})}{\state_{\set^{\point}_{\run},\run}}\right)\underbrace{\state_\run(\point)}_{\leb(\coverset_\run^{\point})^{-1} \state_{\coverset_\run^{\point}, \run}}d\alt{\point}
	\notag\\
		&= \rbound - \rbound + \leb(\coverset_\run^{\point})^{-1}\int_{\set^{\point}_{\run}}\pay_{\run}(\alt{\point})d\alt{\point} 
\end{align}
	
For any $\set \subset \points$ and for any measurable function $f : \points \rightarrow \R$ we denote $\bar{f}(\set) = \leb(\set)^{-1}\int_S f(\point)d\point$. We can therefore write
\begin{equation}
\exof*{\model_\run(\point) \given \filter_\run} = \bar{\pay}_\run(\set_\run^\point).
\end{equation}
Therefore, 
\begin{equation}
\bias_\run(\point) = \pay_\run(\point) - \bar{\pay}_\run(\set_\run^\point),
\end{equation}
and the fact that the stream of payoff functions $\pay_\run$ is uniformly Lipschitz directly delivers $\bias_\run(\point) \leq \lips \diam(\coverset_\run^\point)$. Using \cref{lem:diameter-wrt-nSets} finally brings, for all $\point \in \points$:
\begin{equation}
\bias_\run(\point) \leq 2\lips\diamconst \diam(\points)\nSets_\run^{-1/\vdim}
\end{equation}
which, using $\abs{\braket*{\bias_{\run}}{\simple}} \leq \supnorm{\bias_\run}\onenorm{\simple} = \supnorm{\bias_\run}$ shows that 

\begin{equation}\label{eq:bbound-HEW}
    \mu_\run \leq 2\lips\diamconst \diam(\points)\nSets_\run^{-1/\vdim}
\end{equation}

\para{Bounding $\mbound^2_\run$ in the setting of \ac{HEW}}

We recall the definition of $\mbound$ from \eqref{eq:mbound}:
\begin{equation}
	\textit{Mean square:}
	\hspace{1em}
	\exof*{\tnorm{\model_{\run}}^{2} \given \filter_{\run}}
	\leq \mbound_{\run}^{2}
\end{equation}
To simplify the incoming computations, we denote $l_\run:\points \rightarrow \R+$ the \emph{loss} function such that for all $\point \in \points$, $l_\run(\point) = \rbound - \pay_\run(\point)$. Since $0 \leq \pay_\run \leq \rbound$, we also have $0 \leq l_\run \leq \rbound$. For any $\set \subset \points$, we also introduce $\delta_\set : \points \rightarrow \{0,1\}$ the function such that for all $\point \in \points$, $\delta_\set(\point) = \oneof{\point \in \set}$.

With this in hand, we can proceed to the following rewriting of $\tnorm{\model_{\run}}^{2}$, which we recall is a random quantity given filtration $\filter_\run$ since it depends on the choice of the $\run^{th}$ action, $\choice_\run$:

\begin{align}
\tnorm{\model_{\run}}^{2}
	= \braket*{\model_\run^2}{\state_\run}
    &= \braket*{\left(\rbound - \frac{l_\run(\choice_\run)}{\state_{\set_{\run},\run}} \delta_{\set_\run}\right)^2}{\state_\run}
    \notag\\
    &= \rbound^2 - 2\rbound\frac{l_\run(\choice_\run)}{\state_{\set_{\run},\run}}\braket*{\delta_{\set_\run}}{\state_\run} + \frac{l_\run(\choice_\run)^2}{\state_{\set_{\run},\run}^2}\braket*{\delta^2_{\set_\run}}{\state_\run}
\end{align}
For any $\set \subset \points$, $\delta_\set^2 = \delta_\set$, and simple computations give
\(
\braket*{\delta_{\set_\run}}{\state_\run} = \state_{\set_\run, \run}.
\)
This then delivers the following expression for $\tnorm{\model_{\run}}^{2}$:
\begin{equation}\label{eq:mbound-HEW-intermediary}
    \tnorm{\model_{\run}}^{2} = \rbound^2 - 2\rbound l_\run(\choice_\run) + \frac{l_\run(\choice_\run)^2}{\state_{\set_\run, \run}}.
\end{equation}

The aim of the proof is to bound the expectancy of \eqref{eq:mbound-HEW-intermediary} given filtration $\filter_\run$. We are primarily interested in the quantity $\exof*{\frac{l_\run(\choice_\run)^2}{\state_{\set_\run, \run}} \given \filter_\run}$ which is the most complex to handle. We write
\begin{align}
    \exof*{\frac{l_\run(\choice_\run)^2}{\state_{\set_\run, \run}} \given \filter_\run} &= \int_\points \frac{l_\run(\alt{\point})^2}{\state_{\set^{\alt{\point}}_\run, \run}} \state_\run(\alt{\point})d\alt{\point}
    \notag\\
    &= \sum_{\set \in \cover_\run}\int_\set \frac{l_\run(\alt{\point})^2}{\state_{\set, \run}} \underbrace{\state_\run(\alt{\point})}_{\leb(\set)^{-1}\state_{\set, \run}}d\alt{\point}
    \notag\\
    &= \sum_{\set \in \cover_\run}\leb(\set)^{-1}\int_\set l_\run(\alt{\point})^2d\alt{\point}
    = \sum_{\set \in \cover_\run} \bar{l^2_\run}(\set).
\end{align}
Now using that $\abs{l_\run} \leq \rbound$ we have that $\bar{l^2_\run}(\set) \leq \rbound^2$ for any $\set \in \cover_\run$ and therefore
\begin{equation}\label{eq:mbound-HEW-intermediary2}
    \exof*{\frac{l_\run(\choice_\run)^2}{\state_{\set_\run, \run}} \given \filter_\run} \leq |\cover_\run|\rbound^2 = \nSets_\run \rbound^2
\end{equation}
Then, remarking that $\exof*{l_\run(\choice_\run) \given \filter_\run} \geq 0$ and combining \eqref{eq:mbound-HEW-intermediary} and \eqref{eq:mbound-HEW-intermediary2} we finally get
$$
\exof*{\tnorm{\model_{\run}}^{2} \given \filter_\run} \leq \rbound^2 (\nSets_\run + 1)
$$
which delivers
\begin{equation}\label{eq:mbound-HEW}
\mbound^2_\run \leq \rbound^2 (\nSets_\run + 1),
\end{equation}
and concludes the proof.
\end{proof}

To conclude, we now need to relate $\nSets_\run$, the number of sets in partition $\cover_\run$ at time $\run$, and the chosen splitting schedule. In case of a logarithmic splitting schedule $\scheduler_{\run} = \nSetsExp \log_{2}\run$, we present the following result giving upper and lower bound on $\nSets_{\run}$ with respect to $\run$ and $\nSetsExp$.

\begin{lemma}
\label{lem:logarithmic-splitting-schedule}
In case of a logarithmic splitting schedule $\scheduler_{\run} = \nSetsExp \log_{2}\run$, we have for every $\run$:
\begin{equation}
\label{eq:logarithmic-splitting-schedule}
\frac{1}{2}\run^{\nSetsExp} \leq \nSets_{\run} \leq \run^{\nSetsExp}	
\end{equation}
\end{lemma}

\begin{proof}
Let $\scheduler_{\run} = \nSetsExp \log_{2}\run$. By definition of the scheduler function $\scheduler_\run$, this implies that at any time $\run$, $\floor{\scheduler_\run}$ splitting events have occurred. Therefore, we have
$\nSplits_\run = \floor{\scheduler_\run} = \floor{\nSetsExp \log_{2}\run}$.
Since $\nSets_\run = 2^{\nSplits_\run}$ by definition, we get
\begin{equation}\label{eq:nSets-log-splitting}
    \nSets_\run = 2^{\floor{\nSetsExp \log_{2}\run}}.
\end{equation}
The result then follows directly from remarking that
\begin{equation}
\nSetsExp \log_{2}\run - 1 < \floor{\nSetsExp \log_{2}\run} \leq \nSetsExp \log_{2}\run
\end{equation}
and using the fact that $x\mapsto2^x$ is an increasing function.
\end{proof}

We are now in a position to prove our main regret guarantee for the \ac{HEW} algorithm.
For convenience, we restate the relevant result below.

\regstatIWE*

\begin{proof}[Proof of \cref{cor:reg-stat-IWE}]
The idea of this proof consists in bounding the different terms on the right hand side of \eqref{thm:reg-stat} in the case of \ac{HEW} with learning rate $\temp_{\run} \propto \run^\pexp$ and a logarithmic splitting schedule $\scheduler_{\run} = \nSetsExp \log_{2}\run$.
With this in mind, combining \cref{lem:logarithmic-splitting-schedule,lem:HEW-bias-and-meansquare-bound} we get
\begin{subequations}
\begin{alignat}{3}
\mbound_\run^2
	&\leq \rbound^2(\nSets_\run + 1)
	&&\leq \rbound^2(\run^\nSetsExp + 1)
	&&= \bigoh(\run^\nSetsExp)
	\\
\bbound_\run
	&\leq 2\lips\diamconst \diam(\points) \nSets_\run^{-1/\vdim}
	&&\leq 2\lips\diamconst \diam(\points) \left(\frac{1}{2}\run^\nSetsExp\right)^{-1/\vdim}
	&&= \bigoh(\run^{-\nSetsExp/\vdim}).
\end{alignat}
\end{subequations}
The result stated in \cref{cor:reg-stat-IWE} directly follows from injecting this into \eqref{eq:reg-bound-stat-powers}.
\end{proof}

\subsection{Dynamic regret guarantees}

We now turn to the dynamic regret guarantees of \eqref{eq:HDA} as stated in \cref{thm:reg-dyn} below.

\regdyn*

\begin{proof}[Proof of \cref{thm:reg-dyn}]
Our proof will be again based on a window-by-window analysis.
However, instead of focusing on the windows of $\{1,\dotsc,\nRuns\}$ over which the cover of \eqref{eq:HDA} remains constant and fixed, we will decompose the horizon of the process into $\nBatches$ \emph{virtual} batches, and we will compare the learner's static and dynamic regret over each such batch.%
\footnote{A further important difference is that these virtual batches will all have the same length, in contrast to the windows of time between two consecutive splitting events.}
We will then harvest a bound for the aggregate dynamic regret over $\nRuns$ stages following a comparison technique first introduced by \citet{BGZ15}.

To proceed, write the interval $\runs = \window{1}{\nRuns}$ as the union of $\nBatches$ contiguous sub-intervals $\runs_{\iBatch}$, $\iBatch=1,\dotsc,\nBatches$, each of length $\batch$ (with the possible exception of the final batch, which might be shorter).
Formally, let $\batch = \ceil{\nRuns^{\qexp}}$ for some constant $\qexp\in[0,1]$ to be determined later;
then the number of virtual batches is $\nBatches = \ceil{\nRuns/\batch} = \Theta(\nRuns^{1-\qexp})$ and we have
\begin{equation}
\label{eq:batch}
\runs_{\iBatch}
	= \window{(\iBatch-1)\batch+1}{\iBatch\batch}
	\qquad
	\text{for all $\iBatch = 1,\dotsc,\nBatches-1$},
\end{equation}
with $\runs_{\nBatches} = \runs \setminus \union_{\iBatch=1}^{\nBatches-1} \batch_{\iBatch}$ being excluded from the above enumeration as (possibly) smaller than the rest.

Now, focusing on the $\iBatch$-th batch $\runs_{\iBatch}$ of $\runs$ and taking
$\sol_{\run} \in \argmax_{\point\in\points} \pay_{\run}(\point)$
and
$\sol_{\iBatch} \in \argmax_{\point\in\points} \sum_{\run\in\runs_{\iBatch}} \pay_{\run}(\point)$,
we get
\begin{equation}
\pay_{\run}(\sol_{\run}) - \braket*{\pay_{\run}}{\state_{\run}}
	= \pay_{\run}(\sol_{\iBatch}) - \braket*{\pay_{\run}}{\state_{\run}}
		+ \pay_{\run}(\sol_{\run}) - \pay_{\run}(\sol_{\iBatch})
\end{equation}
We may then bound the dynamic regret incurred by \eqref{eq:HDA} over the interval $\runs_{\iBatch}$ as
\begin{align}
\label{eq:dyn=reg+var}
\dynreg(\runs_{\iBatch})
	&= \sum_{\run\in\runs_{\iBatch}}
		\bracks*{\pay_{\run}(\sol_{\run}) - \braket*{\pay_{\run}}{\state_{\run}}}
		+ \sum_{\run\in\runs_{\iBatch}} \bracks*{\pay_{\run}(\sol_{\run}) - \pay_{\run}(\sol_{\iBatch})}
	\notag\\
	&= \reg(\runs_{\iBatch})
		+ \sum_{\run\in\runs_{\iBatch}} \bracks*{\pay_{\run}(\sol_{\run}) - \pay_{\run}(\sol_{\iBatch})}.
\end{align}
Moving forward, we will bound the difference $\sum_{\run\in\runs_{\iBatch}} \bracks*{\pay_{\run}(\sol_{\run}) - \pay_{\run}(\sol_{\iBatch})}$ following a comparison technique originally due to \citet[Prop.~2]{BGZ15}.
To do so, let $\runstart_{\iBatch} = \min\runs_{\iBatch}$ denote the starting epoch of the $\iBatch$-th virtual batch, and let $\sol_{\runstart_{\iBatch}}$ denote a maximizer of the first payoff function encountered in the batch $\runs_{\iBatch}$.
We then obtain by construction
\begin{align}
\label{eq:varbound}
\sum_{\run\in\runs_{\iBatch}} \bracks*{\pay_{\run}(\sol_{\run}) - \pay_{\run}(\sol_{\iBatch})}
	&\leq \sum_{\run\in\runs_{\iBatch}} \bracks*{\pay_{\run}(\sol_{\run}) - \pay_{\run}(\sol_{\runstart_{\iBatch}})}
	\leq \batch \max_{\run\in\runs_{\iBatch}} \bracks*{\pay_{\run}(\sol_{\run}) - \pay_{\run}(\sol_{\runstart_{\iBatch}})}
	\leq 2\batch \tvar[\iBatch],
\end{align}
where we used the fact that $\abs{\runs_{\iBatch}} \leq \batch$ for all $\iBatch=1,\dotsc,\nBatches$ (this time \emph{including} the last batch).
Hence, by combining \eqref{eq:varbound} and \eqref{eq:dyn=reg+var}, we get
\begin{align}
\dynreg(\runs_{\iBatch})
	\leq \reg(\runs_{\iBatch})
		+ 2 \batch \tvar[\iBatch]
\end{align}
Thus, finally, after summing over all batches and taking expectations, we obtain the static-to-dynamic comparison bound
\begin{align}
\label{eq:dyn=reg+var2}
\exof*{\dynreg(\nRuns)}
	\leq \sum_{\iBatch=\start}^{\nBatches} \exof*{\reg(\runs_{\iBatch})}
		+ 2 \batch \tvar.
\end{align}
We will proceed to bound $\exof*{\dynreg(\nRuns)}$ by bounding the ``batch regret'' $\sum_{\iBatch=\start}^{\nBatches} \exof*{\reg(\runs_{\iBatch})}$ and retroactively tuning the batch-size $\batch$.

To carry out this approach, \cref{thm:reg-stat} with $\hdec(\point) = \point\log\point$, $\temp_{\run} \propto \run^{-\pexp}$ and $\nSplits_{\run} = \floor{\nSetsExp \log_{2}\run}$ readily yields
\begin{equation}
\exof*{\reg(\runs_{\iBatch})}
	= \bigoh\parens*{
		(\iBatch\batch)^{\pexp} 
		+ \sum_{\run\in\runs_{\iBatch}} \run^{-\nSetsExp / \vdim}
		+ \sum_{\run\in\runs_{\iBatch}} \run^{-\bexp}
		+ \sum_{\run\in\runs_{\iBatch}} \run^{2\mexp-\pexp}}
\end{equation}
and hence, after summing over all batches:
\begin{align}
\label{eq:reg-sum}
\sum_{\iBatch=1}^{\nBatches} \exof*{\reg(\runs_{\iBatch})}
	&= \bigoh\parens*{
		\batch^{\pexp} \sum_{\iBatch=1}^{\nBatches} \iBatch^{\pexp}
		+ \sum_{\run=\start}^{\nRuns} \run^{-\nSetsExp / \vdim}
		+ \sum_{\run=\start}^{\nRuns} \run^{-\bexp}
		+ \sum_{\run=\start}^{\nRuns} \run^{2\wexp-\pexp}
		} 
	\notag\\[1ex]
	&= \bigoh\parens*{
		\batch^{\pexp} \nBatches^{1+\pexp}
		+ \nRuns^{1-\nSetsExp / \vdim}
		+ \nRuns^{1-\bexp}
		+ \nRuns^{1+2\wexp-\pexp}
		}. 
\end{align}
Now, since $\batch = \bigoh(\nRuns^{\qexp})$ and $\nBatches = \bigoh(\nRuns/\batch) = \bigoh(\nRuns^{1-\qexp})$, the first summand above can be bounded as
\begin{equation}
\batch^{\pexp} \nBatches^{1+\pexp}
	= \bigoh((\nBatches\batch)^{\pexp} \, \nBatches)
	= \bigoh(\nRuns^{\qexp\pexp} \nRuns^{(1-\qexp)(1+\pexp)})
	= \bigoh(\nRuns^{1+\pexp-\qexp}).
\end{equation}
Thus, going back to \eqref{eq:reg-sum} and \eqref{eq:dyn=reg+var2}, we get the dynamic regret bound
\begin{equation}
\exof*{\dynreg(\nRuns)}
	= \bigoh\parens*{
		\nRuns^{1 + \pexp - \qexp}
		+ \nRuns^{1 - \nSetsExp / \vdim}
		+ \nRuns^{1 - \bexp}
		+ \nRuns^{1 + 2\wexp - \pexp}
		+ \nRuns^{\qexp}\tvar}. 
\end{equation}
To calibrate the above expression, the ``virtual batch-size'' exponent $\qexp$ must be chosen such that $1 + \pexp - \qexp = 1 + 2\wexp - \pexp$, \ie $\qexp = 2\pexp - 2\wexp$.
This choice then yields the bound
\begin{equation*}
\exof*{\dynreg(\nRuns)}
	= \bigoh\parens*{
		\nRuns^{1 + 2\wexp - \pexp}
		+ \nRuns^{1 - \nSetsExp / \vdim}
		+ \nRuns^{1 - \bexp}
		+ \nRuns^{2\pexp - 2\wexp}\tvar}.
\qedhere
\end{equation*}
\end{proof}

Finally, \cref{cor:reg-dyn-IWE} simply follows by combining the dynamic regret guarantee of \cref{thm:reg-dyn} with the bounds of \cref{lem:HEW-bias-and-meansquare-bound} for the \ac{IWE} estimator.

\section{Numerical experiments}
\label{app:numerics}

In this appendix, we present the details of the numerical experiments for the \ac{HEW} algorithm \textendash\ dubbed \texttt{Hierarchical} in the sequel. Specifically, we ran different adversarial models (reward design mechanisms) and compared the performance of \texttt{Hierarchical} with two baselines:
\begin{itemize}
\item
a fixed-mesh strategy \textendash\ \texttt{Grid} \textendash\ 
that employs an underlying structure, being an \textit{a priori} chosen mesh of discretization for the search space $\points$, into grid points, and then uses on top of it, the EXP3 algorithm \citep{ACBF02} with rewards sampled at the latter grid points, as per the \ac{DAX} template;
and
\item
the \texttt{Kernel} policy proposed by \citet{HMMR20}, which plugs a kernel density estimate around the sampled action points, and combines it with an explicit exploration term, as per \eqref{eq:IWE3}. 
\end{itemize}

About the adversary functions we choose analitycal functions, defined on the compact search space $\points\subset \R^d$ for $d \in \{ 1,2\}$. Specifically, we instantiate the adversary strategy (the reward mechanism) being:
\begin{itemize}
    \item \texttt{Sine}, with $\points=[0,1]^d$: $\pay_{\run}\from\points\to\R$ is a linear combination of trigonometric terms with different frequencies and amplitudes, arbitrarily drawn, allowing us to know the best action to choose in hindsight (or instantaneously), in order to compute the instantaneous regret. However, we stress that this setting is more a stationary bandit than a proper adversarial one. 
    For this first adversary, \textit{the dynamic regret and the static regret coincide}. That is why we only display the static regret behavior hereafter. We denote respectively the \texttt{Sine} strategy on the 1-dimensional and 2-dimensional search space  by (resp.) \texttt{Sine1D} and \texttt{Sine2D}.

    
    \item \texttt{Gauss} (gaussian reward with stochastic mean), $\points=[-1,1]^d$: a stochastic bandit, with multinomial type reward (with fixed covariance), with mean $\mu$ randomly drawn (iid) round after round, following a uniform distribution on the action space $[-1,1]^d$. We can compute the asymptotic averaged reward over a high number of rounds (used to know the best fixed action). We draw in Figure \ref{fig:sa} some realization of the gaussian reward in 1 dimension and we display on Figure \ref{fig:sb} its asymptotic mean, averaged over 10000 runs. This plot has been produced using Monte Carlo averaging technique and assess the location of the known best action (0). We denote respectively the \texttt{Gauss} strategy on the 1-dimensional and 2-dimensional search space  by (resp.) \texttt{Gauss1D} and \texttt{Gauss2D}.
\end{itemize}

\begin{figure}
\centering
\begin{subfigure}[b]{.45\linewidth}
\centering
\includegraphics[width=\textwidth]{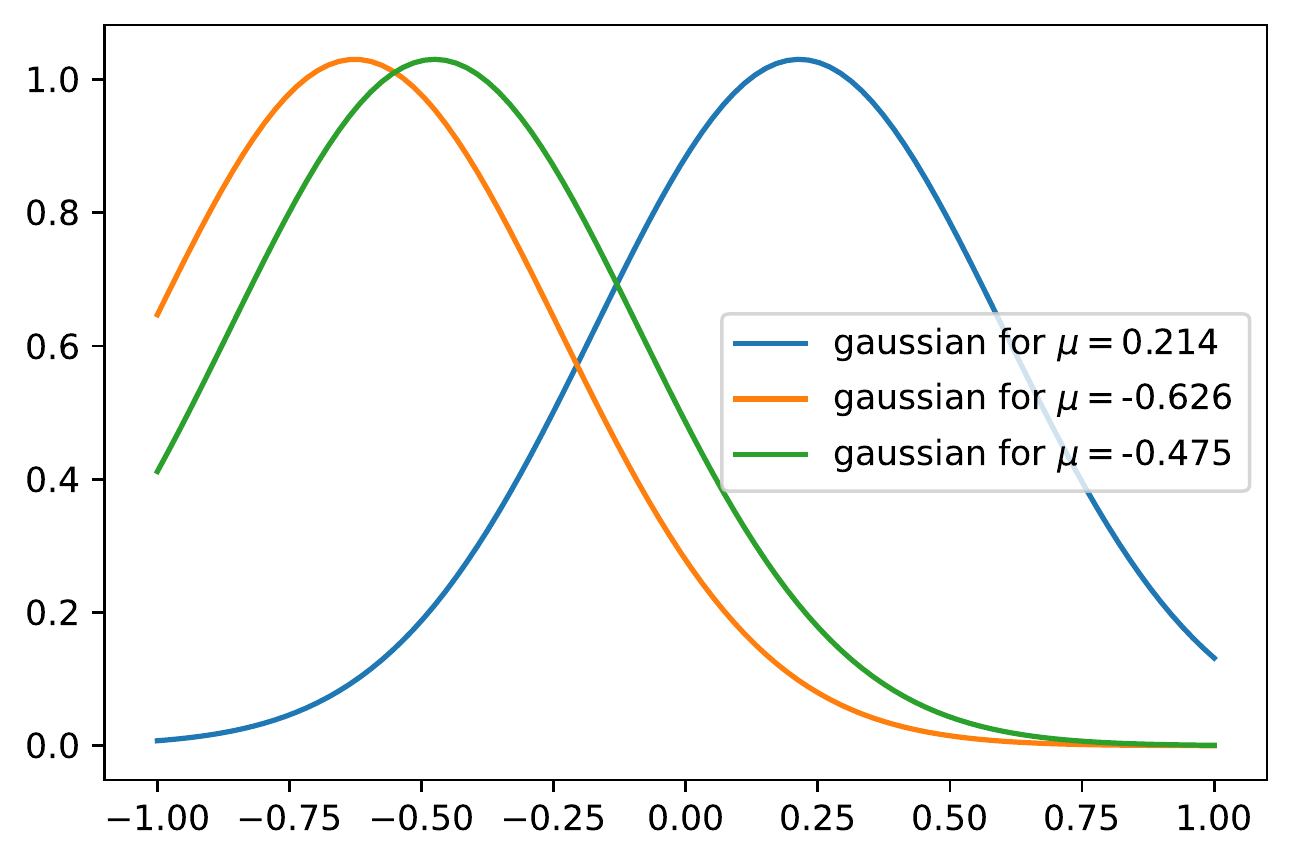}
\caption{3 realizations of such reward functions}
\label{fig:sa}
\end{subfigure}
\quad
\begin{subfigure}[b]{.45\linewidth}
\includegraphics[width=\textwidth]{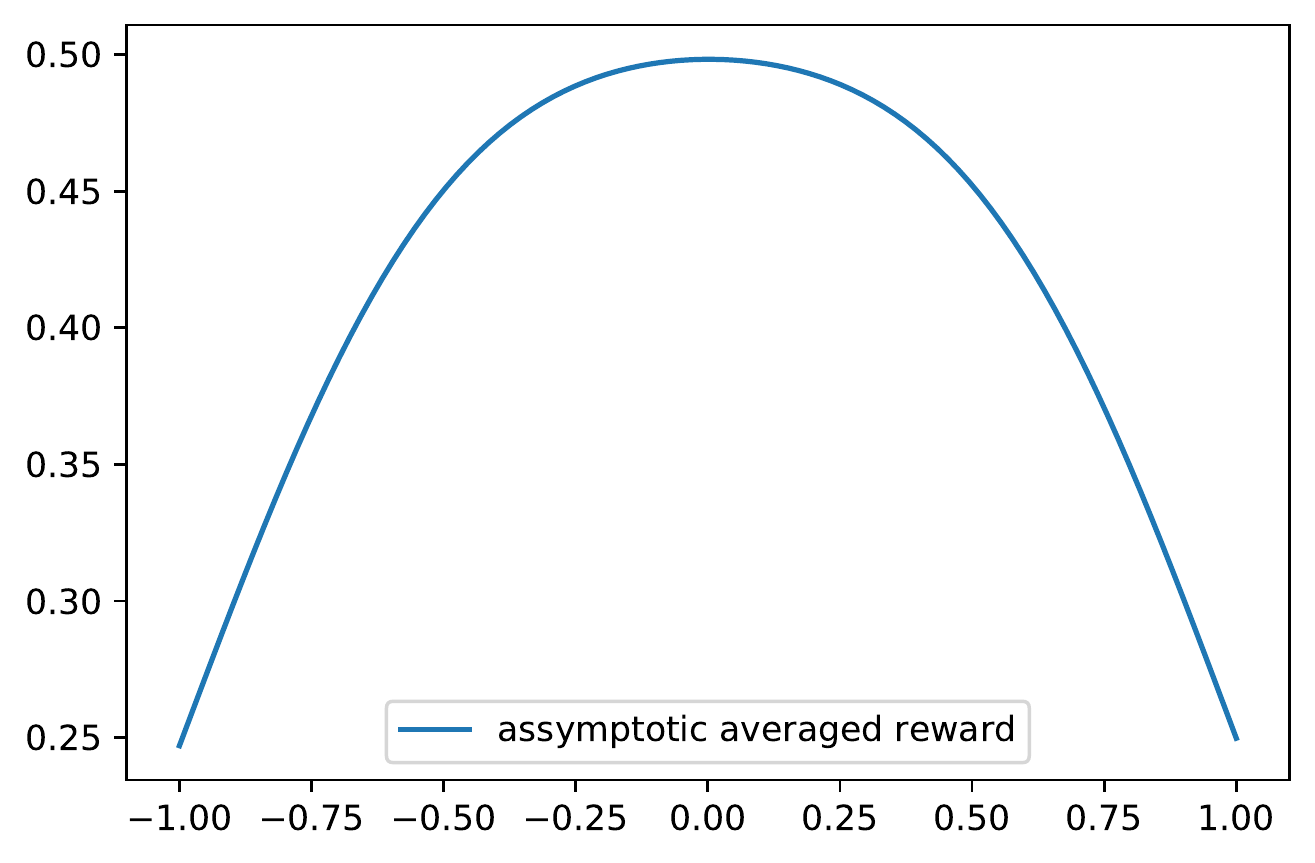}
\caption{Asymptotic averaged reward}
\label{fig:sb}
\end{subfigure}
\caption{Gaussian adversarial reward}
\label{fig:gaussian}
\end{figure}


All numerical experiments were run on a machine with 48 CPUs (Intel(R) Xeon(R) Gold 6146 CPU @ 3.20GHz), with 2 Threads per core, and 500Go of RAM. The horizon was set to $\nRuns= 10^{5}$, and we used the anytime version of every algorithm. We run the algorithm with 46 initial seeds, and then averaged the regret per round, divided by the current round (to exhibit the sub-linear behavior), over the 46 seeds. We add the box-and-whiskers plot showing the confidence interval of the quantity $\reg(t)/t$ for 2 specific round, namely at $t=10^4$ and $t=10^5$, computed empirically on the 46 seeds. We present different sets of hyperparameters for each algorithms, specifically:
\begin{itemize}
    \item \texttt{Kernel}:
    \begin{itemize}
        \item $(\gamma_0,\gamma_r)$ if the learning rate is equal to $\gamma_t=\gamma_0/t^{\gamma_r}$,
        \item number of arms used to store an approximate of the functions defined on $\points$,
        \item $(w_0,w_r )$ if the windows of the squared kernel varies as $w_t=w_0/t^{w_r}$,
        \item $(\text{ee}_0,\text{ee}_r )$ if the explicit exploration equals $\text{ee}_t=\text{ee}_0/t^{\text{ee}_r}$.
    \end{itemize}
    \item \texttt{Grid}
    \begin{itemize}
        \item $\gamma_0$ if the learning rate is equal to $\gamma_t=\gamma_0/t^{\frac{d+1}{d+2}}$,
        \item number of arms used to discretize $\points$ in hindsight, 
    \end{itemize}
    \item \texttt{Hierarchical}
    \begin{itemize}
        \item $(\gamma_0,\gamma_r)$ if the learning rate is equal to $\gamma_t=\gamma_0/t^{\gamma_r}$ 
    \end{itemize}
\end{itemize}
We would like to stress that the number of hyperparameters are not the same, and that the \ac{HEW} algorithm enjoys a lower number of tunable hyperparameters.

On \cref{fig:sinus1D} we plot the mean regret for the \texttt{Sine1D} adversary, with different hyperparameters, over $\nRuns=10^5$ iterations. We display the empirically distribution of such regret divided by the current round $t$ on \cref{fig:sinus1D_slice}, to exhibit the sub-linear behavior. We process the same way on \cref{fig:sinus2D}, and \cref{fig:sinus2D_slice} for the \texttt{Sine2D} adversary, \cref{fig:gauss1D}, \cref{fig:gauss1D_dyn} and \cref{fig:gauss1D_slice} for the \texttt{Gauss1D} adversary and finally \cref{fig:gauss2D}, \cref{fig:gauss2D_dyn} and \cref{fig:gauss2D_slice} for the \texttt{Gauss2D} adversary.

Although the theoretical guarantees of \texttt{Hierarchical} are better, we noticed that the experiments are \emph{very sensitive} to the choice of hyper-parameters and adversary model, which in some cases favored the \texttt{Kernel} algorithm. Moreover, \texttt{Kernel} requires the storage of the entire estimated model as a function defined on the whole domain $\points$, whereas \texttt{Hierarchical} only requires a sub-linear number of bits. Because of this, the per-iteration complexity of \texttt{Kernel} is several orders of magnitude greater than that of \texttt{Hierarchical}; this forced us to consider horizons $T \leq 10^5$ for which the rate difference between \texttt{Kernel} and \texttt{Hierarchical} is relatively small.




\begin{figure}
    \centering
    \includegraphics[width=0.8\textwidth]{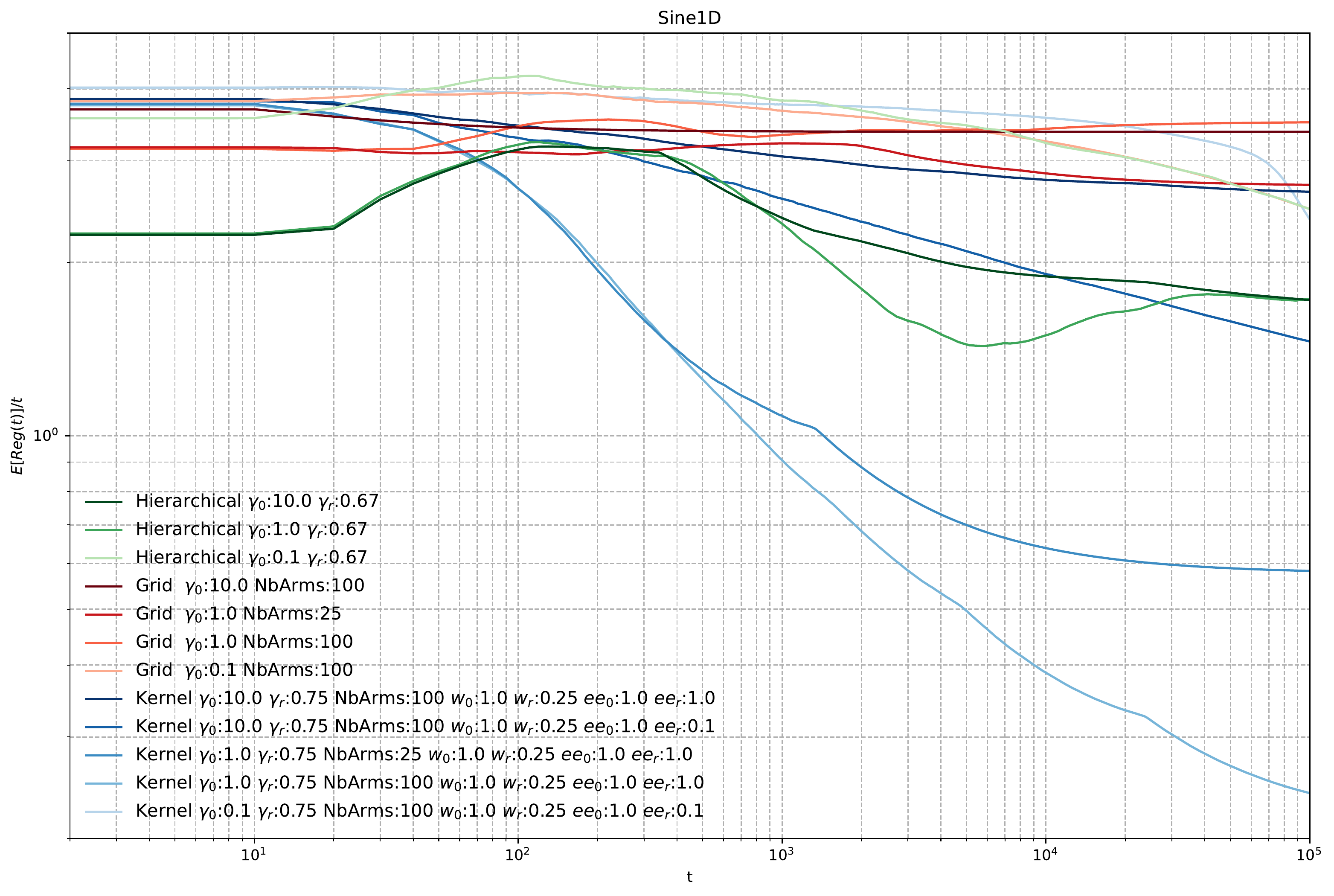}
    \caption{\legendStatic \texttt{Sine1D} adversary.}
\label{fig:sinus1D}
\end{figure}


\begin{figure}
    \centering
    \includegraphics[width=0.95\textwidth]{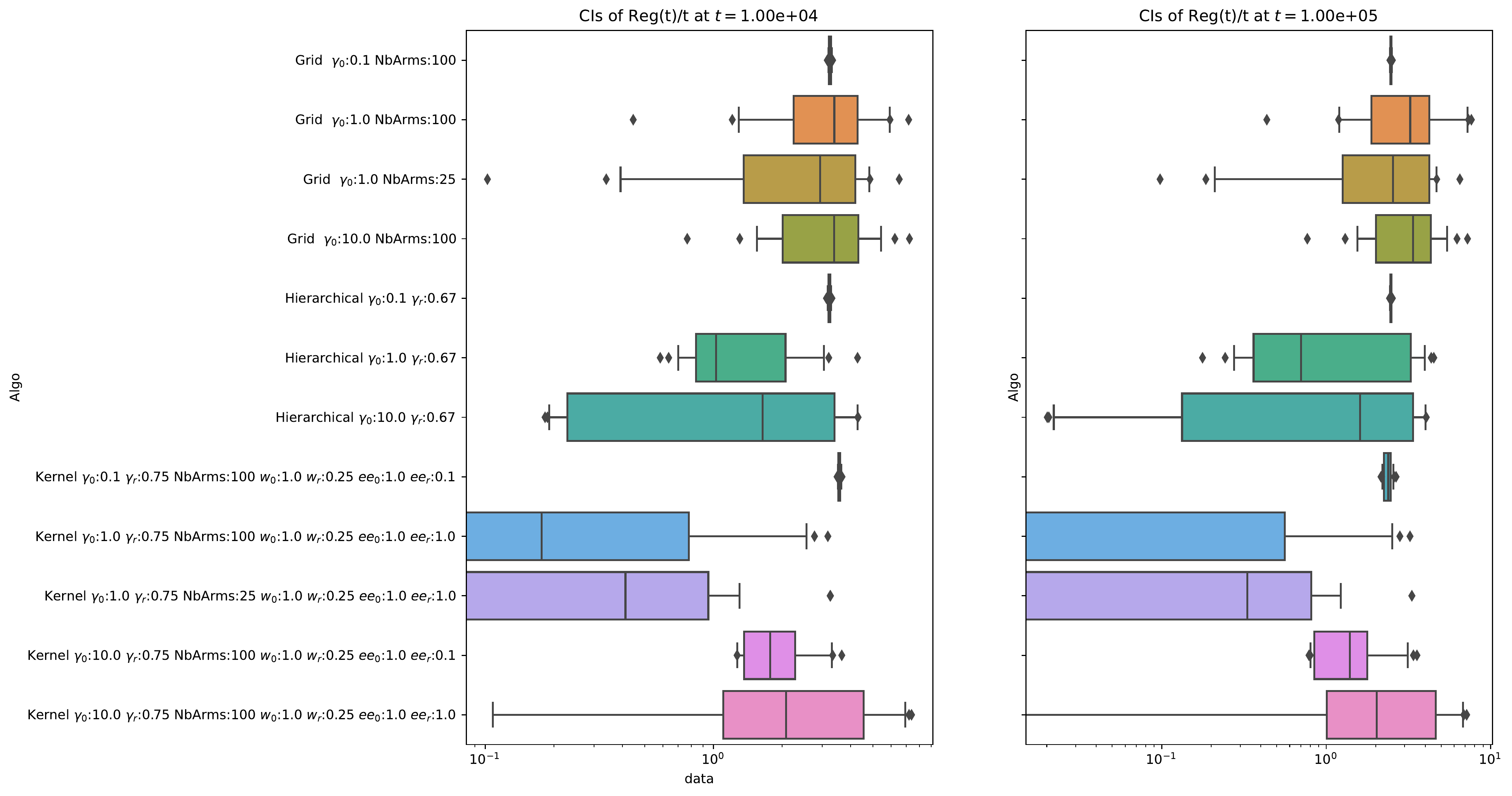}
    \caption{\legendSlide \texttt{Sine1D} adversary.}
\label{fig:sinus1D_slice}
\end{figure}

\begin{figure}
    \centering
    \includegraphics[width=0.8\textwidth]{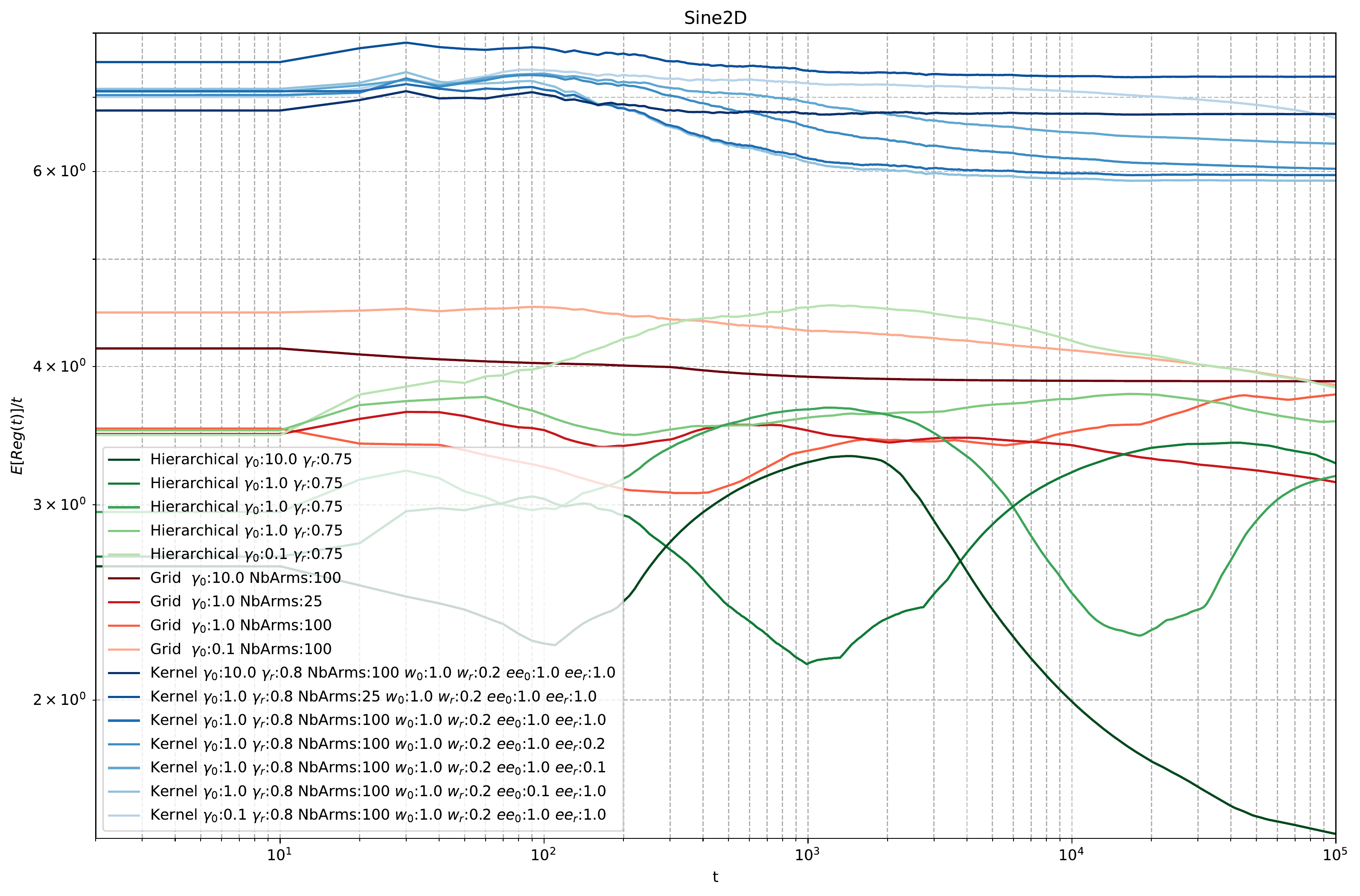}
    \caption{\legendStatic \texttt{Sine2D} adversary.}
\label{fig:sinus2D}
\end{figure}


\begin{figure}
    \centering
    \includegraphics[width=0.95\textwidth]{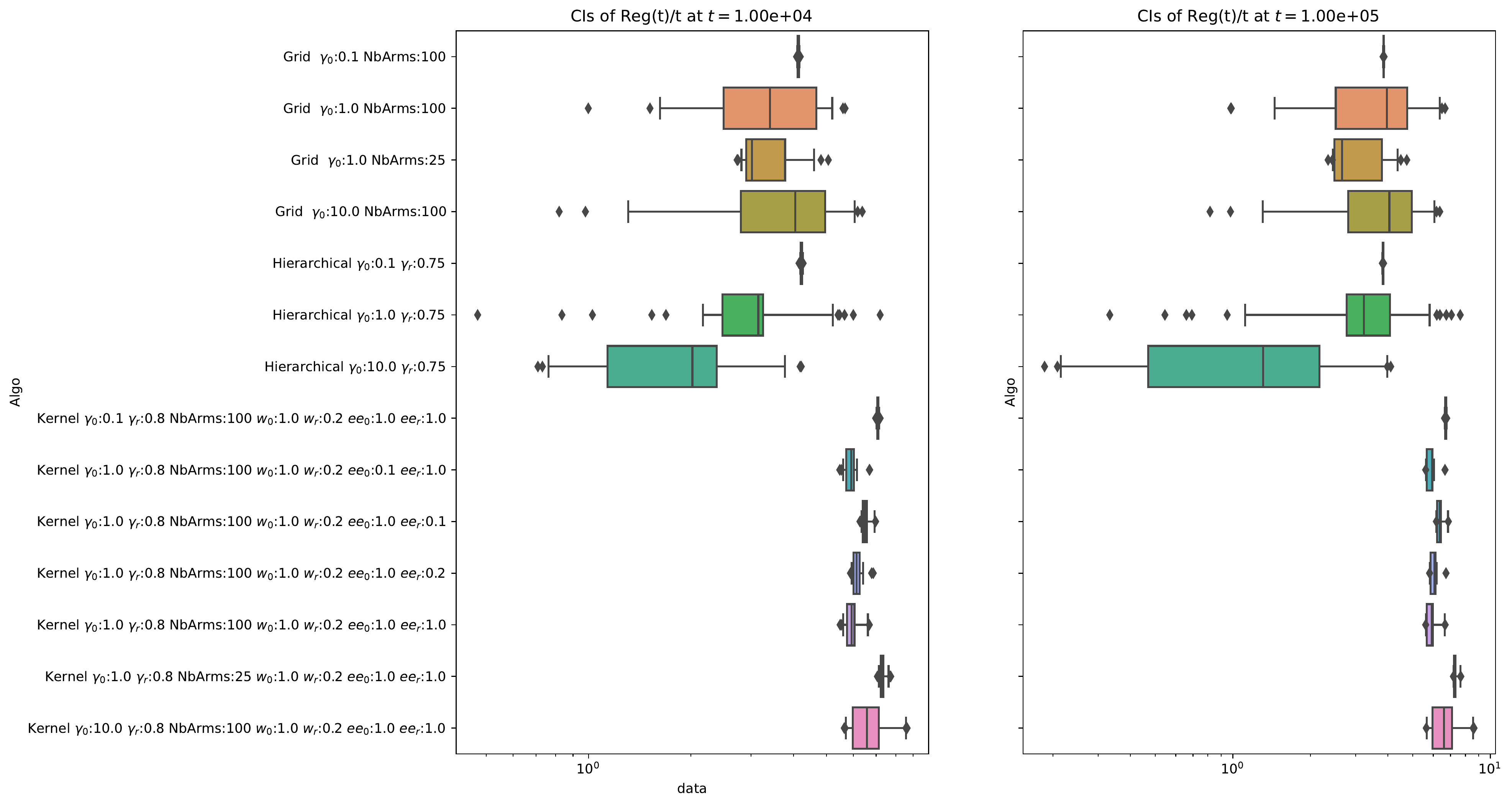}
    \caption{\legendSlide \texttt{Sine2D} adversary.}
\label{fig:sinus2D_slice}
\end{figure}

\begin{figure}
    \centering
    \includegraphics[width=0.8\textwidth]{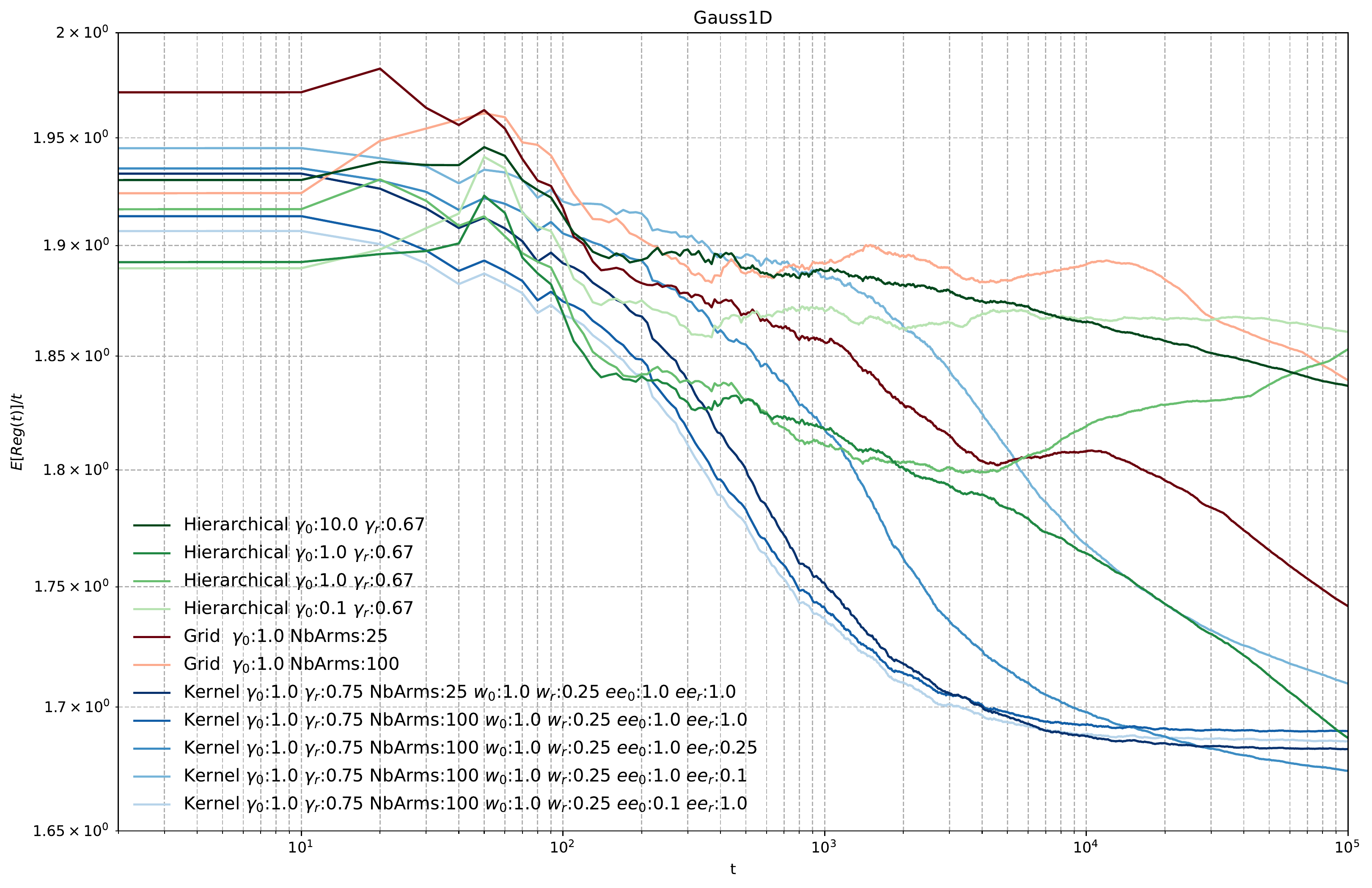}
    \caption{\legendStatic \texttt{Gauss1D} adversary.}
\label{fig:gauss1D}
\end{figure}

\begin{figure}
    \centering
    \includegraphics[width=0.8\textwidth]{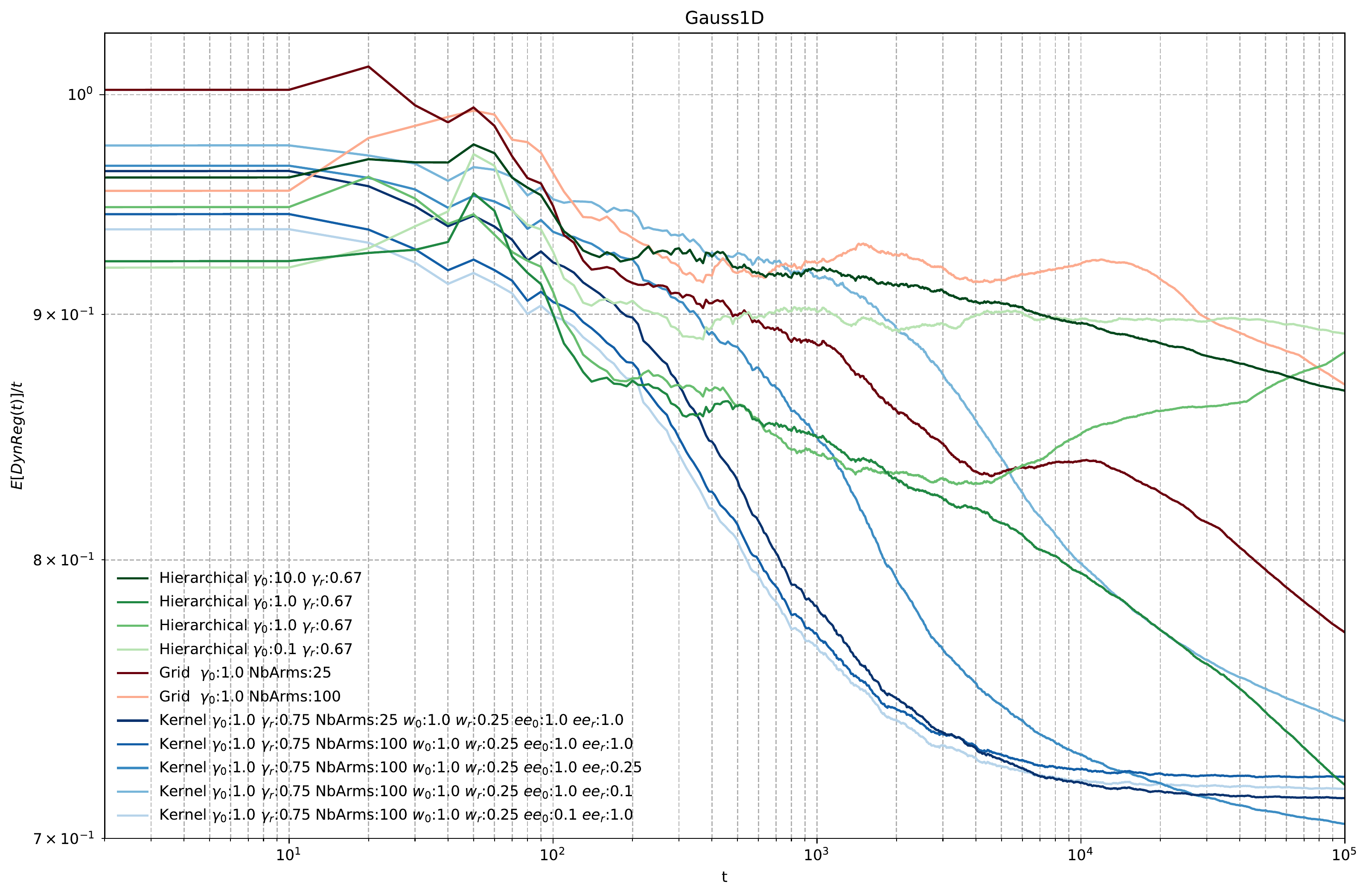}
    \caption{\legendDyn \texttt{Gauss1D} adversary.}
\label{fig:gauss1D_dyn}
\end{figure}

\begin{figure}
    \centering
    \includegraphics[width=0.95\textwidth]{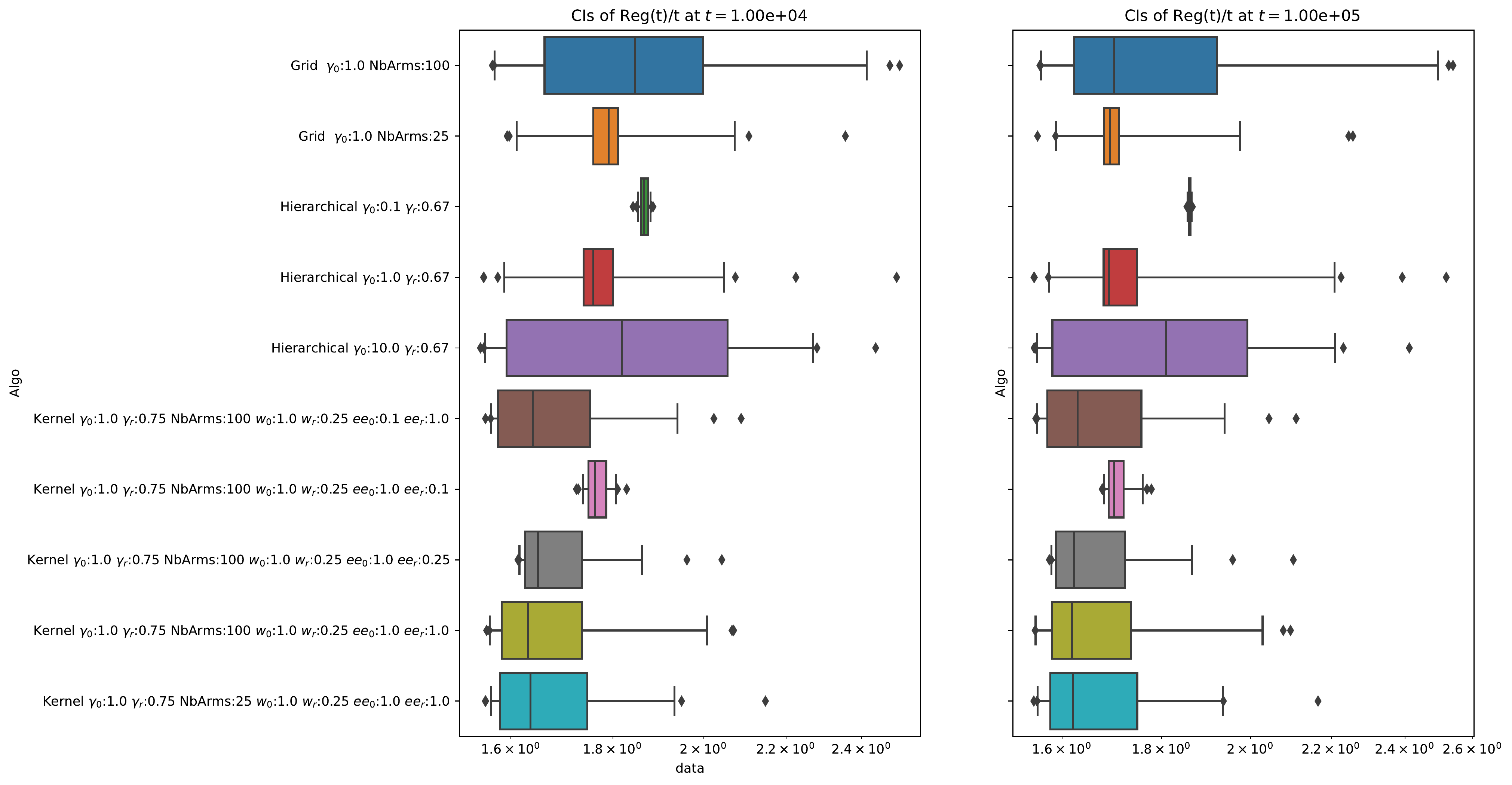}
    \caption{\legendSlide \texttt{Gauss1D} adversary.}
\label{fig:gauss1D_slice}
\end{figure}

\begin{figure}
    \centering
    \includegraphics[width=0.8\textwidth]{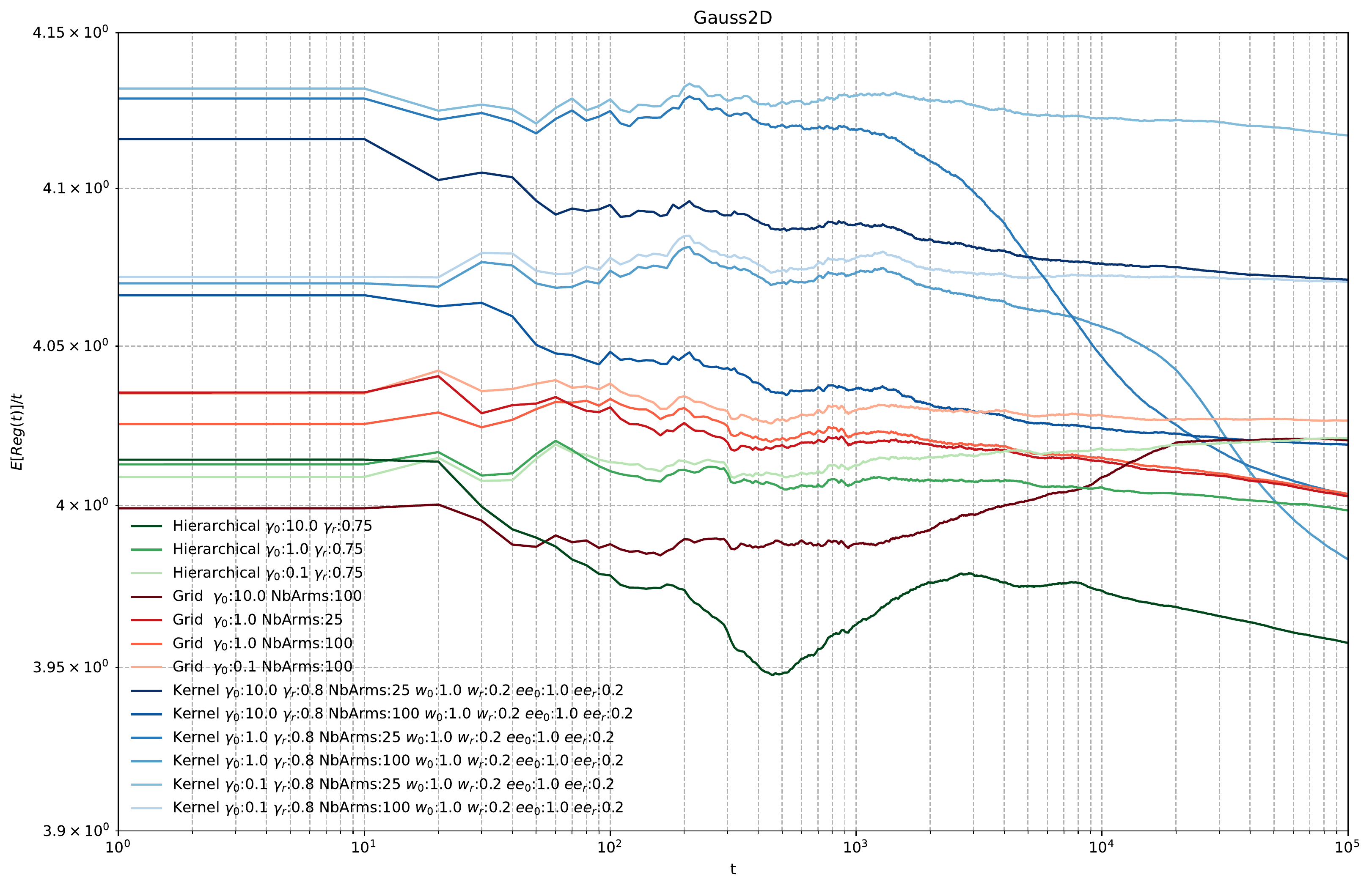}
    \caption{\legendStatic \texttt{Gauss2D} adversary.}
\label{fig:gauss2D}
\end{figure}

\begin{figure}
    \centering
    \includegraphics[width=0.8\textwidth]{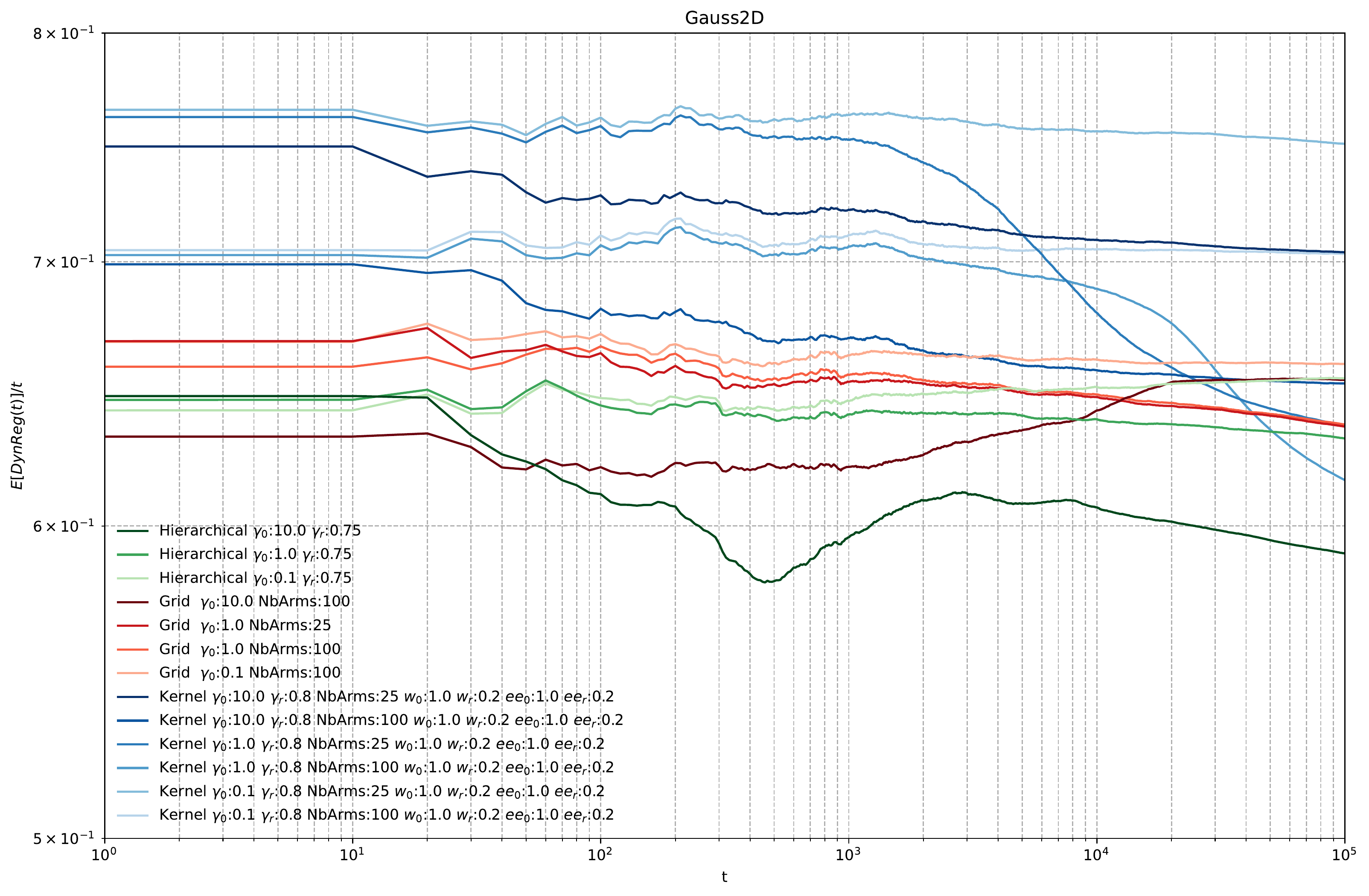}
    \caption{\legendDyn \texttt{Gauss2D} adversary.}
\label{fig:gauss2D_dyn}
\end{figure}

\begin{figure}
    \centering
    \includegraphics[width=0.95\textwidth]{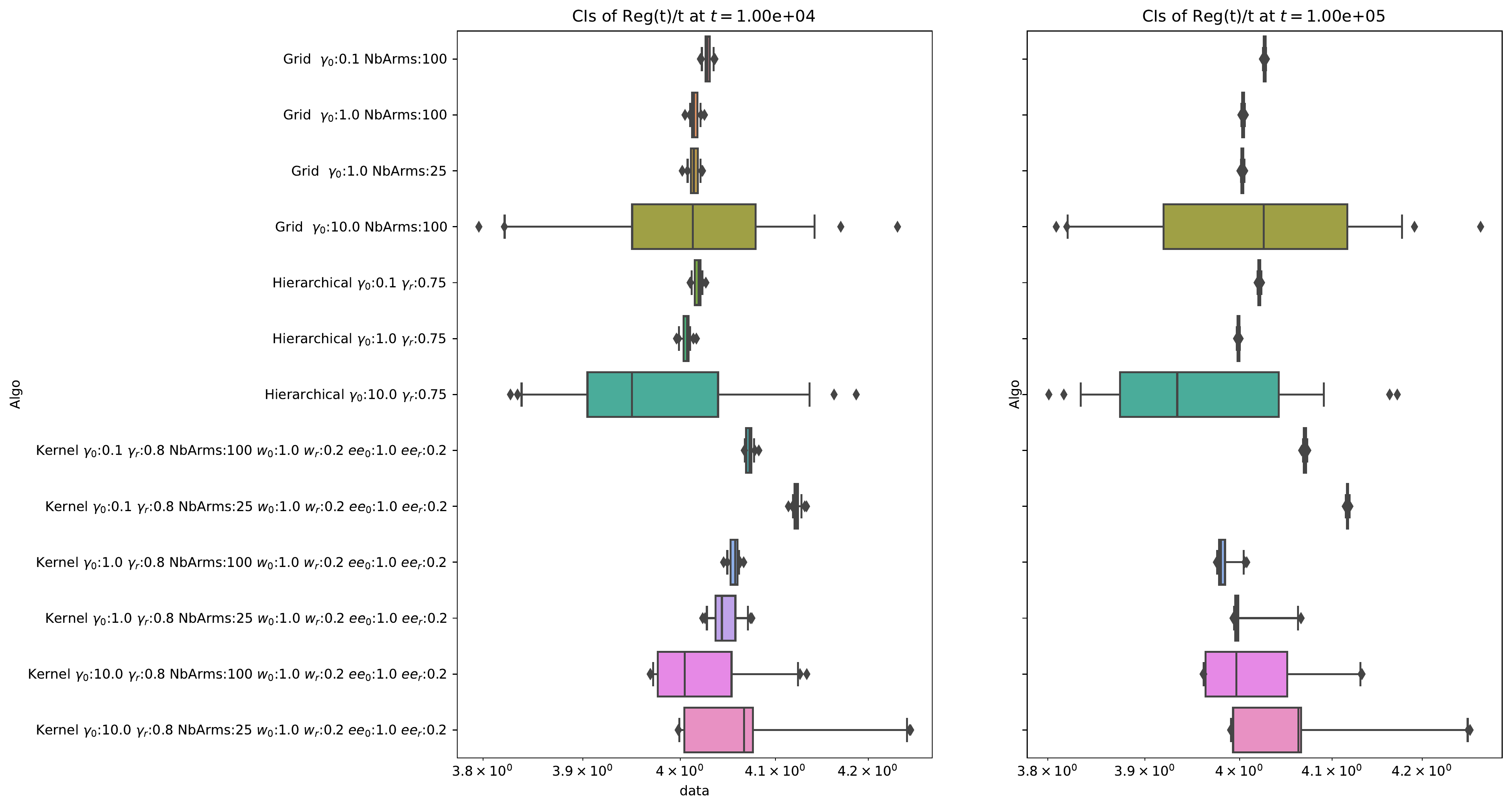}
    \caption{\legendSlide \texttt{Gauss2D} adversary.}
\label{fig:gauss2D_slice}
\end{figure}

\bibliographystyle{ormsv080}
\bibliography{bibtex/IEEEabrv,bibtex/Bibliography-PM}

\end{document}